\documentclass[english]{article}
\usepackage[T1]{fontenc}
\usepackage[latin9]{inputenc}
\usepackage{geometry}
\usepackage[noadjust]{cite}

\usepackage{amsmath,amsbsy,amstext,amssymb,amsthm}
\usepackage{color}
\usepackage{algorithm} 
\RequirePackage[colorlinks,citecolor=blue,urlcolor=blue]{hyperref}
\usepackage{url,mathdots,algorithmic,bm}            
\usepackage{graphics,comment}

\usepackage{setspace}
\usepackage{multirow}
\usepackage{booktabs}       
\usepackage{amsfonts}       
\usepackage{microtype}      
\usepackage{subfigure}
\usepackage{latexsym,epsf,psfrag,wrapfig,epsfig}
\usepackage{booktabs,xspace,paralist}
\usepackage{pmat}

\usepackage{enumitem,url}

\usepackage{dsfont,comment}

\definecolor{yxc}{RGB}{255,0,0}
\definecolor{yjc}{RGB}{125,0,0}
\definecolor{yl}{RGB}{0,0,200}

\newcommand{\yxc}[1]{\textcolor{yxc}{[YXC: #1]}}
\newcommand{\yjc}[1]{\textcolor{yjc}{[YJC: #1]}}

\newcommand{\conj}{\mathsf{H}}

\providecommand{\eref}[1]{\eqref{eq:#1}}  
\providecommand{\cref}[1]{Chapter~\ref{chap:#1}}
\providecommand{\sref}[1]{Section~\ref{sec:#1}}

\providecommand{\abs}[1]{\lvert#1\rvert}
\providecommand{\norm}[1]{\lVert#1\rVert}
\providecommand{\inprod}[1]{\langle#1\rangle}
\providecommand{\set}[1]{\left\{#1\right\}}

\providecommand{\mA}{\bm{A}}

\providecommand{\mI}{\bm{I}}

\providecommand{\mMp}{\bm{M}_{\star}}

\providecommand{\va}{\bm{a}}
\providecommand{\vb}{\bm{b}}
\providecommand{\vh}{\bm{h}}

\providecommand{\vx}{\bm{x}}
\providecommand{\vxp}{\bm{x}_{\star}}
\providecommand{\vhp}{\bm{h}_{\star}}

\providecommand{\vy}{\bm{y}}

\newcommand{\cip}{\overset{\mathbb{P}}{\longrightarrow}}

\providecommand{\EE}{\mathbb{E}}

\usepackage{dsfont}
\newcommand{\charfn}{\mathds{1}}

\DeclareMathOperator{\ind}{\mathds{1}}  

\begin{document}
%

\title{Nonconvex Optimization Meets Low-Rank Matrix Factorization: An Overview\footnotetext{Corresponding author: Yuxin Chen (e-mail: \texttt{yuxin.chen@princeton.edu}).}}
\author{Yuejie Chi \thanks{Y.~Chi is with the Department of Electrical and Computer Engineering, Carnegie Mellon University, Pittsburgh, PA 15213, USA (e-mail: \texttt{yuejiechi@cmu.edu}). The work of Y.~Chi is supported in part by ONR under grants N00014-18-1-2142 and N00014-19-1-2404, by ARO under grant W911NF-18-1-0303, by AFOSR under grant FA9550-15-1-0205, and by NSF under grants CCF-1901199, CCF-1806154 and ECCS-1818571. } 
\and Yue M. Lu  \thanks{Y.~M.~Lu is with the John A. Paulson School of Engineering and Applied
Sciences, Harvard University, Cambridge, MA 02138, USA (e-mail: \texttt{yuelu@seas.harvard.edu}). The work of Y. M. Lu is supported by NSF under grant CCF-1718698. }
\and Yuxin Chen 
\thanks{Y.~Chen is with the Department of Electrical Engineering, Princeton University, Princeton, NJ 08544, USA (e-mail: \texttt{yuxin.chen@princeton.edu}). The work of Y.~Chen is supported in part by the AFOSR YIP award FA9550-19-1-0030, by the ARO  grant W911NF-18-1-0303, by the ONR grant N00014-19-1-2120, by the NSF grants CCF-1907661 and IIS-1900140, and by the Princeton SEAS innovation award. }
}

 \date{September 2018; \quad Revised: September 2019}


\newcommand{\mylinebreak}{}
\newcommand{\myalign}{}
\newcommand{\myaligninv}{&}
\newcommand{\myquad}{}
\newcommand{\myquadinv}{\quad}
\newcommand{\mynonumber}{}
\newcommand{\mycommentbegin}{}

\newcommand{\bLambda}{{\boldsymbol{\Lambda}}}
\newcommand{\blambda}{{\bm{\lambda}}}
\newcommand{\balpha}{{\bm{\alpha}}}
\newcommand{\bOmega}{{\bm{\Omega}}}
\newcommand{\bPsi}{{\bm{\Psi}}}
\newcommand{\bDelta}{{\bm{\Delta}}}
\newcommand{\bTheta}{{\bm{\Theta}}}
\newcommand{\cTheta}{{\mathcal{\Theta}}}
\newcommand{\cA}{{\mathcal{A}}}
\newcommand{\cP}{{\mathcal{P}}}
\newcommand{\cE}{{\mathcal{E}}}
\newcommand{\cS}{{\mathcal{S}}}
\newcommand{\cR}{{\mathcal{R}}}
\newcommand{\bPhi}{{\bm{\Phi}}}
\newcommand{\bphi}{{\bm{\phi}}}
\newcommand{\bT}{{\boldsymbol T}}
\newcommand{\bC}{{\boldsymbol C}}
\newcommand{\bc}{{\boldsymbol c}}
\newcommand{\bD}{{\boldsymbol D}}
\newcommand{\bd}{{\boldsymbol d}}
\newcommand{\bSigma}{{\boldsymbol \Sigma}}
\newcommand{\be}{{\boldsymbol e}}
\newcommand{\bz}{{\boldsymbol z}}
\newcommand{\bq}{{\boldsymbol q}}
\newcommand{\f}{{\boldsymbol f}}
\newcommand{\bh}{{\boldsymbol h}}
\newcommand{\bu}{{\boldsymbol u}}
\newcommand{\bv}{{\boldsymbol v}}
\newcommand{\bE}{{\boldsymbol E}}
\newcommand{\bM}{{\boldsymbol M}}
\newcommand{\E}{{\mathbb E}}
\newcommand{\bA}{{\boldsymbol A}}
\newcommand{\bG}{{\boldsymbol G}}
\newcommand{\bQ}{{\boldsymbol Q}}
\newcommand{\bn}{{\boldsymbol n}}
\newcommand{\ba}{{\boldsymbol a}}
\newcommand{\bb}{{\boldsymbol b}}
\newcommand{\bw}{{\boldsymbol w}}
\newcommand{\bS}{{\boldsymbol S}}
\newcommand{\bx}{{\boldsymbol x}}
\newcommand{\bU}{{\boldsymbol U}}
\newcommand{\by}{{\boldsymbol y}}
\newcommand{\br}{{\boldsymbol r}}
\newcommand{\bt}{{\boldsymbol t}}
\newcommand{\bg}{{\boldsymbol g}}
\newcommand{\bL}{{\boldsymbol L}}
\newcommand{\bR}{{\boldsymbol R}}
\newcommand{\bX}{{\boldsymbol X}}
\newcommand{\bH}{{\boldsymbol H}}
\newcommand{\bV}{{\boldsymbol V}}
\newcommand{\bW}{{\boldsymbol W}}
\newcommand{\bB}{{\boldsymbol B}}
\newcommand{\bs}{{\boldsymbol s}}
\newcommand{\bY}{{\boldsymbol Y}}
\newcommand{\bP}{{\boldsymbol P}}
\newcommand{\bF}{{\boldsymbol F}}
\newcommand{\bp}{{\boldsymbol p}}
\newcommand{\bI}{{\boldsymbol I}}
\newcommand{\bZ}{{\boldsymbol Z}}
\newcommand{\Tr}{\mathop{\rm Tr}}
\newcommand{\T}{{\sf T}}
\newcommand{\F}{{\sf F}}
\newcommand{\argmax}{\mathop{\rm argmax}}
\newcommand{\argmin}{\mathop{\rm argmin}}
\newcommand{\HTP}{\mathop{\tt HTP}}
\newcommand{\diag}{\mathop{\rm diag}}
\newcommand{\dist}{\mathop{\rm dist}}
\newcommand{\phase}{\mathop{\rm Phase}}
\newcommand{\sgn}{\mathop{\rm sgn}}
\newcommand{\vect}{\mathop{\rm vec}}

\renewcommand{\Re}{\operatorname{Re}}
\renewcommand{\Im}{\operatorname{Im}} 
\newcommand{\C}{{\mathbb{C}}}
\newcommand{\R}{{\mathbb{R}}}

\newcommand{\CC}{\mathbb{C}}



\theoremstyle{plain} \newtheorem{lemma}{\textbf{Lemma}} \newtheorem{prop}{\textbf{Proposition}}\newtheorem{theorem}{\textbf{Theorem}}\setcounter{theorem}{0}
\newtheorem{corollary}{\textbf{Corollary}} \newtheorem{assumption}{\textbf{Assumption}}
\newtheorem{example}{\textbf{Example}} \newtheorem{definition}{\textbf{Definition}}
\newtheorem{fact}{\textbf{Fact}} \theoremstyle{definition}

\theoremstyle{remark}\newtheorem{remark}{\textbf{Remark}}\newtheorem{condition}{Condition}\newtheorem{claim}{Claim}


\newcommand{\yuejie}[1]{{\color{red}[Yuejie: {#1}]}}

\renewcommand{\citepunct}{,\penalty\citepunctpenalty\,}
\renewcommand{\citedash}{--}

\maketitle


 \begin{abstract}
Substantial progress has been made recently on developing provably accurate and efficient algorithms for low-rank matrix factorization via nonconvex optimization. While  conventional wisdom often takes a dim view of nonconvex optimization algorithms due to their susceptibility to spurious local minima, simple iterative methods such as gradient descent have been remarkably successful in practice. The theoretical footings, however, had been largely lacking until recently. 
	 
In this tutorial-style overview, we highlight the important role of statistical models in enabling efficient nonconvex optimization with performance guarantees.
	 We review two contrasting approaches: (1) two-stage algorithms, which consist of a tailored initialization step followed by successive refinement; and (2) global landscape analysis and initialization-free algorithms.  	 
	 Several canonical matrix factorization problems are discussed, including but not limited to matrix sensing, phase retrieval, matrix completion, blind deconvolution, robust principal component analysis, phase synchronization, and joint alignment. Special care is taken to illustrate the key technical insights underlying their analyses. This  article serves as a testament that the integrated consideration of optimization and statistics leads to fruitful research findings. 
 \end{abstract}


\tableofcontents

 

\section{Introduction}

Modern information processing and machine learning often have to deal with (structured) low-rank matrix factorization. Given a few observations $\bm{y}\in \mathbb{R}^m$ about a matrix $\bm{M}_{\star}\in \mathbb{R}^{n_1\times n_2}$ of rank $r\ll \min\{n_1,n_2\}$, one seeks a low-rank solution compatible with this set of observations as well as other prior constraints. Examples include low-rank matrix completion \cite{candes2009exact,davenport2016overview,chen2018harnessing}, phase retrieval \cite{shechtman2015phase},  blind deconvolution and self-calibration \cite{ahmed2014blind,ling2015self}, robust principal component analysis \cite{chandrasekaran2011siam,candes2009robustPCA}, synchronization and alignment \cite{singer2011angular,chen2016projected}, to name just a few. A common goal of these problems is to develop reliable, scalable, and robust algorithms to estimate a low-rank matrix of interest, from potentially noisy, nonlinear, and highly incomplete observations.

\subsection{Optimization-based methods}

Towards this goal, arguably one of the most popular approaches is optimization-based methods. By factorizing a candidate solution $\bM\in\mathbb{R}^{n_1\times n_2}$ as $\bm{L}\bm{R}^{\top}$ with low-rank factors $\bm{L}\in \mathbb{R}^{n_1\times r}$ and $\bm{R}\in \mathbb{R}^{n_2\times r}$, one attempts recovery by solving an optimization problem in the form of
\begin{subequations} \label{eq:core_problem}
\begin{align}
	\text{minimize}_{\bm{L},\bm{R}}\quad & f(\bm{L}\bm{R}^{\top}) \label{eq:empirical-risk-min} \\
	\text{subject to}\quad & \bm{M} =\bm{L}\bm{R}^{\top}; \label{eq:low-rank-constraint} \\
			 \quad & \bm{M} \in\mathcal{C}. \label{eq:other-constraints}
\end{align}
\end{subequations}
Here, $f(\cdot)$ is a certain empirical risk function (e.g.~Euclidean loss, negative log-likelihood)  that evaluates how well a candidate solution fits the observations, and the set $\mathcal{C}$ encodes additional prior constraints, if any. This problem is often highly nonconvex and appears daunting to solve to global optimality at first sight. After all, conventional wisdom usually perceives nonconvex optimization as a computationally intractable task that is susceptible to local minima.

To bypass the challenge, one can resort to {\em convex relaxation}, an effective strategy that already enjoys theoretical success  in addressing a large number of problems. The basic idea is to convexify the problem by, amongst others, dropping or replacing the low-rank constraint \eqref{eq:low-rank-constraint} by a nuclear norm constraint \cite{fazel2002matrix,recht2010guaranteed,candes2009exact,CanTao10,davenport2016overview,chen2018harnessing}, and solving the convexified problem in the full matrix space (i.e.~the space of $\bm{M}$).  While such convex relaxation schemes exhibit intriguing performance guarantees in several aspects (e.g. near-minimal sample complexity, stability against noise), its computational cost often scales at least cubically in the size of the matrix $\bm{M}$, which often far exceeds the time taken to read the data. In addition, the prohibitive storage complexity associated with the convex relaxation approach presents another hurdle that limits its applicability to large-scale problems.

This overview article focuses on provable low-rank matrix estimation based on nonconvex optimization. This approach operates over the  parsimonious factorized representation \eqref{eq:low-rank-constraint} and optimizes the nonconvex loss directly over the low-rank factors $\bm{L}$ and $\bm{R}$. The advantage is clear: adopting economical representation of the low-rank matrix results in low storage requirements, affordable per-iteration computational cost, amenability to parallelization, and scalability to large problem size, when performing iterative optimization methods like gradient descent.  
However, despite its wide use and remarkable performance in practice \cite{koren2009matrix,chen2004recovering}, the foundational understanding of generic nonconvex optimization is far from mature.  It is often unclear whether an  optimization algorithm can converge to the desired global solution and, if so, how fast this can be accomplished. For many nonconvex problems, theoretical underpinnings had been lacking until very recently.



\subsection{Nonconvex optimization meets statistical models}

Fortunately, despite  general intractability, some important nonconvex problems may not be as hard as they seem.  For instance, for several low-rank  matrix factorization problems,  it has been shown  that:  under proper statistical models,  simple first-order methods are guaranteed to succeed in a small number of iterations, achieving low computational and sample complexities simultaneously 
(e.g.~\cite{keshavan2010matrix,jain2013low,candes2015phase,sun2016guaranteed,chen2015solving,chen2015fast,li2016deconvolution,tu2015low,zhang2016provable,netrapalli2014non,ma2017implicit}).
The key to enabling guaranteed and scalable computation is to concentrate on problems arising from specific statistical signal estimation tasks, which may exhibit benign structures amenable to computation and help rule out undesired ``hard'' instances by focusing on the average-case performance.  Two messages deserve particular attention when we examine the geometry of associated nonconvex loss functions: 

\begin{itemize}
	\itemsep0.5em

	\item {\bf Basin of attraction.} For several statistical problems of this kind,  there often exists a  reasonably large basin of attraction around the global solution, within which an iterative  method like gradient descent is guaranteed to be successful and converge fast. Such a basin might exist even when the sample complexity is quite close to the information-theoretic limit  \cite{keshavan2010matrix,jain2013low,candes2015phase,sun2016guaranteed,chen2015solving,chen2015fast,li2016deconvolution}.

	\item {\bf Benign global landscape.} Several problems  provably enjoy benign optimization landscape when the sample size is sufficiently
large, in the sense that there is no spurious local minima, i.e. all local minima are also global minima, and that the only undesired stationary points are strict saddle points \cite{sun2016geometric,sun2015complete,ge2016matrix,bhojanapalli2016global,li2018non}.

\end{itemize}

These important messages inspire a recent flurry of activities in the design of two contrasting algorithmic approaches: 

\begin{itemize}
	\itemsep0.5em

	\item {\bf Two-stage approach.} Motivated by the existence of a basin of attraction, a large number of works follow a two-stage paradigm:  (1) {\em initialization}, which locates an initial guess within the basin; (2) {\em iterative refinement}, which successively refines the estimate without leaving the basin. This approach often leads to very efficient algorithms that run in  time proportional to that taken to read the data.  

	\item {\bf Saddle-point escaping algorithms.} In the absence of spurious local minima, a key challenge boils down to how to efficiently escape undesired saddle points and find a local minimum, which is the focus of this approach. This approach does not rely on carefully-designed initialization. 

\end{itemize}

The research along these lines  highlights the synergy between statistics and optimization in data science and machine learning. The algorithmic choice often  needs to properly exploit the underlying statistical models in order to be truly efficient, in terms of both statistical accuracy and computational efficiency.  


%

\subsection{This paper}

Understanding the effectiveness of nonconvex optimization  is currently among the most active areas of research in machine learning, information and signal processing, optimization and statistics. Many exciting new developments in the last several years have significantly advanced our understanding of this approach for various statistical problems.  This article aims to provide a thorough, but by no means exhaustive, technical overview of important recent results in this exciting area, targeting the broader machine learning, signal processing, statistics, and optimization communities.  

The rest of this paper is organized as follows. Section~\ref{sec:preliminaries} reviews some preliminary facts on optimization that are instrumental to understanding the materials in this paper. Section~\ref{sec:noncvx_eg} uses a toy (but non-trivial) example (i.e.~rank-1 matrix factorization) to illustrate why it is possible to solve a nonconvex problem to global optimality, through both local and global lenses. Section~\ref{sec:examples} introduces a few canonical statistical estimation problems that will be visited multiple times in the sequel. Section~\ref{sec:gd} and Section~\ref{sec:GD_variants} review gradient descent and its many variants as a local refinement procedure, followed by a discussion of other methods in Section~\ref{sec:other_algorithms}. Section~\ref{sec:spectral-initialization} discusses the spectral method, which is commonly used to provide an initialization within the basin of attraction. Section~\ref{sec:global} provides a global landscape analysis, in conjunction with algorithms that work without the need of careful initialization. We conclude the paper in Section~\ref{sec:conclusions} with some discussions and remarks. Furthermore, a short note is provided at the end of several sections to cover some historical remarks and provide further pointers.


\subsection{Notations}

It is convenient to introduce a few notations that will be used throughout. We use boldfaced symbols to represent vectors and matrices.   For any vector $\bm{v}$, we let $\|\bm{v}\|_2$, $\|\bm{v}\|_1$ and $\|\bm{v}\|_0$ denote its $\ell_2$, $\ell_1$, and $\ell_0$ norm, respectively.  For any matrix $\bm{M}$, let $\|\bm{M}\|$, $\|\bm{M}\|_{\mathrm{F}}$, $\|\bm{M}\|_{*}$, $\|\bm{M}\|_{2,\infty}$, and $\|\bm{M}\|_{\infty}$  stand for the spectral norm (i.e.~the largest singular value), the Frobenius norm, the nuclear norm (i.e.~the sum of  singular values), the $\ell_2/\ell_{\infty}$ norm (i.e.~the largest $\ell_2$ norm of the rows), and the entrywise $\ell_{\infty}$ norm (the largest magnitude of all entries), respectively. 
We denote by $\sigma_{j}(\bm{M})$ (resp.~$\lambda_{j}(\bm{M})$) the $j$th largest singular value (resp.~eigenvalue)  of $\bm{M}$, and let $\bm{M}_{j,\cdot}$ (resp.~$\bm{M}_{\cdot,j}$) represent  its $j$th row (resp.~column). The condition number of $\bm{M}$ is denoted by $\kappa(\bm{M})$. In addition,  $\bm{M}^{\top}$, $\bm{M}^{\conj}$  and $\overline{\bm{M}}$  indicate the transpose, the conjugate transpose, and the entrywise conjugate of $\bm{M}$, respectively. For two matrices $\bm{A}$ and $\bm{B}$ of the same size, we define their inner product as $\langle {\bm{A}}, \bm{B} \rangle = \mathsf{Tr}(\bm{A}^{\top}\bm{B})$, where $\mathsf{Tr}(\cdot)$ stands for the trace. The matrix $\bm{I}_{n}$ denotes the $n\times n$ identity matrix, and $\bm{e}_i$ denotes the $i$th column of $\bm{I}_{n}$. For a linear operator $\mathcal{A}$, denote by $\mathcal{A}^*$  its adjoint operator. For example, if $\mathcal{A}$ maps $\bm{X}\in \mathbb{R}^{n\times n}$ to $[\langle \bm{A}_i, \bm{X}\rangle ]_{1\leq i\leq m}$, then 
$\mathcal{A}^*(\bm{y})= \sum_{i=1}^m y_i \bm{A}_i$. We also let $\mathsf{vec}(\bm{Z})$  denote the vectorization of a matrix $\bm{Z}$. The indicator function $\charfn_{A}$ equals $1$ when the event $A$ holds true, and $0$ otherwise. Further, the notation $\mathcal{O}^{r\times r}$ denotes  the set of $r\times r$ orthonormal matrices. Let $\bm{A}$ and $\bm{B}$ be two square matrices. We write $\bm{A} \succ \bm{B}$ (resp.~$\bm{A} \succeq \bm{B}$) if their difference $\bm{A} - \bm{B}$ is a positive definite (resp.~positive semidefinite) matrix.


Additionally, the standard notation $f(n)=O\left(g(n)\right)$ or
$f(n)\lesssim g(n)$ means that there exists a constant $c>0$ such
that $\left|f(n)\right|\leq c|g(n)|$,  $f(n)\gtrsim g(n)$ means that there exists a constant $c>0$ such
that $|f(n)|\geq c\left|g(n)\right|$, 
and $f(n)\asymp g(n)$ means that there exist constants $c_{1},c_{2}>0$
such that $c_{1}|g(n)|\leq|f(n)|\leq c_{2}|g(n)|$.  

\section{Preliminaries in optimization theory}
\label{sec:preliminaries}

We start by reviewing some basic concepts  and preliminary facts in optimization theory. For simplicity of presentation, this section focuses on an unconstrained problem
\begin{equation}
\label{eq:f_min}
	\text{minimize}_{\bx\in \mathbb{R}^n} \quad f(\bx). 
\end{equation}
The optimal solution, if it exists, is denoted by 
\begin{equation}
	\label{eq:opt-x}
	\bm{x}_{\mathsf{opt}}=\argmin_{\bx\in \mathbb{R}^n} f(\bx).
\end{equation}
When $f(\bx)$ is strictly convex,\footnote{Recall that  $f(\bx)$ is said to be strictly convex if and only if for any $\lambda\in (0,1)$ and $\bx,\by\in\mbox{dom}(f)$, one has
$f(\lambda \bx +(1-\lambda)\by ) < \lambda f(\bx) + (1-\lambda)f(\by)$ unless $\bm{x}=\bm{y}$.
%
} $\bm{x}_{\mathsf{opt}}$ is unique. But it may be non-unique when $f(\cdot)$ is nonconvex. 

\subsection{Gradient descent for locally strongly convex functions}

To solve \eqref{eq:f_min}, arguably the simplest method is (vanilla) gradient descent (GD), which follows the update rule
\begin{equation}
	\label{eq:GD-general}
	\bm{x}_{t+1} = \bm{x}_t - \eta_t \nabla f(\bm{x}_t), \qquad t=0,1,\cdots
\end{equation}
Here, $\eta_t$ is the step size or learning rate at the $t$th iteration, and $\bm{x}_0$ is the initial point. 
This method and its variants are  widely used in practice, partly due to their simplicity and scalability to large-scale problems.

A central question is when GD converges fast to the global minimum $\bm{x}_{\mathsf{opt}}$. As is well-known in the optimization literature, GD is provably convergent at a linear rate when $f(\cdot)$ is (locally) strongly convex and smooth. Here, an algorithm is said to {\em converge linearly} if the error $\|\bm{x}_t-\bm{x}_{\mathsf{opt}}\|_2$ converges to 0 as a geometric series. To formally state this result, we define two concepts that commonly arise in the optimization literature. 
\begin{definition}[$\mathsf{Strong~convexity}$]A twice continuously differentiable
function $f:\mathbb{R}^{n}\mapsto\mathbb{R}$ is said to be $\alpha$-strongly
	convex in a set $\mathcal{B}$ if
\begin{equation}
	\label{eq:def-strong-cvx}
	\nabla^{2}f(\bm{x})\succeq\alpha\bm{I}_{n}, \qquad \forall \bm{x}\in\mathcal{B}.
\end{equation}
\end{definition}
\begin{definition}[$\mathsf{Smoothness}$]A twice continuously
differentiable function $f:\mathbb{R}^{n}\mapsto\mathbb{R}$ is said
to be $\beta$-smooth in a set $\mathcal{B}$ if 
\begin{equation}
	\label{eq:def-smoothness}
	\left\Vert \nabla^{2}f(\bm{x})\right\Vert \leq\beta,\qquad \forall \bm{x}\in\mathcal{B}.
\end{equation}
\end{definition}

With these definitions in place, we have the following standard result (e.g.~\cite{bubeck2015convex}).
\begin{lemma}\label{lem:GD-convergence}
	Suppose that $f$ is $\alpha$-strongly convex and $\beta$-smooth within a local ball $\mathcal{B}_{\zeta}(\bm{x}_{\mathsf{opt}}):=\left\{\bm{x}: \|\bm{x}-\bm{x}_{\mathsf{opt}}\|_2\leq \zeta\right\}$, and that  $\bm{x}_0 \in \mathcal{B}_{\zeta}(\bm{x}_{\mathsf{opt}})$. If $\eta_t\equiv {1}/{\beta}$, then GD obeys 
	\begin{equation*}
			\big\|\bm{x}_{t}-\bm{x}_{\mathsf{opt}} \big\|_{2}\leq\left(1-\frac{\alpha}{\beta}\right)^{t}\big\|\bm{x}_{0}-\bm{x}_{\mathsf{opt}}\big\|_{2},\quad t=0,1,\cdots
	\end{equation*}
\end{lemma}

\begin{proof}[Proof of Lemma~\ref{lem:GD-convergence}]
	The optimality of  $\bm{x}_{\mathsf{opt}}$ indicates that $\nabla f(\bm{x}_{\mathsf{opt}})=\bm{0}$, which allows us to rewrite the GD update rule as
\begin{align*}
\bm{x}_{t+1}-\bm{x}_{\mathsf{opt}} & =\bm{x}_{t}-\eta_t \nabla f\left(\bm{x}_{t}\right)-\left[\bm{x}_{\mathsf{opt}}-\eta_t \nabla f\left(\bm{x}_{\mathsf{opt}}\right)\right] \nonumber\\
	& =\left[\bm{I}_{n}-\eta_t \int_{0}^{1}\nabla^{2}f\left(\bm{x}\left(\tau\right)\right)\mathrm{d}\tau\right] \left(\bm{x}_{t}-\bm{x}_{\mathsf{opt}}\right), \label{gd_iterates}
\end{align*}
where  $\bm{x}\left(\tau\right):=\bm{x}_{\mathsf{opt}}+\tau(\bm{x}_{t}-\bm{x}_{\mathsf{opt}})$. Here, the second line arises from the fundamental theorem of calculus \cite[Chapter XIII, Theorem 4.2]{lang1993real}. If   $\bm{x}_t\in\mathcal{B}_{\zeta}(\bm{x}_{\mathsf{opt}})$, then it is self-evident that $\bm{x}(\tau)\in\mathcal{B}_{\zeta}(\bm{x}_{\mathsf{opt}})$, which combined with the assumption of Lemma \ref{lem:GD-convergence} gives
\[
\alpha\bm{I}_{n}\preceq\nabla^{2}f\left(\bm{x}(\tau)\right)\preceq\beta\bm{I}_{n}, \qquad 0\leq \tau \leq 1.
\]
Therefore, as long as $\eta_t\leq 1/\beta$ (and hence $\|\eta_t \nabla^{2}f(\bm{x}(\tau))\|\leq1$), we have
\[
\bm{0}\preceq\bm{I}_{n}-\eta_t \int_{0}^{1}\nabla^{2}f\left(\bm{x}\left(\tau\right)\right)\mathrm{d}\tau\preceq\left(1-\alpha \eta_t \right)\cdot\bm{I}_{n}.
\]
This together with the sub-multiplicativity of $\|\cdot\|$ yields 
\begin{align*}
\left\Vert \bm{x}_{t+1}-\bm{x}_{\mathsf{opt}}\right\Vert _{2} & \leq\left\Vert \bm{I}_{n}-\eta_t\int_{0}^{1}\nabla^{2}f\left(\bm{x}\left(\tau\right)\right)\mathrm{d}\tau\right\Vert \left\Vert \bm{x}_{t}-\bm{x}_{\mathsf{opt}}\right\Vert _{2}   \\
& \leq\left(1-\alpha\eta_t \right)\left\Vert \bm{x}_{t}-\bm{x}_{\mathsf{opt}}\right\Vert _{2}.
\end{align*}

By setting $\eta_t = 1/\beta$, we arrive at the desired $\ell_{2}$ error
contraction, namely,
\begin{equation}
	\big\|\bm{x}_{t+1}-\bm{x}_{\mathsf{opt}}\|_{2}\leq\left(1-\frac{\alpha}{\beta}\right)\big\|\bm{x}_{t}-\bm{x}_{\mathsf{opt}}\big\|_{2}. \label{eq:error-contraction-GD}  
\end{equation}
A byproduct is: if $\bm{x}_{t}\in \mathcal{B}_{\zeta}(\bm{x}_{\mathsf{opt}})$, then the next iterate $\bm{x}_{t+1}$ also falls in $\mathcal{B}_{\zeta}(\bm{x}_{\mathsf{opt}})$. 
	Consequently,  applying the above argument recursively and recalling the assumption $\bm{x}_0\in\mathcal{B}_{\zeta}(\bm{x}_{\mathsf{opt}})$, we see that all GD iterates remain within $\mathcal{B}_{\zeta}(\bm{x}_{\mathsf{opt}})$. Hence, (\ref{eq:error-contraction-GD}) holds true for all $t$. 
	This immediately concludes the proof.  
\end{proof}
	
This result essentially implies that: to yield $\varepsilon$-accuracy (in a relative sense), i.e.~$ \|\bm{x}_{t}-\bm{x}_{\mathsf{opt}}\|_{2} \leq \varepsilon  \|\bm{x}_{\mathsf{opt}}\|_{2}$,
 the number of iterations required  for GD --- termed the {\em iteration complexity} --- is at most
\[
	O\left(\frac{\beta}{\alpha}\log\frac{\|\bm{x}_0 - \bm{x}_{\mathsf{opt}}\|_2 }{\varepsilon \|\bm{x}_{\mathsf{opt}}\|_2 }\right),
\]
if we initialize GD properly such that $\bx_0$ lies in the local region $\mathcal{B}_{\zeta}(\bm{x}_{\mathsf{opt}})$.
In words, the iteration complexity scales
linearly with the {\em condition number} --- the ratio $\beta/\alpha$
of smoothness to strong convexity parameters. As we shall see, for multiple problems considered herein,  the radius $\zeta$ of this locally strongly convex and smooth ball $\mathcal{B}_{\zeta}(\bm{x}_{\mathsf{opt}})$ can be reasonably large (e.g.~on the same order of $\|\bm{x}_{\mathsf{opt}}\|_2)$.


\begin{figure}
	\centering
	\includegraphics[width=0.3\textwidth]{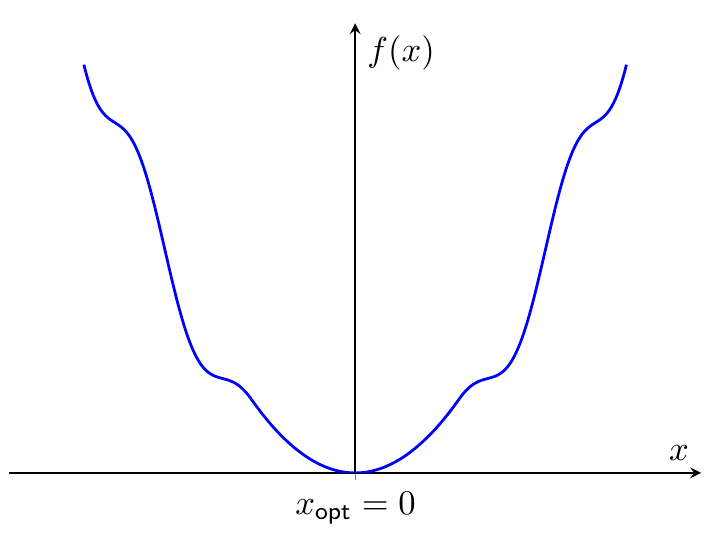}
	\caption{An example of $f(\cdot)$ taken from \cite{chen2015solving}. Here, $f(x)=x^{2}$ if $x\in[-6,6]$ and $f(x)=x^{2}+1.5|x|(\cos(|x|-6)-1)$
if $|x|>6$. It satisfies $\mathsf{RC}(\mu,\lambda,\zeta)$ with some $\mu,\lambda>0$ and $\zeta=\infty$, but is clearly nonconvex. \label{eq:example-RC}}
\end{figure}

\subsection{Convergence  under regularity conditions}

Another condition that has been extensively employed in the literature is the {\em Regularity Condition (RC)} (see e.g.~\cite{candes2015phase,chen2015solving}), which accommodates algorithms beyond vanilla GD, as well as is applicable to possibly nonsmooth functions. Specifically, consider the iterative algorithm
\begin{equation}
	\label{eq:iterative}
	\bm{x}_{t+1} = \bm{x}_t - \eta_t \,\bm{g}(\bm{x}_t)
\end{equation}
for some general mapping $\bm{g}(\cdot):\mathbb{R}^n\mapsto \mathbb{R}^n$. In vanilla GD, $\bm{g}(\bx)=\nabla f(\bx)$, but $\bm{g}(\cdot)$ can also incorporate several variants of GD; see Section~\ref{sec:GD_variants}. The regularity condition is defined as follows. 
\begin{definition}[$\mathsf{Regularity~condition}$]
\label{def:reg-condition}
 $\bm{g}(\cdot)$ is said to obey the regularity condition $\mathsf{RC}( \mu,\lambda, \zeta )$ for some $ \mu,\lambda, \zeta >0 $   if
\begin{equation}\label{eq:RC}
2 \langle\, \bm{g}(\bm{x}), \bm{x}-   \bm{x}_{\mathsf{opt}} \rangle\geq {\mu}\| \bm{g}(\bm{x})\|^2_2 + \lambda \left\| \bm{x}-   \bm{x}_{\mathsf{opt}} \right\|_2^2
\end{equation}
for all $\bm{x}\in\mathcal{B}_{\zeta}(\bm{x}_{\mathsf{opt}})$.
\end{definition}
This condition basically implies that at any feasible point $\bm{x}$, the associated negative search direction $ \bm{g}(\bm{x})$ is positively correlated with the  error $\bm{x} - \bm{x}_{\mathsf{opt}}$, and hence the update rule   \eqref{eq:iterative} --- in conjunction with a sufficiently small step size --- drags the current point closer to the global solution. It follows from the Cauchy-Schwarz inequality that one must have $\mu\lambda\leq 1$. 


It is worth noting that this condition does not require  $\bm{g}(\cdot)$ to be differentiable. Also, 
when $\bm{g}(\cdot)=\nabla f(\cdot)$, it does not require $f(\cdot)$ to be  convex within $\mathcal{B}_{\zeta}(\bm{x}_{\mathsf{opt}})$; see Fig.~\ref{eq:example-RC} for an example. Instead, the regularity condition can be viewed as a combination of  smoothness and  ``{\em one-point strong convexity}'' (as the condition is stated w.r.t.~a single point $\bx_{\mathsf{opt}}$) defined as follows 
\begin{equation} 
	\label{eq:one_point_convexity}
	 f(\bm{x}_{\mathsf{opt}}) - f(\bm{x}) \geq \langle \nabla f(\bm{x}), \bm{x}_{\mathsf{opt} }  - \bm{x}  \rangle + \frac{\alpha}{2} \| \bm{x} - \bm{x}_{\mathsf{opt}} \|_2^2
\end{equation}
for all $\bm{x}\in \mathcal{B}_{\zeta}(\bm{x}_{\mathsf{opt}})$. To see this, observe that
\begin{align}
f(\bm{x}_{\mathsf{opt}})- f(\bm{x}) & \leq f \Big(\bm{x} - \frac{1}{\beta}\nabla f(\bm{x}) \Big)- f(\bm{x})  \nonumber \\
& \leq \Big\langle \nabla f(\bm{x}), - \frac{1}{\beta}\nabla f(\bm{x}) \Big\rangle + \frac{\beta}{2} \Big\| \frac{1}{\beta}\nabla f(\bm{x}) \Big\|_2^2 \nonumber \\
&= -\frac{1}{2\beta} \left\| \nabla f(\bm{x}) \right\|_2^2 , \label{eq:smoothness_consequence}
\end{align}
where the second line follows from an equivalent definition of the smoothness condition \cite[Theorem 5.8]{beck2017first}.
Combining \eqref{eq:one_point_convexity} and \eqref{eq:smoothness_consequence} arrives at the regularity condition with $\mu= 1/\beta$ and $\lambda = \alpha$.

Under the regularity condition, the iterative algorithm  \eqref{eq:iterative} converges linearly with suitable initialization.
\begin{lemma}
	\label{lem:convergence-RC}
	Under $\mathsf{RC}( \mu,\lambda, \zeta )$, the iterative algorithm \eqref{eq:iterative} with $\eta_t\equiv \mu$ and $\bm{x}_0\in \mathcal{B}_{\zeta}(\bm{x}_{\mathsf{opt}})$ obeys
	$$\big\|\bm{x}_{t}-\bm{x}_{\mathsf{opt}}\|^2_{2}\leq\left(1- \mu\lambda \right)^{t}\big\|\bm{x}_{0}-\bm{x}_{\mathsf{opt}}\big\|^2_{2},\quad t=0,1,\cdots $$
\end{lemma}
\begin{proof}[Proof of Lemma \ref{lem:convergence-RC}]
Assuming that $\bm{x}_t\in \mathcal{B}_{\zeta}(\bm{x}_{\mathsf{opt}})$, we can obtain
\begin{flalign*}
\| \bm{x}_{t+1}- \bx_{\mathsf{opt}} \|_2^2  
&= \left\| \bx_t -\eta_t \bm{g}(\bx_t)-\bx_{\mathsf{opt}} \right\|_2^2\\
&=\| \bx_t -\bx_{\mathsf{opt}} \|_2^2+\eta_t^2 \| \bm{g} (\bx_t)\|_2^2-2\eta_t \left\langle\bx_t -\bx_{\mathsf{opt}}, \, \bm{g} (\bx_t )\right\rangle\\
&\le \|\bx_t -\bx_{\mathsf{opt}} \|_2^2+\eta_t^2 \| \bm{g}(\bx_t )\|_2^2  -\eta_t\big(\mu\| \bm{g}(\bx_t  )\|_2^2+ \lambda \left\|\bx_t  -\bx_{\mathsf{opt}} \right\|_2^2\big)\\
&=(1-\eta_t\lambda) \| \bx_t  -\bx_{\mathsf{opt}} \|_2^2 + \eta_t \left( \eta_t - \mu\right) \| \bm{g}(\bx_t  )\|_2^2 \\
&\leq (1-\eta_t\lambda)  \| \bx_t  -\bx_{\mathsf{opt}} \|_2^2 ,
\end{flalign*}
where the first inequality comes from $\mathsf{RC}(\mu,\lambda,\zeta)$, and the last line arises if $0\leq \eta_t \leq \mu$. By setting $\eta_t=\mu$, we arrive at
\begin{equation}
\big\|\bm{x}_{t+1}-\bm{x}_{\mathsf{opt}}\|^2_{2}\leq\left(1- \mu\lambda \right)\big\|\bm{x}_{t}-\bm{x}_{\mathsf{opt}}\big\|_{2}^2,
\end{equation}
which also shows that $\bm{x}_{t+1}\in \mathcal{B}_{\zeta}(\bm{x}_{\mathsf{opt}})$. The claim then follows by induction under the assumption that $\bm{x}_{0}\in \mathcal{B}_{\zeta}(\bm{x}_{\mathsf{opt}})$. 
\end{proof}

In view of Lemma \ref{lem:convergence-RC}, the iteration complexity to reach $\varepsilon$-accuracy (i.e.~$\|\bm{x}_t-\bm{x}_{\mathsf{opt}}\|_2\leq \varepsilon \|\bm{x}_{\mathsf{opt}}\|_2$) is at most
\[
	O\left(\frac{1}{\mu\lambda}\log\frac{ \|\bm{x}_0-\bm{x}_{\mathsf{opt}}\|_2  }{\varepsilon \|\bm{x}_{\mathsf{opt}}\|_2 }\right),
\]
as long as a suitable initialization is provided. 

\subsection{Critical points}

An iterative algorithm like gradient descent often converges to one of its fixed points \cite{nesterov2013introductory}. For gradient descent, the associated fixed points are (first-order) {\em critical points} or {\em stationary points} of the loss function, defined as follows.
\begin{definition}[$\mathsf{First\text{-}order~critical~points}$]
	\label{def:1st-criticals}
	A first-order critical point (stationary point) $\bx$ of $f(\cdot)$ is any point that satisfies
	$$ \nabla f(\bx) = \bm{0} .$$ 
\end{definition} 
\noindent Moreover, we call a point $\bx$ an {\em $\varepsilon$-first-order critical point}, for some $\varepsilon>0$, if it satisfies $\|\nabla f(\bx)\|_2\leq \varepsilon$.

A critical point can be a local minimum, a local maximum, or a saddle point of $f(\cdot)$, depending on the curvatures at\,/\,surrounding the point. Specifically, denote by $\nabla^2 f(\bx)$  the Hessian matrix at $\bx$, and let $\lambda_{\min}(\nabla^2 f(\bx))$ be its minimum eigenvalue. Then for any first-order critical point $\bx$:
\begin{itemize}
	\itemsep0.1em
	\item if $ \nabla^2 f(\bx) \prec \bm{0}$, then $\bx$ is a local maximum;
	\item if $ \nabla^2 f(\bx)) \succ \bm{0}$, then $\bx$ is a local minimum;
\item $\lambda_{\min}(\nabla^2 f(\bx))=0$, then $\bx$ is either a local minimum or a degenerate saddle point;
\item $\lambda_{\min}(\nabla^2 f(\bx))<0$, then $\bx$ is a {\em strict} saddle point.
\end{itemize}

Another useful concept  is second-order critical points, defined as follows.
\begin{definition}[$\mathsf{Second\text{-}order~critical~points}$]
	\label{def:2nd-criticals}
	A point $\bm{x}$ is said to be a second-order critical point (stationary point) if  $\nabla f(\bx)= \bm{0}$ and $\nabla^2 f(\bx) \succeq \bm{0}$.  
\end{definition} 
\noindent Clearly, second-order critical points do not encompass local maxima and strict saddle points, and, as we shall see, are of particular interest for the nonconvex problems considered herein. Since we are interested in minimizing the loss function, we do not distinguish local maxima and strict saddle points.



\section{A warm-up example: rank-1 matrix factorization}
\label{sec:noncvx_eg}

For pedagogical reasons, we begin with a self-contained study of a simple nonconvex matrix factorization problem, demonstrating  local convergence in the basin of attraction in Section~\ref{sec:local_convergence_mf} and benign global landscape in Section~\ref{sec:landscape_mf}. 
The analysis in this section requires only elementary calculations.

\begin{figure}[t]
	\begin{center}
		\hspace{-0.1in}\includegraphics[width=0.3\textwidth]{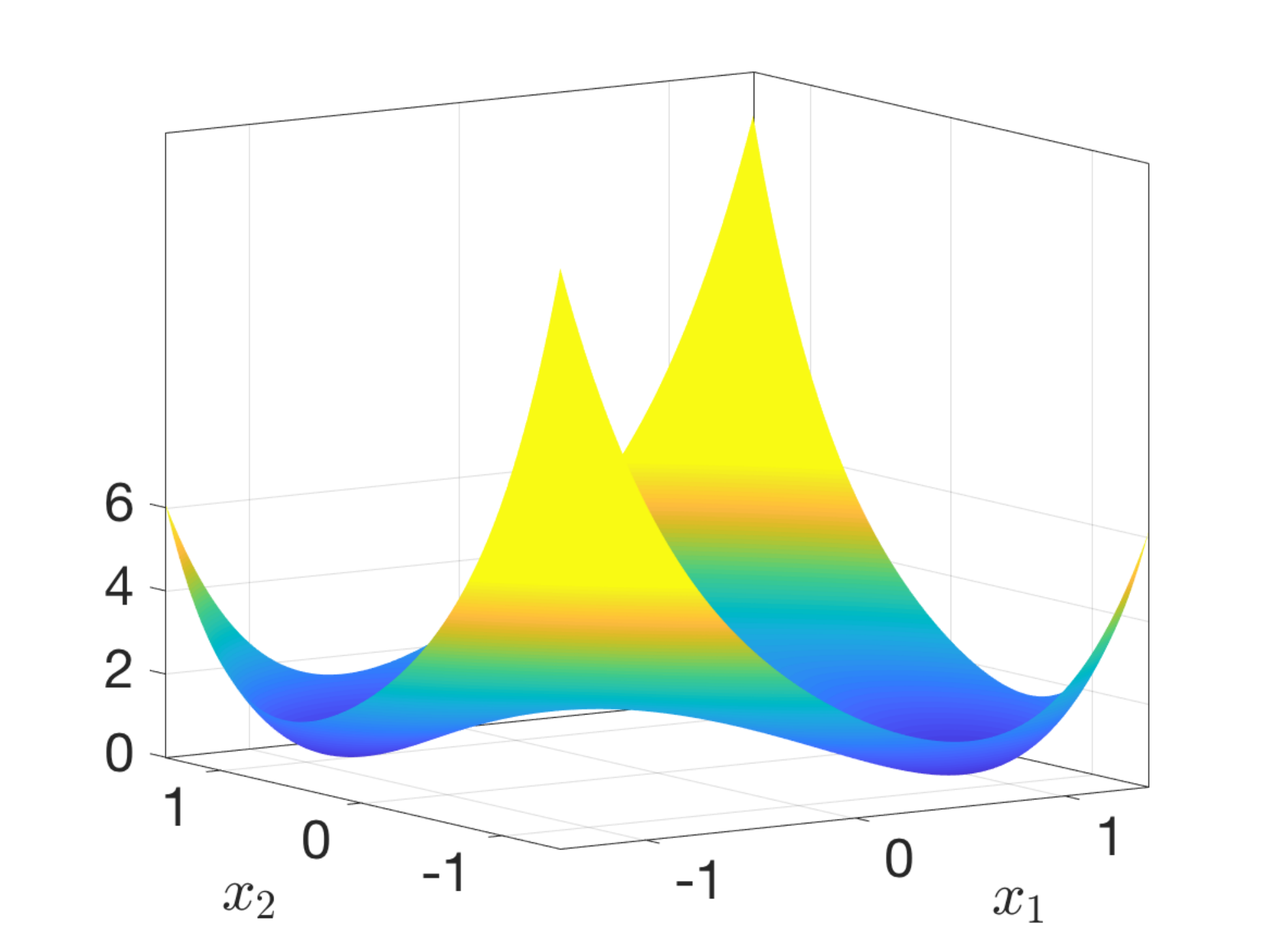}
	\end{center}
	\caption{Illustration of the function $f(\bx)=\frac{1}{4}\|\bm{x}\bm{x}^{\top}-\bm{M}\|_{\mathrm{F}}^{2}$
	, where $\bx=[x_{1},x_{2}]^{\top}$ and $\bm{M}=${\tiny$\left[\protect\begin{array}{cc}
1 & -0.5\protect\\
-0.5 & 1
	\protect\end{array}\right]$}.  
	\label{fig:strict_saddle2}}
\end{figure}

Specifically, consider a positive semidefinite  matrix $\bm{M}\in\mathbb{R}^{n\times n}$ (which is not necessarily low-rank) with eigenvalue decomposition $\bM =\sum_{i=1}^n \lambda_i \bm{u}_i\bm{u}_i^{\top}$. We assume  throughout this section that there is a gap between the 1st and 2nd largest eigenvalues
\begin{equation}
	\label{assumption-rank1}
	\lambda_1 >\lambda_2\geq \ldots \geq \lambda_n \geq 0.
\end{equation}
The aim is to find the best rank-$1$ approximation  of $\bm{M}$. Clearly,  this can be posed as the following problem:\footnote{Here, the preconstant $1/4$ is introduced for the purpose of normalization  and does not affect the solution.}
\begin{equation} 
	\label{eq:rank_one_approx}
	\underset{\bm{x}\in\mathbb{R}^{n}}{\text{minimize}} \quad f(\bm{x}) = \frac{1}{4} \| \bm{x}\bm{x}^{\top}-\bm{M} \|_{\mathrm{F}}^2  ,
\end{equation}
%
%
%
where $f(\cdot)$ is a degree-four polynomial and highly nonconvex.  
The solution to \eqref{eq:rank_one_approx} can be expressed in closed form as the scaled leading eigenvector $\pm \sqrt{\lambda_1} \bm{u}_1$. See Fig.~\ref{fig:strict_saddle2} for an illustration of the function $f(\bx)$ when $\bx\in\mathbb{R}^2$. This problem stems from interpreting principal component analysis (PCA) from an optimization perspective, which has a long history in the literature of (linear) neural networks and unsupervised learning; see for example \cite{baldi1989neural,baldi1995learning,yang1995projection,li2016symmetry,zhu2017global,zhu2018global,hauser2018pca}.

We attempt to minimize the nonconvex function $f(\cdot)$ directly in spite of nonconvexity. This problem, though simple, plays a critical role in understanding the success of nonconvex optimization,  since  several important nonconvex low-rank estimation problems can be regarded as randomized versions or extensions of this problem. 

\subsection{Local linear convergence of gradient descent}
\label{sec:local_convergence_mf}

To begin with, we demonstrate that gradient descent, when initialized at a point sufficiently close to the true optimizer (i.e.~$\pm\sqrt{\lambda_1} \bm{u}_1$), is guaranteed to converge fast.  
\begin{theorem}\label{thm:rank1-local}
	Consider the problem \eqref{eq:rank_one_approx}. Suppose  $\eta_t\equiv\frac{1}{4.5\lambda_1}$ and $\|\bm{x}_0 - \sqrt{\lambda_1}\bm{u}_1\|_{2}\leq\frac{\lambda_{1}-\lambda_{2}}{15\sqrt{\lambda_{1}}} $. Then the GD iterates \eqref{eq:GD-general} obey 
\[
	\big\|\bm{x}_{t}-\sqrt{\lambda_{1}}\bm{u}_{1}\big\|_{2}\leq\left(1-\frac{\lambda_{1}-\lambda_{2}}{18\lambda_{1}}\right)^{t}\big\|\bm{x}_{0}-\sqrt{\lambda_{1}}\bm{u}_{1}\big\|_{2}, ~~ t\geq 0.
\]
\end{theorem}
\begin{remark} 
By symmetry, Theorem \ref{thm:rank1-local} continues to hold if $\bm{u}_1$ is replaced by $-\bm{u}_1$.
\end{remark}
In a nutshell, Theorem \ref{thm:rank1-local} establishes linear convergence of GD for rank-1 matrix factorization, where the convergence rate largely depends upon the eigen-gap (relative to the largest eigenvalue). This is a ``local'' result, assuming that a suitable initialization is present in the basin of attraction $\mathcal{B}_{\zeta}(\sqrt{\lambda_1}\bu_1)$. Its radius, which is not optimized in this theorem, is given by $\zeta = \frac{\lambda_{1}-\lambda_{2}}{15\lambda_{1}}\| \sqrt{\lambda_1} \bu_1\|_2$. This also depends on the relative eigen-gap $(\lambda_1-\lambda_2)/\lambda_1$.  

\begin{proof}[Proof of Theorem \ref{thm:rank1-local}]
The proof mainly consists of  showing that $f(\cdot)$ is locally strongly convex and smooth, which allows us to invoke Lemma \ref{lem:GD-convergence}. The gradient and the Hessian of $f(\bm{x})$ are given respectively by 
\begin{align}
	\nabla f(\bx) & = (\bx\bx^{\top}-\bM) \bx;  \label{eq:grad-rank1} \\
	\nabla^2 f(\bx) & = \|\bm{x}\|_2^2 \bm{I}_n + 2\bm{x}\bm{x}^{\top} - \bm{M}. \label{eq:Hess-rank1}
\end{align}
For notational simplicity, let $\bm{\Delta}:=\bm{x}-\sqrt{\lambda_{1}}\bm{u}_{1}$. A little algebra yields that if $\|\bm{\Delta}\|_2\leq \frac{\lambda_{1}-\lambda_{2}}{15\sqrt{\lambda_{1}}}$, then
\begin{align}
	& \|\bm{\Delta}\|_{2}\leq\|\bm{x}\|_{2} , ~~\myquadinv \|\bm{\Delta}\|_{2}\|\bm{x}\|_{2}\leq (\lambda_1 - \lambda_2) /12; \label{eq:condition-Delta-1}
\\
	& \qquad  \lambda_1 - 0.25(\lambda_1-\lambda_2) \leq \|\bm{x}\|_{2}^2 \leq 1.15\lambda_1.  \label{eq:condition-Delta-2}
\end{align}
	
We start with the smoothness condition. The triangle inequality gives
\begin{align*}
\left\Vert \nabla^{2}f(\bm{x})\right\Vert  & \leq\left\Vert \|\bm{x}\|_{2}^{2}\bm{I}_{n}\right\Vert +\left\Vert 2\bm{x}\bm{x}^{\top}\right\Vert +\left\Vert \bm{M}\right\Vert \\
 & =3\|\bm{x}\|_{2}^{2}+\lambda_{1} < 4.5\lambda_{1},
\end{align*}
where the last line follows from \eqref{eq:condition-Delta-2}.

	Next, it comes from the definition of $\bm{\Delta}$ and \eqref{eq:condition-Delta-1} that
\begin{align*}
\bm{x}\bm{x}^{\top} & =\lambda_{1}\bm{u}_{1}\bm{u}_{1}^{\top}+\bm{\Delta}\bm{x}^{\top}+\bm{x}\bm{\Delta}^{\top}-\bm{\Delta}\bm{\Delta}^{\top}\\
 & \succeq\lambda_{1}\bm{u}_{1}\bm{u}_{1}^{\top}-3\|\bm{\Delta}\|_{2}\|\bm{x}\|_{2}\bm{I}_n\\
 & \succeq\lambda_{1}\bm{u}_{1}\bm{u}_{1}^{\top}- 0.25\left(\lambda_{1}-\lambda_{2}\right)\bm{I}_n.
\end{align*}
Substitution into (\ref{eq:Hess-rank1}) yields a strong convexity lower bound:
\begin{align*}
\nabla^{2}f(\bm{x}) & =\|\bm{x}\|_{2}^{2}\bm{I}_{n}+2\bm{x}\bm{x}^{\top}-\lambda_{1}\bm{u}_{1}\bm{u}_{1}^{\top}-\sum\nolimits_{i=2}^{n}\lambda_{i}\bm{u}_{i}\bm{u}_{i}^{\top}\\
 & \succeq\left(\|\bm{x}\|_{2}^{2}+\lambda_{1}- 0.5\left(\lambda_{1}-\lambda_{2}\right)\right)\bm{u}_{1}\bm{u}_{1}^{\top} \mylinebreak
 \myalign \myquad\myquad +\sum\nolimits_{i=2}^{n}\left(\|\bm{x}\|_{2}^{2}- 0.5\left(\lambda_{1}-\lambda_{2}\right)-\lambda_{i}\right)\bm{u}_{i}\bm{u}_{i}^{\top}\\
 & \succeq\left(\|\bm{x}\|_{2}^{2}- 0.5\left(\lambda_{1}-\lambda_{2}\right)-\lambda_{2}\right) \underset{=\bm{I}_n}{\underbrace{ \sum\nolimits_{i=1}^{n}\bm{u}_{i}\bm{u}_{i}^{\top} }}\\
 & \succeq 0.25(\lambda_{1}-\lambda_{2}) \bm{I}_n,
\end{align*}
where the last inequality is an immediate consequence of  \eqref{eq:condition-Delta-2}.

In summary, for all $\bm{x}$ obeying $\|\bm{\Delta}\|_2\leq \frac{\lambda_{1}-\lambda_{2}}{15\sqrt{\lambda_{1}}}$, one has
\[
 0.25(\lambda_{1}-\lambda_{2})\bm{I}_n \preceq \nabla^2 f(\bm{x})\preceq4.5\lambda_1 \bm{I}_n.
\]
Applying Lemma \ref{lem:GD-convergence}  establishes the claim. 
%
\end{proof}

The question then comes down to whether one can secure an initial guess of this quality. One popular approach is spectral initialization, obtained by computing the leading eigenvector of $\bm{M}$. For this simple problem, this already yields a solution of arbitrary accuracy. As it turns out,  such a spectral initialization approach is particularly useful when dealing with noisy and incomplete measurements. We refer the readers to Section \ref{sec:spectral-initialization} for detailed discussions of spectral initialization methods. 
%



\subsection{Global optimization landscape}\label{sec:landscape_mf}

We then move on to examining the optimization landscape of this simple problem. In particular, what kinds of critical points does $f(\bx)$ have? This is addressed as follows; see also \cite[Section 3.3]{li2018non}. 

\begin{theorem}
	\label{thm:rank1-global}
	Consider the objective function \eqref{eq:rank_one_approx}. All local minima of $f(\cdot)$ are global optima. The rest of the critical points are either local maxima or strict saddle points. 
\end{theorem}

\begin{proof}[Proof of Theorem \ref{thm:rank1-global}]
Recall that  $\bm{x}$ is a critical point if and only if $\nabla f(\bx) = (\bx\bx^{\top}-\bM) \bx = \bm{0} $, or equivalently,  
$$ \bM\bx = \|\bx\|_2^2  \bx. $$
As a result, $\bx$ is a critical point if either it aligns with an eigenvector of $\bM$ or $\bx = \bm{0}$. Given that the eigenvectors of $\bm{M}$ obey $ \bM\bu_i = \lambda_i  \bu_i $, by properly adjusting the scaling, we determine the set of critical points as follows
$$ \mathsf{critical}\text{-}\mathsf{points} = \{\bm{0} \} \cup \{  \pm \sqrt{\lambda_k } \bu_k: k=1,\ldots, n\}.$$



To further categorize the critical points, we need to examine the associated Hessian matrices as given by \eqref{eq:Hess-rank1}. 
Regarding the critical points $\pm \sqrt{\lambda_k } \bu_k$, we have
\begin{align*}
 \myalign \nabla^2 f\big( \pm \sqrt{\lambda_k } \bu_k\big) 
   \myaligninv = \lambda_k \bI_n + 2\lambda_k \bu_k \bu_k^{\top}  - \bM \\
 & \myquad =  \lambda_k \sum_{i=1}^n \bu_i\bu_i^{\top} + 2\lambda_k \bu_k \bu_k^{\top}  - \sum_{i=1}^n \lambda_i \bu_i\bu_i^{\top} \\
 & \myquad =  \sum_{i:i\neq k} (\lambda_k - \lambda_i ) \bu_i\bu_i^{\top}  + 2\lambda_k \bu_k\bu_k^{\top}.
\end{align*}
%
We can then categorize them as follows:
\begin{enumerate}
\itemsep0.5em
\item With regards to the points $\{\pm \sqrt{\lambda_1 } \bu_1\}$, one has $$\nabla^2 f\big( \pm \sqrt{\lambda_1 } \bu_1\big)\succ \bm{0},$$ and hence  they are (equivalent) local minima of $f(\cdot)$; 
\item For the points $\{\pm \sqrt{\lambda_k } \bu_k\}_{k=2}^n$, one has 
	\begin{align*}		
		\lambda_{\min}\big(\nabla^2 f\big( \pm \sqrt{\lambda_k } \bu_k\big)\big) &< 0, \\  \lambda_{\max}\big(\nabla^2 f\big( \pm \sqrt{\lambda_k } \bu_k\big)\big) &>0,
	\end{align*}
	and therefore they are strict saddle points of $f(\cdot)$.
\item Finally, the critical point at the origin satisfies
\begin{align*}
	\nabla^2 f( \bm{0}) = -\bm{M} \preceq \bm{0},  
\end{align*}
and is hence either a local maxima (if $\lambda_n>0$) or a strict saddle point (if $\lambda_n=0$).
\end{enumerate} 
\end{proof}

This result reveals the benign geometry of the problem \eqref{eq:rank_one_approx} amenable to optimization. All undesired fixed points of gradient descent are strict saddles, which have negative directional curvature and may not be difficult to escape or avoid.

\section{Formulations of a few canonical problems}
\label{sec:examples}

For an article of this length, it is impossible to cover all nonconvex statistical problems of interest.  
Instead, we decide to focus on a few concrete and fundamental matrix factorization problems.  This section presents formulations of several such examples that will be visited multiple times throughout this article. Unless otherwise noted,  the assumptions made in this section (e.g.~restricted isometry for matrix sensing, Gaussian design for phase retrieval) will be  imposed throughout the rest of the paper. 


\subsection{Matrix sensing}
\label{sec:matrix_sensing}

Suppose that we are given a set of $m$ measurements of $\bm{M}_{\star}\in\mathbb{R}^{n_1\times n_2}$ of the form
\begin{equation}
	\label{eq:matrix-sensing-samples}
	y_i = \langle \bm{A}_i, \bm{M}_{\star} \rangle, \qquad 1\leq i\leq m,
\end{equation}
where $\{\bm{A}_i\in \mathbb{R}^{n_1\times n_2}\}$ is a collection of sensing matrices known {\em a priori}. We are asked to recover $\bm{M}_{\star}$ --- which is assumed to be of rank $r$ --- from these linear matrix equations \cite{recht2010guaranteed,candes2011tight}. 

When $\bm{M}_{\star}=\bm{X}_{\star}\bm{X}_{\star}^{\top}\in \mathbb{R}^{n\times n}$ is positive semidefinite with $\bm{X}_{\star} \in \mathbb{R}^{n\times r}$, this can be cast as solving the least-squares problem   
\begin{align}
	\underset{\bm{X}\in \mathbb{R}^{n\times r}}{\text{minimize}}\,\, \myquadinv f(\bm{X})=\frac{1}{4m}\sum_{i=1}^{m}\big(\langle\bm{A}_{i},\bm{X}\bm{X}^{\top}\rangle- y_i \big)^{2}.   
	\label{eq:min-matrix-sensing-rank-r}	
\end{align}
Clearly, we cannot distinguish $\bm{X}_{\star}$ from $\bm{X}_{\star}\bm{H}$ for any orthonormal matrix $\bm{H}\in \mathcal{O}^{r\times r}$, as they correspond to the same low-rank matrix $\bm{X}_{\star}\bm{X}_{\star}^{\top}=\bm{X}_{\star}\bm{H}\bm{H}^{\top}\bm{X}_{\star}^{\top}$. This simple fact implies that there exist multiple global optima for \eqref{eq:min-matrix-sensing-rank-r}, a phenomenon that holds for most  problems discussed herein. 

For the general  case where  $\bm{M}_{\star}=\bm{L}_{\star}\bm{R}_{\star}^{\top}$ with $\bm{L}_{\star} \in \mathbb{R}^{n_1\times r}$ and $\bm{R}_{\star} \in \mathbb{R}^{n_2\times r}$, we wish to minimize   
\begin{align}
	\underset{\bm{L}\in \mathbb{R}^{n_1\times r},\bm{R}\in \mathbb{R}^{n_2\times r}}{\text{minimize}} \myquadinv f(\bm{L},\bm{R})=\frac{1}{4m}\sum_{i=1}^{m}\big(\langle\bm{A}_{i},\bm{L}\bm{R}^{\top}\rangle- y_i \big)^{2}.   
	\label{eq:min-matrix-sensing-rank-r-asymmetric}	
\end{align}
Similarly, we cannot distinguish $(\bm{L}_{\star},\bR_{\star})$ from $(\bm{L}_{\star}\bm{H},\bR_{\star}(\bm{H}^{\top})^{-1})$ for any {\em invertible} matrix $\bm{H}\in \mathbb{R}^{r\times r}$, as $\bm{L}_{\star}\bm{R}_{\star}^{\top}=\bm{L}_{\star}\bm{H}\bm{H}^{-1}\bm{R}_{\star}^{\top}$. Throughout the paper, we denote by 
\begin{equation}
	\label{eq:ground_truth_factors}
	\bL_{\star }:=\bm{U}_{\star}\bm{\Sigma}_{\star}^{1/2}\quad \text{and} \quad  \bR_{\star} :=\bm{V}_{\star}\bm{\Sigma}_{\star}^{1/2}
\end{equation}
the true low-rank factors, where $\bm{M}_{\star}=\bm{U}_{\star}\bm{\Sigma}_{\star}\bm{V}_{\star}^\top$ stands for its singular value decomposition.

In order to make the problem well-posed, we need to make proper assumptions on the sensing operator $\mathcal{A}:\mathbb{R}^{n_1\times n_2}\mapsto\mathbb{R}^{m}$ defined by:
\begin{equation}
	\mathcal{A}(\bm{X}):=\big[m^{-1/2}\langle\bm{A}_{i},\bm{X}\rangle\big]_{1\leq i\leq m}. \label{eq:defn-A-sensing}
\end{equation}
%
A useful property for the sensing operator that enables tractable algorithmic solutions is the {\em Restricted Isometry Property (RIP)}, which says that the operator preserves approximately the Euclidean norm of the input matrix when restricted to the set of low-rank matrices. More formally:  
\begin{definition}[\textsf{Restricted isometry property} \cite{candes2008restricted}]
	\label{defn:RIPs}
	An operator $\mathcal{A}:\mathbb{R}^{n_1\times n_2}\mapsto\mathbb{R}^{m}$ is said to satisfy the $r$-RIP with RIP constant $\delta_r<1$ if
	\begin{equation}
(1-\delta_{r})\|\bm{X}\|_{\mathrm{F}}^2\leq\|\mathcal{A}(\bm{X})\|_{2}^2\leq(1+\delta_{r})\|\bm{X}\|_{\mathrm{F}}^2 \label{eq:defn-RIPr}
\end{equation}
	holds simultaneously for all $\bm{X}$ of rank at most $r$. 
\end{definition}

\noindent As an immediate consequence, the inner product between two low-rank matrices is also nearly preserved if $\mathcal{A}$ satisfies the RIP. 
Therefore, $\mathcal{A}$ behaves approximately like an isometry when restricting its operations over low-rank matrices.
\begin{lemma}[\cite{candes2008restricted}]
	\label{lemmq:RIP-cross}	
If an operator $\mathcal{A}$ satisfies the $2r$-RIP with RIP constant $\delta_{2r}<1$, then 
\begin{equation}
	\big|\big\langle\mathcal{A}(\bm{X}),\mathcal{A}(\bm{Y})\big\rangle-\langle\bm{X},\bm{Y}\rangle\big|\leq\delta_{2r}\|\bm{X}\|_{\mathrm{F}}\|\bm{Y}\|_{\mathrm{F}}\label{eq:defn-RIP2r}
\end{equation}
	holds simultaneously for all $\bm{X}$ and $\bm{Y}$ of rank at most $r$. 
\end{lemma}

Notably, many random sensing designs are known to satisfy the RIP with high probability, with one remarkable example given below. 
\begin{fact}[$\mathsf{RIP~for~Gaussian~matrices}$ \cite{recht2010guaranteed,candes2011tight}] 
\label{fact:gaussian_rip}
	If the entries of $\bm{A}_i$ are i.i.d.~Gaussian entries $\mathcal{N}(0,1)$, then $\mathcal{A}$ as defined in \eqref{eq:defn-A-sensing} satisfies the $r$-RIP with RIP constant $\delta_{r}$ with high probability as soon as $m \gtrsim   (n_1+n_2)r/\delta_r^2 $.
\end{fact}

\subsection{Phase retrieval and quadratic sensing}
\label{sec:PR}

Imagine that we have access to $m$ quadratic measurements of a rank-1 matrix $\bm{M}_{\star}:=\bm{x}_{\star}\bm{x}_{\star}^{\top}\in \mathbb{R}^{n\times n}$:
\begin{equation}
	\label{eq:PR-samples}
	y_i = (\bm{a}_i^{\top} \bm{x}_{\star})^2 = \bm{a}_i^{\top} \bm{M}_{\star} \bm{a}_i,  \qquad 1\leq i\leq m,
\end{equation}
where $\bm{a}_i\in \mathbb{R}^n$ is the design vector known {\em a priori}. How can we reconstruct $\bm{x}_{\star}\in \mathbb{R}^n$ --- or equivalently, $\bm{M}_{\star}=\bm{x}_{\star}\bm{x}_{\star}^{\top}$ --- from this collection of quadratic equations about $\bm{x}_{\star}$? 
This problem, often dubbed as {\em phase retrieval}, arises for example in X-ray
crystallography, where one needs to recover a specimen based on intensities of the diffracted waves scattered by the object \cite{fienup1982phase,shechtman2015phase,candes2012phaselift,jaganathan2015phase}. Mathematically, the problem can be posed as finding a solution to the following program
%
%
\begin{equation}
	\underset{\bm{x}\in\mathbb{R}^{n}}{\text{minimize}}\,\,\myquadinv f(\bm{x})=\frac{1}{4m}\sum_{i=1}^{m}\big((\bm{a}_{i}^{\top}\bm{x})^{2}- y_i \big)^{2} .
	\label{eq:min-PR}
\end{equation}

More generally, consider the {\em quadratic sensing} problem, where we collect $m$ quadratic measurements of a rank-$r$ matrix $\bm{M}_{\star}:=\bm{X}_{\star}\bm{X}_{\star}^{\top}$ with $\bm{X}_{\star}\in\mathbb{R}^{n\times r}$:
\begin{equation}
	y_i =  \| \bm{a}_i^{\top} \bm{X}_{\star} \|_2^2 = \bm{a}_i^{\top} \bm{M}_{\star} \bm{a}_i,  \qquad 1\leq i\leq m.
\end{equation}
This subsumes phase retrieval as a special case, and comes up in applications such as covariance sketching  for  streaming data \cite{chen2015exact,cai2015rop}, and phase space tomography under the name {\em coherence retrieval} \cite{tian2012experimental,bao2018coherence}. Here, we wish to solve
%
%
\begin{equation}
	\underset{\bm{X}\in\mathbb{R}^{n\times r}}{\text{minimize}}\,\,\myquadinv f(\bm{X})=\frac{1}{4m}\sum_{i=1}^{m}\big(\|\bm{a}_{i}^{\top}\bm{X}\|_2^{2}- y_i \big)^{2} .
	\label{eq:min-PR_lowrank}
\end{equation}
Clearly, this is equivalent to the matrix sensing problem by taking $\bm{A}_i = \bm{a}_i \bm{a}_i^{\top}$.   Here and throughout, we assume i.i.d.~Gaussian design as follows,  a tractable model that has been extensively studied recently. 

\begin{assumption}[$\mathsf{Gaussian~design~in~phase~retrieval}$]  Suppose that the design vectors $\bm{a}_i \overset{\text{i.i.d.}}{\sim} \mathcal{N}(\bm{0},\bm{I}_n)$.  \end{assumption}

\subsection{Matrix completion} 
\label{sec:MC}

Suppose we observe partial entries of a low-rank matrix $\bm{M}_{\star}\in\mathbb{R}^{n_1\times n_2}$ of rank $r$, indexed by the sampling location set $\Omega$.  It is convenient to introduce a projection operator $\mathcal{P}_{\Omega}:\mathbb{R}^{n_1\times n_2} \mapsto \mathbb{R}^{n_1\times n_2}$ such that for an input matrix $\bm{M}\in\mathbb{R}^{n_1\times n_2}$, 
\begin{equation}
\big(\mathcal{P}_{\Omega}(\bm{M})\big)_{i,j}=\begin{cases}
M_{i,j},\quad & \text{if }(i,j)\in\Omega,\\
0, & \text{else}.
\end{cases}
	\label{defn:Pomega}
\end{equation}
%
%
The matrix completion problem then boils down to recovering $\bm{M}_{\star}$ from $\mathcal{P}_{\Omega}(\bm{M}_{\star})$ (or equivalently, from the partially observed entries of $\bm{M}_{\star}$) \cite{candes2009exact,CanTao10}. 
This arises in numerous scenarios; for instance, in collaborative filtering, we may want to predict the preferences of all users about a collection of movies, based on partially revealed ratings from the users.  Throughout this paper, we adopt the following random sampling model:

\begin{assumption}[$\mathsf{Random~sampling~in~matrix~completion}$] Each entry is observed independently with probability $0<p\leq 1$, i.e.
\begin{equation}
	(i,j)\in\Omega \qquad \text{ independently with probability } p.  \label{eq:random-sample-MC}
\end{equation}
\end{assumption}

%
For the positive semidefinite case where $\bm{M}_{\star}=\bm{X}_{\star}\bm{X}_{\star}^{\top}$, the task can be cast as solving
\begin{equation}
	\underset{\bm{X}\in\mathbb{R}^{n\times r}}{\text{minimize}}~~\myquadinv f(\bm{X})
	= \frac{1}{4p}\left\| \mathcal{P}_{\Omega}( \bm{X}\bm{X}^{\top} - \bM_{\star} ) \right\|_{\mathrm{F}}^{2} .
	\label{eq:MC-empirical-risk}
\end{equation}
When it comes to the more general case where $\bm{M}_{\star}=\bm{L}_{\star}\bm{R}_{\star}^{\top}$, the task boils down to solving
\begin{equation}
	\underset{\bm{L}\in\mathbb{R}^{n_1\times r},\bm{R}\in\mathbb{R}^{n_2\times r}} {\text{minimize}} \myquadinv f(\bm{L},\bm{R})
	 = \frac{1}{4p}  \left\| \mathcal{P}_{\Omega}( \bm{L}\bm{R}^{\top} - \bM_{\star} ) \right\|_{\mathrm{F}}^{2}. \label{eq:MC-empirical-risk-asym}
\end{equation}

One parameter that plays a crucial role in determining the feasibility of matrix completion is a certain coherence measure, defined as follows \cite{candes2009exact}.
\begin{definition}[$\mathsf{Incoherence~for~matrix~completion}$]
	\label{def:mc-incoherence}
	A rank-$r$ matrix $\bm{M}_{\star}\in \mathbb{R}^{n_1\times n_2}$  with singular value decomposition (SVD) $\bm{M}_{\star}=\bm{U}_{\star}\bm{\Sigma}_{\star}\bm{V}_{\star}^\top$
is said to be $\mu$-incoherent if 
\begin{subequations}
\begin{align}
	\left\Vert \bm{U}_{\star}\right\Vert _{2,\infty}\leq\sqrt{{\mu}/{n_1}}\left\Vert \bm{U}_{\star}\right\Vert _{\mathrm{F}}=\sqrt{{\mu r}/{n_1}};
	\label{eq:incoherence-U-MC} \\
	\left\Vert \bm{V}_{\star}\right\Vert _{2,\infty}\leq\sqrt{{\mu}/{n_2}}\left\Vert \bm{V}_{\star}\right\Vert _{\mathrm{F}}=\sqrt{{\mu r}/{n_2}}. 
	\label{eq:incoherence-V-MC}
\end{align}
\end{subequations}
\end{definition}
\noindent As shown in \cite{CanTao10}, a low-rank matrix cannot be recovered from a highly incomplete set of entries, unless the matrix satisfies the incoherence condition with a small $\mu$.

Throughout this paper, we let $n:= \max\{n_1,n_2\}$ when referring to the matrix completion problem, and set $\kappa =\sigma_1(\bm{M}_{\star})/\sigma_r(\bm{M}_{\star})$ to be the condition number of $\bm{M}_{\star}$. 
%

\subsection{Blind deconvolution (the subspace model)} 
\label{sec:BD}

Suppose that we want to recover two objects $\bm{h}_{\star}\in \mathbb{C}^K$ and $\bm{x}_{\star}\in \mathbb{C}^N$ --- or equivalently, the outer product $\bm{M}_{\star} = \bm{h}_{\star} \bm{x}_{\star}^{\conj}$ --- from $m$ bilinear measurements of the form
\begin{equation}
	y_i = \bm{b}_i^{\conj}\bm{h}_{\star} \bm{x}_{\star}^{\conj} \bm{a}_i,  \quad i = 1,\cdots,m. 
	\label{eq:samples-BD}
\end{equation}
To explain why this is called blind deconvolution, imagine  we would like to recover two signals $\bm{g}\in\mathbb{C}^m$ and $\bm{d}\in\mathbb{C}^m$ from their circulant convolution \cite{ahmed2014blind}. In the frequency domain, the outputs can be written as $\bm{y} = \mbox{diag}(\hat{\bm{g}})\, \hat{\bm{d}}$, 
where $\hat{\bm{g}}$ (resp.~$\hat{\bm{d}}$) is the Fourier transform of $\bm{g}$ (resp.~$\bm{d}$). 
If we have  additional knowledge that $\hat{\bm{g}}=\overline{\bm{A}\bm{x}_{\star}}$ and $\hat{\bm{d}}=\bm{B}\bm{h}_{\star}$ lie in some known  subspace characterized by $\bm{A}=[\bm{a}_{1},\cdots,\bm{a}_{m}]^{\conj}$ and $\bm{B}=[\bm{b}_{1},\cdots,\bm{b}_{m}]^{\conj}$, then $\bm{y}$ reduces to the bilinear form \eqref{eq:samples-BD}.  In this paper, we assume the following {\em semi-random} design, a common {\em subspace model} studied in the literature \cite{ahmed2014blind,li2016deconvolution}.

\begin{assumption}[$\mathsf{Semi}\text{-}\mathsf{random~design~in~blind~deconvolution}$]
Suppose that $\bm{a}_{j}\overset{\mathrm{i.i.d.}}{\sim}\mathcal{N}\left(\bm{0},\frac{1}{2}\bm{I}_{N}\right)+i\mathcal{N}\left(\bm{0},\frac{1}{2}\bm{I}_{N}\right)$, 
and that $\bm{B} \in\mathbb{C}^{m\times K}$ is formed by the first $K$ columns of a unitary discrete Fourier
transform (DFT) matrix $\bm{F}\in\mathbb{C}^{m\times m}$.   
\end{assumption}


To solve this problem, one seeks a solution to
\begin{equation}
	\label{eq:fmin-BD}
	\underset{\bm{h}\in \mathbb{C}^K,\bm{x}\in\mathbb{C}^{N}}{ \text{minimize}}~~\myquadinv f(\bm{h},\bm{x})=  \sum_{j=1}^{m}\big| \bm{b}_{j}^{\conj}\bm{h}\bm{x}^{\conj}\bm{a}_{j}- y_i \big|^{2}.
\end{equation}
The recovery performance typically depends on an incoherence
measure crucial for blind deconvolution.
\begin{definition}[$\mathsf{Incoherence~for~blind~deconvolution}$]
	\label{def:BD-mu}
	Let the incoherence parameter $\mu$ of $\bh_{\star}$ be the smallest number such that 
\begin{equation}
\max_{1\leq j\leq m}\left|\bm{b}_{j}^{\conj}\bm{h}_{\star}\right|\leq\frac{\mu}{\sqrt{m}}\left\Vert \bm{h}_{\star}\right\Vert _{2}.\label{eq:incoherence-BD}
\end{equation}
\end{definition}


\subsection{Low-rank and sparse matrix decomposition\,/\,robust principal component analysis}
\label{sec:RPCA}

Suppose we are given a matrix $\bm{\Gamma}_{\star}\in\mathbb{R}^{n_1\times n_2}$ that is a superposition of a rank-$r$ matrix $\bm{M}_{\star}\in\mathbb{R}^{n_1\times n_2}$ and a sparse matrix $\bm{S}_{\star}\in\mathbb{R}^{n_1\times n_2}$:
\begin{equation}\label{eq:rpca_decomp}
\bm{\Gamma}_{\star} = \bm{M}_{\star} + \bm{S}_{\star}.
\end{equation}
The goal is to separate $\bm{M}_{\star}$ and $\bm{S}_{\star}$ from the (possibly partial) entries of $\bm{\Gamma}_{\star}$. This problem is also known as {\em robust principal component analysis} \cite{chandrasekaran2011siam,candes2009robustPCA,vaswani2017robust}, since we can think of it as recovering the low-rank factors of  $\bm{M}_{\star}$ when the observed entries are corrupted by {\em sparse outliers} (modeled by $\bm{S}_{\star}$). The problem spans numerous applications in computer vision, medical imaging, and surveillance. 

Similar to the matrix completion problem, we assume the random sampling model \eqref{eq:random-sample-MC}, where $\Omega$ is the set of observed entries. 
In order to make the problem well-posed, we need the incoherence parameter of $\bM_{\star}$ as defined in Definition~\ref{def:mc-incoherence} as well, which precludes $\bM_{\star}$ from being too spiky.  In addition, it is sometimes convenient to introduce the following deterministic condition on the sparsity pattern and sparsity level of $\bm{S}_{\star}$,  originally proposed by \cite{chandrasekaran2011siam}. Specifically, it is assumed that the non-zero entries of $\bm{S}_{\star}$ are ``spread out'', where there are at most a  fraction $\alpha$ of non-zeros per row\,/\,column. Mathematically, this means:

\begin{assumption}
	It is assumed that $\bS_{\star}\in\mathcal{S}_\alpha\subseteq \mathbb{R}^{n_1\times n_2}$, where
\begin{equation}
	\mathcal{S}_\alpha := \left\{\bS \,:\,  \|(\bS_{\star})_{i,\cdot}\|_0 \leq \alpha n_2, \; \|(\bS_{\star})_{\cdot,j}\|_0 \leq \alpha n_1, \forall i,j \right\}. 
\end{equation}
\end{assumption}

For the positive semidefinite case where $\bm{M}_{\star}=\bm{X}_{\star}\bm{X}_{\star}^{\top}$ with $\bm{X}_{\star}\in \mathbb{R}^{n\times r}$, the task can be cast as solving
\begin{equation}
	\underset{\bm{X}\in\mathbb{R}^{n\times r},\,\bm{S}\in\mathcal{S}_{\alpha}}{\text{minimize}}\myquadinv f(\bm{X},\bm{S})
	= \frac{1}{4p}
	\left\| \mathcal{P}_{\Omega}(\bm{\Gamma}_{\star} - \bm{X}\bm{X}^{\top}- \bm{S} )\right\|_{\mathrm{F}}^2.
	\label{eq:RPCA-empirical-risk}
\end{equation}
The general case where $\bm{M}_{\star}=\bm{L}_{\star}\bm{R}_{\star}^{\top}$ can be formulated similarly by replacing $\bm{X}\bm{X}^{\top}$ with $\bm{L}\bm{R}^{\top}$ in \eqref{eq:RPCA-empirical-risk} and optimizing over $\bm{L}\in \mathbb{R}^{n_1\times r}$, $\bm{R}\in \mathbb{R}^{n_2\times r}$ and $\bm{S}\in\mathcal{S}_{\alpha}$, namely, 
\begin{equation}
	\underset{\bm{L}\in\mathbb{R}^{n_1\times r},\, \bm{R}\in\mathbb{R}^{n_2\times r},\, \bm{S}\in\mathcal{S}_{\alpha}}{\text{minimize}}\myquadinv f(\bm{L}, \bm{R},\bm{S})
	= \frac{1}{4p}
	\left\| \mathcal{P}_{\Omega}(\bm{\Gamma}_{\star} - \bm{L}\bm{R}^{\top}- \bm{S} )\right\|_{\mathrm{F}}^2.
	\label{eq:RPCA-empirical-risk-LR}
\end{equation}
%



\section{Local refinement via gradient descent}
\label{sec:gd}

This section contains extensive discussions of local convergence analysis of gradient descent. GD is perhaps the most basic optimization algorithm,  and its practical importance cannot be overstated.  Developing fundamental understanding  of this algorithm sheds light on the effectiveness of many other iterative algorithms for solving nonconvex problems.  

In the sequel, we will first examine what standard GD theory (cf.~Lemma \ref{lem:GD-convergence}) yields for matrix factorization problems; see Section \ref{sec:GD-standard-MF}. While the resulting computational guarantees are optimal for nearly isotropic sampling operators, they become highly pessimistic for most of other problems. We diagnose the cause in Section \ref{sec:restricted_geometry} and isolate an incoherence condition that is crucial to enable fast convergence of GD.   Section \ref{sec:Reg-GD} discusses how to enforce proper regularization to promote such an incoherence condition, while Section \ref{sec:Implicit-Reg} illustrates an implicit regularization phenomenon that allows unregularized GD to converge fast as well. We emphasize that generic optimization theory alone yields overly pessimistic convergence bounds;   one needs to blend computational and statistical analyses in order to understand the intriguing  performance of GD.



\subsection{Computational analysis via strong convexity and smoothness} 
\label{sec:GD-standard-MF}

To analyze  local convergence of GD, a natural strategy  is to resort to the standard GD theory in Lemma \ref{lem:GD-convergence}. This requires 
checking whether strong convexity and smoothness hold locally, as done in Section~\ref{sec:local_convergence_mf}. If so, then  Lemma \ref{lem:GD-convergence} yields an upper bound on the iteration complexity. This simple strategy works well when, for example, the sampling operator is nearly isotropic. In the sequel, we use a few examples to illustrate the applicability and potential drawback of this analysis strategy. 

\subsubsection{Measurements that satisfy the RIP (the rank-1 case)} 
We begin with the  {\em matrix sensing} problem \eqref{eq:matrix-sensing-samples} and consider the case where the  truth has rank 1, i.e.~$\bm{M}_{\star}=\bm{x}_{\star}\bm{x}_{\star}^{\top}$ for some vector $\bm{x}_{\star}\in \mathbb{R}^n$. This requires us to solve
\begin{align}
	\underset{\bm{x}\in \mathbb{R}^n}{\text{minimize}}~\myquadinv f(\bm{x})=\frac{1}{4m}\sum_{i=1}^{m}\big(\langle\bm{A}_{i},\bm{x}\bm{x}^{\top}\rangle- y_i \big)^{2}.    
	\label{eq:min-matrix-sensing}	
\end{align}
For notational simplicity, this subsection  focuses on the {\em symmetric} case where $\bm{A}_i=\bm{A}_i^{\top}$. 
The gradient  update rule \eqref{eq:GD-general} for this problem reads 
\begin{align} 
	\label{eq:GD-rank1-sensing}
	\bm{x}_{t+1} & =\bm{x}_{t}- \frac{\eta_{t}}{m} \sum_{i=1}^{m} \big( \langle\bm{A}_{i},\bm{x}_t\bm{x}_t^{\top}\rangle - y_i \big) \bm{A}_{i}\bm{x}_t  \nonumber \\
	& =\bm{x}_{t}-  \eta_{t} \mathcal{A}^*\mathcal{A}(\bm{x}_t\bm{x}_t^{\top}  - \bm{x}_{\star}\bm{x}_{\star}^{\top} ) \cdot \bm{x}_t ,
\end{align}
where $\mathcal{A}$ is defined in Section \ref{sec:matrix_sensing}, and $\mathcal{A}^*$ is the conjugate operator of $\mathcal{A}$. 

When the sensing matrices $[\bA_i]_{1\leq i\leq m}$ are random and isotropic,
\eqref{eq:min-matrix-sensing} can be viewed as a randomized version
of  the rank-1 matrix factorization problem discussed
in Section \ref{sec:noncvx_eg}. 
To see this, consider for instance the case where the $\bA_i$'s are drawn from the symmetric  Gaussian design, i.e.~the diagonal entries of $\bm{A}_i$ are i.i.d.~$\mathcal{N}(0,1)$ and the off-diagonal entries are i.i.d.~$\mathcal{N}(0,1/2)$. For any
fixed $\bm{x}$, one has
\[
	\mathbb{E}\left[\nabla f(\bm{x})\right]=(\bm{x}\bm{x}^{\top}-\bm{x}_{\star}\bm{x}_{\star}^{\top})\bm{x},
\]
which coincides with the warm-up example \eqref{eq:grad-rank1} by taking $\bm{M}=\bm{x}_{\star}\bm{x}_{\star}^{\top}$. This bodes well for fast local convergence of GD, at least at the population level (i.e.~the case when the sample size $m\rightarrow \infty$).

What happens in the finite-sample regime? It turns out that if the sensing operator satisfies the RIP (cf.~Definition~\ref{defn:RIPs}), then $\nabla^2 f(\cdot)$ does not deviate too much from its population-level counterpart, and hence $f(\cdot)$ remains locally strongly convex and smooth. This in turn allows one to invoke the standard GD theory to establish local linear convergence. 
\begin{theorem}[$\mathsf{GD~for~matrix~sensing~(rank\text{-}1)}$]
	\label{thm:convergence-rank1-sensing}
	Consider the  problem \eqref{eq:min-matrix-sensing}, and suppose the operator \eqref{eq:defn-A-sensing} satisfies $4$-RIP for RIP constant $\delta_4\leq {1}/{44}$. If $\|\bm{x}_0 - \bm{x}_{\star}\|_2\leq {\|\bm{x}_{\star}\|_2}/{12}$, then GD with $\eta_t\equiv 1/(3\|\bm{x}_{\star}\|_2^2)$ obeys
	\begin{equation}
		\|\bm{x}_{t}-\bm{x}_{\star}\|_{2}\leq\left(11/12\right)^{t}\|\bm{x}_{0}-\bm{x}_{\star}\|_{2}, \quad t = 0,1,\cdots
	\end{equation}
\end{theorem}
This theorem, which is a deterministic result, 
is established in Appendix~\ref{appendix:proof-thm:convergence-rank1-sensing}. 
An appealing feature is that: it is possible for such RIP to hold as long as the sample size $m$ is on the order of the information-theoretic limits (i.e.~$O(n)$), in view of Fact~\ref{fact:gaussian_rip}. The take-home message is: for highly random and nearly isotropic sampling schemes, local strong convexity and smoothness continue to hold even in the sample-limited regime.

\subsubsection{Measurements that satisfy the RIP (the rank-$r$ case)} 
\label{sec:GD-RIP-rankr}

The rank-1 case is singled out above due to its simplicity. The result certainly goes well beyond the rank-1 case. Again, we focus on the symmetric case\footnote{If the $\bA_i$'s are asymmetric, the gradient is given by $\nabla f(\bm{X})=\frac{1}{2m} \sum_{i=1}^{m} ( \langle\bm{A}_{i},\bm{X}\bm{X}^{\top}\rangle - y_i ) (\bm{A}_{i}+\bm{A}_i^{\top})\bm{X}$, although the  update rule \eqref{eq:GD-rankr-sensing} remains applicable.} \eqref{eq:min-matrix-sensing-rank-r} with $\bm{A}_i=\bm{A}_i^{\top}$, in which the gradient update rule \eqref{eq:GD-general} satisfies
\begin{align} 
	\label{eq:GD-rankr-sensing}
	\bm{X}_{t+1}=\bm{X}_{t}- \eta_{t} \underset{=\nabla f(\bm{X}_t)}{\underbrace{ \frac{1}{m} \sum\nolimits_{i=1}^{m} \big( \langle\bm{A}_{i},\bm{X}_t\bm{X}_t^{\top}\rangle - y_i \big) \bm{A}_{i}\bm{X}_t }}.
\end{align}
%

At first, one might imagine that $f(\cdot)$ remains locally strongly convex. This is, unfortunately, not true, as demonstrated by the following example.  
\begin{example}
\label{example:strong-cvx}
Suppose $\mathcal{A}^*\mathcal{A}(\cdot)$ is identity, then $f(\cdot)$ reduces to 
\begin{equation}
	f_{\infty}(\bX) =  \frac{1}{4}\|\bm{X}\bm{X}^{\top}-\bm{X}_{\star}\bm{X}_{\star}^{\top}\|_{\mathrm{F}}^{2}. \label{eq:pop-level-factorization}
\end{equation}
It can be shown that for any $\bm{Z}\in \mathbb{R}^{n\times r}$ (see \cite{ge2016matrix,ma2017implicit}):
\begin{align}
 	&\mathsf{vec}(\bm{Z})^{\top}\nabla^{2}f_{\infty}(\bm{X})\,\mathsf{vec}(\bm{Z}) \mynonumber \mylinebreak
	\myalign = 0.5\|\bm{X}\bm{Z}^{\top}+\bm{Z}\bm{X}^{\top}\|_{\mathrm{F}}^{2} + \langle\bm{X}\bm{X}^{\top}-\bm{X}_{\star}\bm{X}_{\star}^{\top},\bm{Z}\bm{Z}^{\top}\rangle.  \label{eq:Hess-infty}
\end{align}
Think of the following example
\begin{align}
	\label{eq:example-X-Z}
	\bm{X}_{\star}=[\bm{u},\bm{v}], ~~\bm{X}=(1-\delta)\bm{X}_{\star},~ \text{ and } ~\bm{Z}=[\bm{v},-\bm{u}]
\end{align}
for two unit vectors $\bm{u}, \bm{v}$ obeying $\bm{u}^{\top}\bm{v}=0$ and any $0<\delta<1$. 
It is straightforward  to verify that
\[
\mathsf{vec}(\bm{Z})^{\top}\nabla^{2}f_{\infty}(\bm{X})\,\mathsf{vec}(\bm{Z})=-2(2\delta-\delta^{2})<0 ,
\]
which violates convexity. Moreover, this happens even when $\bm{X}$ is arbitrarily close
	to $\bm{X}_{\star}$ (by taking $\delta\rightarrow0$). 
\end{example}

Fortunately, the above issue can be easily addressed. The key is to recognize that: one can only hope to recover $\bm{X}_{\star}$ up to global orthonormal transformation, unless further constraints are imposed. Hence, a more suitable error metric is
\begin{equation}
	\label{defn:dist-matrix}
	\mathsf{dist}(\bm{X},\bm{X}_{\star}):=\min_{\bm{H}\in \mathcal{O}^{r\times r}} \| \bm{X}\bm{H} - \bm{X}_{\star} \|_{\mathrm{F}},
\end{equation}
a counterpart of the Euclidean error when accounting for global ambiguity. For notational convenience, we let
\begin{align}\label{defn:HX}
	\bm{H}_{\bm{X}}:= \argmin_{\bm{H}\in\mathcal{O}^{r\times r}}\|\bm{X}\bm{H}-\bm{X}_{\star}\|_{\mathrm{F}}.
\end{align}
Finding $\bm{H}_{\bm{X}}$ is a classical problem called the {\em orthogonal Procrustes problem} \cite{ten1977orthogonal}. 

With these metrics in mind, we are ready to generalize the standard GD theory in Lemma \ref{lem:GD-convergence} and Lemma \ref{lem:convergence-RC}. In what follows, we assume that $\bm{X}_{\star}$ is a global minimizer of $f(\cdot)$, and make the further homogeneity assumption $\nabla f(\bm{X})\bm{H}=\nabla f(\bm{X}\bm{H})$ for any orthonormal matrix $\bm{H}\in \mathcal{O}^{r\times r}$ --- a common fact that arises in matrix factorization problems.   

\begin{lemma}
	\label{lem:GD-convergence-dist} 
	Suppose that $f$ is $\beta$-smooth within a  ball $\mathcal{B}_{\zeta}(\bm{X}_{\star}):=\left\{ \bm{X}: \|\bm{X}-\bm{X}_{\star} \|_{\mathrm{F}}\leq\zeta\right\} $, and that $\nabla f(\bm{X})\bm{H}=\nabla f(\bm{X}\bm{H})$ for any orthonormal matrix $\bm{H}\in \mathcal{O}^{r\times r}$. 
	Assume that for any $\bm{X}\in\mathcal{B}_{\zeta}(\bm{X}_{\star})$ and any $\bm{Z}$,
	\begin{align}
		& \mathsf{vec}(\bm{Z}\bm{H}_{\bm{Z}}-\bm{X}_{\star})^{\top}\nabla^{2}f(\bm{X})\,\mathsf{vec}(\bm{Z}\bm{H}_{\bm{Z}}-\bm{X}_{\star}) \mynonumber\mylinebreak
		\myalign \myquad \myquad \geq\alpha\| \bm{Z} \bm{H}_{\bm{Z}} - \bm{X}_{\star}\|_{\mathrm{F}}^{2}. 	\label{eq:strong-cvx-rotate-1point}
	\end{align}
	If $\eta_{t}\equiv 1/\beta$, then GD with  $\bm{X}_{0}\in\mathcal{B}_{\zeta}(\bm{X}_{\star})$ obeys
	\begin{equation*}
		\mathsf{dist}^2\big( \bm{X}_{t}, \bm{X}_{\star} \big) \leq \left(1-\frac{\alpha}{\beta}\right)^{t} \mathsf{dist}^2\big( \bm{X}_{0}, \bm{X}_{\star} \big),\quad t\geq 0.
	\end{equation*}
\end{lemma}
\begin{proof}[Proof of Lemma \ref{lem:GD-convergence-dist}]
For notational simplicity, let $$\bm{H}_{t}:=\argmin_{\bm{H}\in\mathcal{O}^{r\times r}}\|\bm{X}_{t}\bm{H}-\bm{X}_{\star}\|_{\mathrm{F}}.$$
	First, by definition of $\mathsf{dist}(\cdot,\cdot)$ we have 
$$ \mathsf{dist}^{2}(\bm{X}_{t+1},\bm{X}_{\star}) \leq\|\bm{X}_{t+1}\bm{H}_{t}-\bm{X}_{\star}\|_{\mathrm{F}}^{2}. $$
Next, we claim that a modified regularity condition  $\mathsf{RC}( 1/\beta ,\alpha, \zeta )$  (cf.~Definition \ref{def:reg-condition}) holds for all $t\geq 0$, that is, 
\begin{align}\label{eq:RC-rotate-1point}
	2 \langle\, \nabla f(\bm{X}_{t}\bm{H}_{t}), \bm{X}_t\bm{H}_{t}-\bm{X}_{\star} \rangle\geq \frac{1}{\beta}\|  \nabla f(\bm{X}_t\bm{H}_t) \|^2_{\mathrm{F}} + \alpha \| \bm{X}_t\bm{H}_t- \bm{X}_{\star}\|_{\mathrm{F}}^{2}, \qquad  t\geq 0.
\end{align}
If this claim is valid, then it is straightforward to adapt Lemma \ref{lem:convergence-RC} to obtain 
\begin{align}
	\mathsf{dist}^{2}(\bm{X}_{t+1},\bm{X}_{\star}) \leq\|\bm{X}_{t+1}\bm{H}_{t}-\bm{X}_{\star}\|_{\mathrm{F}}^{2} \leq \left(1- \frac{\alpha}{\beta} \right) \|\bm{X}_{t}\bm{H}_{t}-\bm{X}_{\star}\|_{\mathrm{F}}^{2} = \left(1- \frac{\alpha}{\beta} \right) \mathsf{dist}^{2}(\bm{X}_{t},\bm{X}_{\star}),
	\label{eq:dist-shrinkage-sensing}
\end{align}
thus establishing the advertised linear convergence. We omit this part for brevity.  



The rest of the proof is thus dedicated to justifying \eqref{eq:RC-rotate-1point}. First,  Taylor's theorem reveals that
\begin{align}
	f(\bm{X}_{\star}) & = f(\bm{X}_{t}\bm{H}_t) - \left\langle \nabla f(\bm{X}_{t}\bm{H}_t) , \bm{X}_{t}\bm{H}_t - \bm{X}_{\star} \right\rangle +\frac{1}{2}\mathsf{vec}(\bm{X}_{t}\bm{H}_{t}-\bm{X}_{\star})^{\top} \nabla^{2}f(\bm{X}(\tau) ) \mathsf{vec}(\bm{X}_{t}\bm{H}_{t}-\bm{X}_{\star}), \label{eq:f-taylor-sensing}
\end{align}
where $\bm{X}(\tau):=\bm{X}_{t}\bm{H}_{t}+\tau(\bm{X}_{\star} -\bm{X}_{t}\bm{H}_{t})$ for some $0\leq \tau \leq 1$. 
If $\bm{X}_{t}\bm{H}_{t}$ lies within $\mathcal{B}_{\zeta}(\bm{X}_{\star})$, then one has $\bm{X}(\tau) \in \mathcal{B}_{\zeta}(\bm{X}_{\star})$ as well. 
We can then substitute the condition (\ref{eq:strong-cvx-rotate-1point}) into \eqref{eq:f-taylor-sensing} to reach
\begin{align} \label{eq:modified_one_point_convexity}
f(\bm{X}_{\star}) & \geq f(\bm{X}_{t}\bm{H}_t) - \left\langle \nabla f(\bm{X}_{t}\bm{H}_t) , \bm{X}_{t}\bm{H}_t - \bm{X}_{\star} \right\rangle + \frac{\alpha}{2} \left\|\bm{X}_{t}\bm{H}_{t}-\bm{X}_{\star} \right\|_{\mathrm{F}}^2,
\end{align}
which can be regarded as a modified version of the one point convexity \eqref{eq:one_point_convexity}. In addition, repeating the argument in \eqref{eq:smoothness_consequence} yields
\begin{align}
f(\bm{X}_{\star})- f(\bm{X}_{t}\bm{H}_t) &\leq -\frac{1}{2\beta} \left\| \nabla f(\bm{X}_{t}\bm{H}_t) \right\|_{\mathrm{F}}^2  \label{eq:modified_smoothness_consequence}
\end{align}
which is a consequence of the smoothness assumption. Combining \eqref{eq:modified_one_point_convexity} and \eqref{eq:modified_smoothness_consequence} establishes  \eqref{eq:RC-rotate-1point} for the $t$th iteration, provided that $\bm{X}_t\bm{H}_t\in \mathcal{B}_{\zeta}(\bm{X}_{\star})$.
Finally, since the initial point is assumed to fall within $\mathcal{B}_{\zeta}(\bm{X}_{\star})$,  we immediately learn from 
	\eqref{eq:dist-shrinkage-sensing} and induction that $\bm{X}_t\bm{H}_t\in \mathcal{B}_{\zeta}(\bm{X}_{\star})$ for all $t\geq 0$. This concludes the proof. 
\end{proof}

The condition \eqref{eq:strong-cvx-rotate-1point} 
is a modification of strong convexity to account for global rotation. In particular, 
it restricts attention to directions of the form $\bm{Z}\bm{H}_{\bm{Z}}-\bm{X}_{\star}$, where one first adjusts the orientation of $\bm{Z}$ to best align with the global minimizer. To confirm that such restriction is sensible, we revisit Example \ref{example:strong-cvx}. With proper rotation, one has
$\bm{Z}\bm{H}_{\bm{Z}}=\bm{X}_{\star}$ and hence
\[
	\mathsf{vec}(\bm{Z}\bm{H}_{\bm{Z}})^{\top}\nabla^{2}f(\bm{X})\,\mathsf{vec}(\bm{Z}\bm{H}_{\bm{Z}})= 2(2-6\delta+3\delta^2),
\]
which becomes strictly positive for $\delta\leq 1/3$. In fact, 
 if $\bm{X}$ is sufficiently close to $\bm{X}_{\star}$, then the condition \eqref{eq:strong-cvx-rotate-1point} is valid for \eqref{eq:pop-level-factorization}.  Details are deferred to Appendix \ref{appendix-strong-cvx-rotation-pop}.

 Further, similar to the analysis for the rank-1 case, we can  demonstrate that if  $\mathcal{A}$ satisfies the RIP for some sufficiently small RIP constant, then $\nabla^2 f(\cdot)$ is locally not far from $\nabla^2 f_{\infty}(\cdot)$ in Example \ref{example:strong-cvx}, meaning that the condition \eqref{eq:strong-cvx-rotate-1point} continues to hold for some $\alpha > 0$. 
This leads to the following result. It is assumed that 
the ground truth $\bm{M}_{\star}$ has condition number $\kappa$.   
\begin{theorem}[$\mathsf{GD~for~matrix~sensing~(rank\text{-}}r\mathsf{)}$ \cite{tu2015low,zheng2015convergent}]
	\label{thm:convergence-rankr-sensing}
	Consider the  problem \eqref{eq:min-matrix-sensing-rank-r}, and suppose the operator \eqref{eq:defn-A-sensing} satisfies the $6r$-RIP with  RIP constant $\delta_{6r}\leq {1}/{10}$. Then there exist some universal constants $c_0,c_1>0$ such that if $\mathsf{dist}^2( \bm{X}_{0}, \bm{X}_{\star} )\leq \sigma_{r}(\bm{M}_{\star})/16$, then GD with $\eta_t\equiv c_0/\sigma_{1}(\bm{M}_{\star})$ obeys
	\[
		\mathsf{dist}^2( \bm{X}_{t}, \bm{X}_{\star} ) \leq \left(1-\frac{c_1}{\kappa}\right)^{t} \mathsf{dist}^2( \bm{X}_{0}, \bm{X}_{\star} ), \quad t = 0,1,\cdots
	\]
\end{theorem}
\begin{remark} The algorithm \eqref{eq:GD-rankr-sensing} is referred to as {\em Procrustes flow} in \cite{tu2015low}. \end{remark}
Three implications of Theorem \ref{thm:convergence-rankr-sensing} merit particular attention: (1) the quality of the initialization depends on the least singular value of the truth, so as to ensure that the estimation error does not overwhelm any of the important signal direction; (2) the convergence rate becomes a function of the condition number $\kappa$: the better conditioned the truth is, the faster GD converges; (3) when a good initialization is present, GD converges linearly as long as the sample size $m$ is on the order of the information-theoretic limits (i.e.~$O(nr)$), in view of Fact~\ref{fact:gaussian_rip}.

\begin{remark}[$\mathsf{Asymmetric~case}$]\label{rmk:GD_asymetric}
	Our discussion continues to hold for the more general case where $\bM_{\star}=\bL_{\star}\bR_{\star}^{\top}$, although an extra regularization term has been suggested to balance the size of the two factors. Specifically, we introduce a regularized version of the loss \eqref{eq:min-matrix-sensing-rank-r-asymmetric} as follows
\begin{equation}\label{eq:regularized_loss_asymmetric}
f_{\mathsf{reg}}(\bm{L},\bm{R}) = f(\bm{L},\bm{R})  + \lambda \left\|\bL^{\top}\bL - \bR^{\top}\bR \right\|_{\mathrm{F}}^2
\end{equation}
with $\lambda$ a regularization parameter, e.g. $\lambda=1/32$ as suggested in~\cite{tu2015low,bhojanapalli2016dropping}.\footnote{In practice, $\lambda=0$ also works, indicating a regularization term might not be necessary.}
	Here, the regularization term 	$\left\|\bL^{\top}\bL - \bR^{\top}\bR \right\|_{\mathrm{F}}^2$ is included in order to balance the size of the two low-rank factors $\bm{L}$ and $\bm{R}$. 
	If one applies GD to the regularized loss:
\begin{subequations}\label{eq:GD-rankr-asym-sensing}
\begin{align} 
	\bm{L}_{t+1}  =\bm{L}_{t} &- \eta_{t} \nabla_{\bL} f_{\mathsf{reg}}(\bm{L}_t,\bm{R}_t), \\
	\bm{R}_{t+1}  =\bm{R}_{t} &- \eta_{t} \nabla_{\bR} f_{\mathsf{reg}}(\bm{L}_t,\bm{R}_t),
\end{align}
\end{subequations}
	then the convergence rate for the symmetric case remains valid by replacing $\bX_{\star}$ (resp.~$\bX_t$) with $\bX_{\star}={\scriptsize\begin{bmatrix}
\bm{L}_{\star} \\
\bm{R}_{\star}
	\end{bmatrix}}$ (resp.~$\bX_{t}={\scriptsize\begin{bmatrix}
\bm{L}_{t} \\
\bm{R}_{t}
	\end{bmatrix}}$) in the error metric.
\end{remark}


\subsubsection{Measurements that do not obey the RIP}
There is no shortage of important examples where the sampling operators fail to satisfy the standard RIP at all. For these cases, the standard theory in Lemma~\ref{lem:GD-convergence} (or Lemma \ref{lem:convergence-RC}) either is not directly applicable or leads to pessimistic computational guarantees. This subsection presents a few such cases.  

We start with {\em phase retrieval} \eqref{eq:min-PR}, for which the gradient update rule is given by
\begin{equation}
	\bm{x}_{t+1}=\bm{x}_{t}- \frac{\eta_{t}}{m}\sum_{i=1}^{m} \big((\bm{a}_{i}^{\top}\bm{x}_t)^{2}-y_{i}\big) \bm{a}_{i}\bm{a}_{i}^{\top}\bm{x}_{t}.
	\label{eq:PR-GD}
\end{equation}
This algorithm, also dubbed as {\em Wirtinger flow}, was first investigated in \cite{candes2015phase}. The name ``Wirtinger flow'' stems from the fact that Wirtinger calculus is used to calculate the gradient in the complex-valued case.    

The associated sampling operator $\mathcal{A}$ (cf.~\eqref{eq:defn-A-sensing}), unfortunately, does not satisfy the standard RIP (cf.~Definition~\ref{defn:RIPs}) unless the sample size far exceeds the statistical limit; see \cite{candes2012phaselift,chen2015exact}.\footnote{
	More specifically, $\|\mathcal{A}(\bm{x}\bm{x}^{\top})\|_2$ cannot be uniformly  controlled from above, a fact that is closely related to the large smoothness parameter of $f(\cdot)$. 
	We note, however, that $\mathcal{A}$ satisfies other variants of RIP (like $\ell_1/\ell_1$ RIP w.r.t.~rank-1 matrices \cite{candes2012phaselift} and $\ell_1/\ell_2$ RIP w.r.t.~rank-$r$ matrices \cite{chen2015exact}) with high probability after proper de-biasing. } 
We can, however, still evaluate the local strong convexity and smoothness parameters to see what computational bounds they produce. Recall that the Hessian of \eqref{eq:min-PR} is given by
\begin{equation}
	\nabla^{2}f(\bm{x})=\frac{1}{m}\sum_{i=1}^{m}\left[3(\bm{a}_{i}^{\top}\bm{x})^{2}-(\bm{a}_{i}^{\top}\bm{x}^{\star})^{2}\right]\bm{a}_{i}\bm{a}_{i}^{\top}.
	\label{eq:Hessian-PR}
\end{equation}
Using standard concentration inequalities for random matrices  \cite{tropp2015introduction,vershynin2010nonasym}, one  derives the following strong convexity and smoothness  bounds  \cite{soltanolkotabi2014algorithms,sanghavi2017local,ma2017implicit}.\footnote{While \cite[Chapter 15.4.3]{soltanolkotabi2014algorithms}  presents the  bounds only for the complex-valued case, all arguments immediately extend to the real case.} 
\begin{lemma}[$\mathsf{Local~strong~convexity~and~smoothness }$ $\mathsf{for}$ $\mathsf{phase}$ $\mathsf{retrieval}$]\label{lem:strong-cvx-PR}
	Consider the problem \eqref{eq:min-PR}.
	There exist some constants $c_0,c_1,c_2>0$ such that with probability at least $1- O(n^{-10})$, 
	\begin{equation}
		0.5 \bm{I}_n \preceq \nabla^2 f(\bm{x}) \preceq c_2 n \bm{I}_n 
	\end{equation}
	 holds simultaneously for all $\bm{x}$ obeying $\|\bm{x}-\bm{x}_{\star}\|_2 \leq c_1\|\bm{x}_{\star}\|_2$,   provided that $m \geq c_0n \log n$. 
\end{lemma}
This lemma says that $f(\cdot)$ is locally 0.5-strongly convex and $c_2n$-smooth when the sample size $m\gtrsim n\log n$. The sample complexity  only exceeds the information-theoretic limit by a logarithmic factor. Applying Lemma~\ref{lem:GD-convergence} then reveals that:
\begin{theorem}[$\mathsf{GD~for~phase~retrieval~(loose~bound)}$ \cite{candes2015phase}]
\label{thm:GD-PR-loose}
Under the assumptions of Lemma \ref{lem:strong-cvx-PR},   the GD iterates  obey
\begin{equation}
\|\bm{x}_{t}-\bm{x}_{\star}\|_{2}\leq\left(1-\frac{1}{2c_{2}n}\right)^{t}\|\bm{x}_{0}-\bm{x}_{\star}\|_{2},\quad t\geq 0, \label{eq:PR-loose}
\end{equation}
with probability at least $1-O(n^{-10})$, provided that $\eta_t \equiv 1/(c_2n\|\bm{x}_{\star}\|_2^2)$ and $\|\bm{x}_0-\bm{x}_{\star}\|_2 \leq c_1\|\bm{x}_{\star}\|_2$. 
\end{theorem}
This is precisely the computational guarantee given in \cite{candes2015phase}, albeit derived via a different argument.  The above iteration complexity bound, however, is not appealing in practice: it requires $O(n\log \frac{1}{\varepsilon})$ iterations to guarantee $\varepsilon$-accuracy (in a relative sense). For large-scale problems where $n$ is very large, such an iteration complexity could be prohibitive.   

Phase retrieval is certainly not the only problem where classical results in Lemma \ref{lem:GD-convergence} yield unsatisfactory answers. The situation is even worse for other important problems like matrix completion and blind deconvolution, where strong convexity (or the modified version accounting for global ambiguity) does not hold at all unless we restrict attention to a constrained class of decision variables.   All  this calls for new ideas in establishing computational guarantees that match practical performances.

\subsection{Improved computational guarantees via restricted geometry and regularization} 
\label{sec:improved-GD-reg}

As emphasized in the preceding subsection, two issues stand out in the absence of RIP: 
\begin{itemize}
	\item The smoothness condition may not be well-controlled;
	\item Local strong convexity may fail, even if we account for global ambiguity.
\end{itemize}
We discuss how to address these issues, by first identifying a restricted region with amenable geometry for fast convergence (Section~\ref{sec:restricted_geometry}) and then applying regularized gradient descent to ensure the iterates stay in the restricted region (Section~\ref{sec:Reg-GD}).

\subsubsection{Restricted strong convexity and smoothness}\label{sec:restricted_geometry}

While desired strong convexity and smoothness (or regularity conditions) may fail to hold in the entire local ball, it is possible for them to arise when we restrict ourselves to a small subset of the local ball and\,/\,or a set of special directions.  

Take phase retrieval for example: $f(\cdot)$ is locally $0.5$-strongly convex, but the smoothness parameter is exceedingly large (see Lemma \ref{lem:strong-cvx-PR}).  On closer inspection, those points $\bm{x}$ that are too aligned with any of the sampling vector $\{\bm{a}_i\}$ incur ill-conditioned Hessians. For instance, suppose $\bm{x}_{\star}$ is a unit vector independent of $\{\bm{a}_i\}$. Then the point $\bm{x}=\bm{x}_{\star}+\delta\frac{\bm{a}_{j}}{\|\bm{a}_{j}\|_{2}}$ for some constant $\delta$ often results in extremely large $\bm{x}^{\top}\nabla^{2}f(\bm{x})\bm{x}$.\footnote{  
With high probability, 
$(\bm{a}_{j}^{\top}\bm{x})^{2}=\left(1-o(1)\right)\delta^{2}n,$
and hence the Hessian (cf.~\eqref{eq:Hessian-PR}) at this point $\bm{x}$ satisfies 
%
$\bm{x}^{\top}\nabla^{2}f(\bm{x})\bm{x}  \geq\frac{3}{m}(\bm{a}_{j}^{\top}\bm{x})^{4}- O( \|\frac{1}{m}\sum\nolimits _{j} (\bm{a}_{j}^{\top}\bm{x}^{\star})^{2} \bm{a}_{j}\bm{a}_{j}^{\top} \| )
  =\left(3-o(1)\right)\delta^{4} {n^2}/{m}$,
%
much larger than the strong convexity parameter when $m\ll n^{2}$.} This simple instance suggests that: in order to ensure well-conditioned Hessians, one needs to preclude points that are too ``coherent'' with the sampling vectors, as formalized below. 
\begin{lemma}[$\mathsf{Restricted~smoothness~for~phase~retrieval}$ \cite{ma2017implicit}]
\label{lemma:wf_hessian}
Under the assumptions of Lemma \ref{lem:strong-cvx-PR}, 
there exist some constants  $c_0, \cdots, c_3>0$ such that if 
$m\geq c_{0}n\log n$, then with probability at least $1-O(mn^{-10})$, 
\[
	\nabla^{2}f\left(\bm{x}\right)\preceq c_1\log n\cdot\bm{I}_{n}
\]
holds simultaneously for all $\bm{x}\in\mathbb{R}^{n}$  obeying
\begin{subequations}
\label{eq:WF-hessian-condition} 
\begin{align} \left\Vert \bm{x} -\bm{x}_{\star}\right\Vert _{2} & \leq c_2,\label{eq:WF-induction-L2error-hessian}\\
\max_{1\leq j\leq m}\left|\bm{a}_{j}^{\top}\left(\bm{x} -\bm{x}_{\star}\right)\right| & \leq c_3\sqrt{\log n}.
\label{eq:WF-induction-incoherence-t-hessian}
\end{align}
\end{subequations}
\end{lemma}

In words,  desired smoothness is guaranteed when considering only points sufficiently  near-orthogonal to all sampling vectors. Such a near-orthogonality property will be referred to as ``{\em incoherence}'' between $\bm{x}$ and  the sampling vectors.

Going beyond phase retrieval, the  notion of incoherence  is not only crucial to control smoothness, but also plays a critical role in ensuring local strong convexity (or regularity conditions). A partial list of examples include matrix completion, quadratic sensing, blind deconvolution and demixing, etc \cite{sun2016guaranteed, chen2015fast, li2018nonconvex, ma2017implicit, li2016deconvolution, ling2017regularized, shi2018demising}. In the sequel, we single out the matrix completion problem to illustrate this fact. The interested readers are referred to \cite{sun2016guaranteed, chen2015fast,zheng2015convergent} for regularity conditions for matrix completion, to \cite{li2018nonconvex} for strong convexity and smoothness for quadratic sensing, to \cite{ma2017implicit,shi2018demising} (resp.~\cite{li2016deconvolution, ling2017regularized}) for strong convexity and smoothness (resp.~regularity condition) for blind deconvolution and demixing.  
\begin{lemma}
	[$\mathsf{Restricted~strong~convexity~and~smoothness}$ $\mathsf{for}$ $\mathsf{matrix}$ $\mathsf{completion}$ \cite{ma2017implicit}] 
	\label{lemma:hessian-MC}
	Consider the problem \eqref{eq:MC-empirical-risk}. 
	Suppose that $n^{2}p\geq c_0\kappa^{2}\mu rn\log n$ for some large
constant $c_0>0$. Then with probability  $1-O\left(n^{-10}\right)$,
the Hessian  obeys 
\begin{subequations}
\begin{align}
	\myalign \mathsf{vec}\left(\bm{Z}\bm{H}_{\bm{Z}}-\bm{X}_{\star}\right)^{\top}\nabla^{2}f\left(\bm{X}\right)\mathsf{vec}\left(\bm{Z}\bm{H}_{\bm{Z}}-\bm{X}_{\star}\right)  \mynonumber \mylinebreak
	\myalign \myquad\myquad  \myaligninv \geq 0.5 \sigma_r(\bm{M}_{\star})\left\Vert \bm{Z}\bm{H}_{Z}-\bm{X}_{\star} \right\Vert _{\mathrm{F}}^{2}  \\
	\myalign \left\Vert \nabla^{2}f \left(\bm{X}\right)\right\Vert \myaligninv \leq 2.5 \sigma_1(\bm{M}_{\star})\label{eq:strong-convexity-smoothness-MC}
\end{align}
\end{subequations}
	for all $\bm{Z}$ (with $\bm{H}_{\bm{Z}}$ defined in \eqref{defn:HX}) and all $\bm{X}$  satisfying
\begin{align}
 	\left\Vert \bm{X}-\bm{X}_{\star}\right\Vert _{2,\infty} & \leq\epsilon\left\Vert \bm{X}_{\star}\right\Vert _{2,\infty}, 
	\label{eq:MC_incoherence_neighborhood}    
\end{align}
where $\epsilon\leq c_1/\sqrt{\kappa^{3}\mu r\log^{2}n}$ for some constant $c_1>0$.\footnote{This is a simplified version of \cite[Lemma 7]{ma2017implicit} where we restrict the descent direction to point to $\bX_{\star}$.}
\end{lemma}
This lemma confines attention to the set of points obeying  
$$\|\bm{X}-\bm{X}_{\star}\|_{2,\infty}= \max_{1\leq i\leq n}\| (\bm{X}-\bm{X}_{\star})\bm{e}_i \|_2 \leq  \epsilon\left\Vert \bm{X}_{\star}\right\Vert _{2,\infty}.$$ 
Given that each observed entry $M_{i,j}$ can be viewed as $\bm{e}_i^{\top}\bm{M}\bm{e}_j$, the sampling basis relies heavily on the  standard basis vectors. As a result, the above lemma is essentially imposing conditions on the incoherence between $\bm{X}$ and the sampling basis.

%

%
	


\subsubsection{Regularized gradient descent}
\label{sec:Reg-GD}

While we have demonstrated  favorable geometry within  the set of local points satisfying the desired incoherence condition,  the challenge remains as to    how to ensure the GD iterates fall within this set. A natural strategy is to enforce proper regularization. Several auxiliary regularization procedures have been proposed to explicitly promote the incoherence constraints \cite{keshavan2010matrix, sun2016guaranteed, chen2015solving, chen2015fast, li2016deconvolution, ling2017regularized, zheng2016convergence,yi2016fast}, in the hope of improving computational guarantees. Specifically, one can {\em regularize the loss function}, by adding additional regularization term $G(\cdot)$ to the objective function and designing the GD update rule w.r.t.~the  regularized problem 
\begin{align}
	\text{minimize}_{\bm{x}} \quad f_{\mathsf{reg}}(\bx) := f(\bm{x}) + \lambda G(\bm{x})
\end{align}
with $\lambda>0$ the regularization parameter.  For example:   
\begin{itemize}
	\item {\em Matrix completion \eqref{eq:MC-empirical-risk-asym}:} the following regularized loss has been proposed \cite{keshavan2010matrix, keshavan2010noisy, sun2016guaranteed}	
\begin{align} \label{eq:regularized_mc_loss}
&  f_{\mathrm{reg}}(\bm{L},\bm{R})=f(\bm{L},\bm{R})+\lambda\left\{ G_{0}(\alpha_{1}\|\bm{L}\|_{\mathrm{F}}^{2})+G_{0}(\alpha_{2}\|\bm{R}\|_{\mathrm{F}}^{2})\right\} \nonumber\\
 & \qquad\qquad \qquad +\lambda\Bigg\{\sum_{i=1}^{n_{1}}G_{0}\big(\alpha_{3}\|\bm{L}_{i,\cdot}\|_{2}^{2}\big)+\sum_{i=1}^{n_{2}}G_{0}\big(\alpha_{4}\|\bm{R}_{i,\cdot}\|_{2}^{2}\big)\Bigg\}
\end{align}
for some scalars $\alpha_1,\cdots,\alpha_4>0$. 
There are numerous choices of $G_0$; the one suggested by \cite{sun2016guaranteed} is $G_0(z)=\max\{z-1,0\}^2$.
With suitable learning rates and proper initialization, GD w.r.t.~the  regularized loss  provably yields $\varepsilon$-accuracy in $O(\mathsf{poly}(n)\log \frac{1}{\varepsilon})$ iterations, provided that the sample size  $n^2 p \gtrsim \mu^2 \kappa^6 nr^7 \log n$ \cite{sun2016guaranteed}.  

	\item {\em Blind deconvolution~\eqref{eq:fmin-BD}:} \cite{li2016deconvolution, ling2017regularized,huang2017blind} suggest the following regularized loss
	%
\begin{align*}
 		& f_{\mathrm{reg}}(\bm{h},\bm{x})=f(\bm{h},\bm{x}) +\lambda \sum_{i=1}^{m}G_{0}\Big(\frac{m |\bm{b}_{i}^{*}\bm{h} |^{2}}{8\mu^{2}\|\bm{h}_{\star}\|_{2}\|\bm{x}_{\star}\|_{2}}\Big) \mylinebreak
		\myalign \myquad\myquad +\lambda G_{0}\Big(\frac{ \|\bm{h} \|_2^{2}}{2\|\bm{h}_{\star}\|_{2}\|\bm{x}_{\star}\|_{2}}\Big)
		+ \lambda G_{0}\Big(\frac{ \|\bm{x} \|_2^{2}}{2\|\bm{h}_{\star}\|_{2}\|\bm{x}_{\star}\|_{2}}\Big)
\end{align*}
with  $G_0(z):=\max\{z-1,0\}^2$. It has been demonstrated that under proper initialization and step size, gradient methods w.r.t.~the regularized loss reach $\varepsilon$-accuracy in $O(\mathsf{poly}(m) \log\frac{1}{\varepsilon})$ iterations, provided that the sample size $m\gtrsim {(K+N)\log ^2 m}$ \cite{li2016deconvolution}. 
	
\end{itemize}
In both cases, the regularization terms penalize, among other things, the incoherence measure between the decision variable and the corresponding sampling basis. 
We note, however, that the regularization terms are often found unnecessary in both theory and practice, and the theoretical guarantees derived in this line of works are also subject to improvements, as unveiled in the next subsection (cf. Section~\ref{sec:Implicit-Reg}).

Two other regularization approaches are also worth noting:  (1) {\em truncated gradient descent}; (2) {\em projected gradient descent}. Given that they are extensions of vanilla GD and might enjoy additional benefits, we postpone the discussions of them to  Section \ref{sec:GD_variants}.



\subsection{The phenomenon of implicit regularization}
\label{sec:Implicit-Reg}

 
 
Despite the theoretical success of regularized GD,  it is often observed  that vanilla GD --- in the absence of any regularization --- converges geometrically fast in practice. One intriguing fact is this:  for the problems mentioned above, GD automatically forces its iterates to stay incoherent with the sampling vectors\,/\,matrices, without any need of explicit regularization  \cite{ma2017implicit,li2018nonconvex,shi2018demising}. This means that with high probability, the entire GD trajectory lies within a nice region that enjoys desired strong convexity and smoothness,  thus enabling fast convergence.  

To illustrate this fact, we display in Fig.~\ref{fig:implicit-reg} a typical GD trajectory. The incoherence region --- which  enjoys local strong convexity and smoothness --- is often a polytope (see the shaded region in the right panel of Fig.~\ref{fig:implicit-reg}). The implicit regularization suggests that with high probability, the entire GD trajectory is constrained within this polytope, thus exhibiting linear convergence. It is worth noting that this cannot be derived from generic GD theory like Lemma~\ref{lem:GD-convergence}. For instance, Lemma~\ref{lem:GD-convergence} implies that starting with a good initialization, the next iterate experiences $\ell_2$ error contraction,  but it falls short of enforcing the incoherence condition and hence does not preclude the iterates from leaving the polytope.

\begin{figure}
	\centering
	\includegraphics[width=0.24\textwidth]{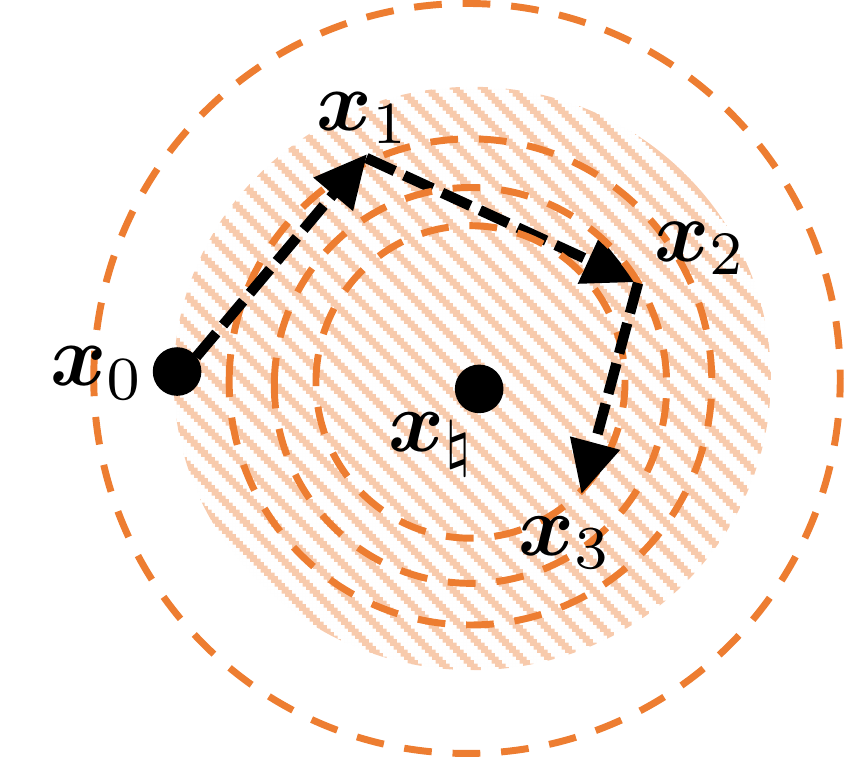}  \myquadinv \includegraphics[width=0.24\textwidth]{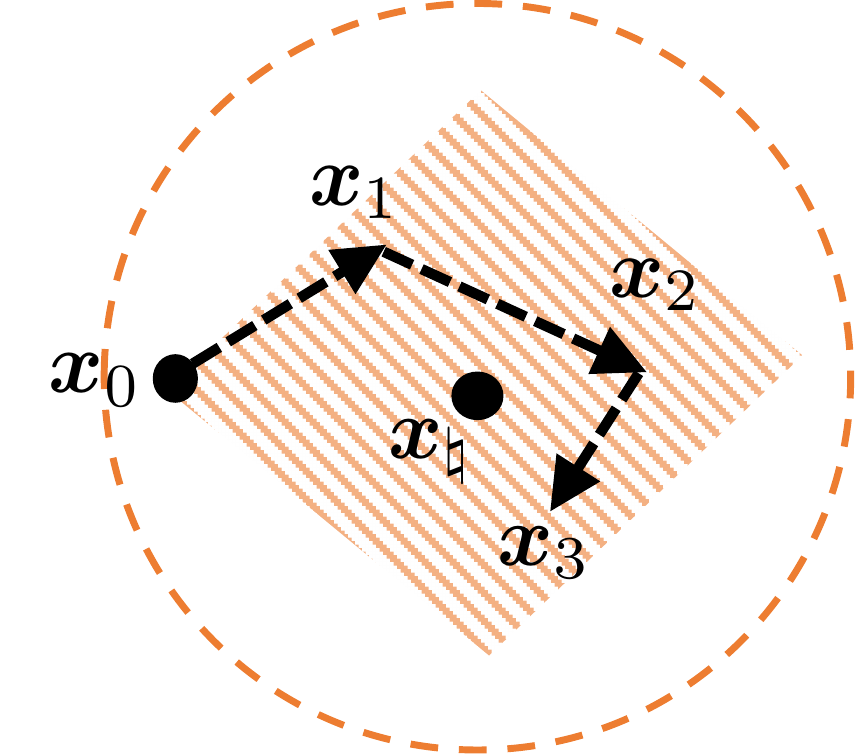}
	\caption{The GD iterates and the locally strongly convex and smooth region (the shaded region). (Left) When this region is an $\ell_2$ ball, then standard GD theory implies $\ell_2$ convergence. (Right) When this region is a polytope, the implicit regularization phenomenon implies that the GD iterates still stay within this nice region.   \label{fig:implicit-reg}}
\end{figure}

In the sequel, we start with  phase retrieval as the first example:  
\begin{theorem}[$\mathsf{GD~for~phase~retrieval~(improved~bound)}$ \cite{ma2017implicit}]
\label{thm:GD-PR-improved}
	Under the assumptions of Lemma \ref{lem:strong-cvx-PR},  the GD iterates with proper initialization (see, e.g., spectral initialization in Section \ref{sec:vanilla-spectral}) and $\eta_t \equiv 1/(c_3\|\bm{x}_{\star}\|_2^2\log n)$ obey
\begin{subequations}
\begin{align}
	& \|\bm{x}_{t}-\bm{x}_{\star}\|_{2} \leq\left(1- \frac{1}{c_4\log n}\right)^{t}\|\bm{x}_{0}-\bm{x}_{\star}\|_{2}  \label{eq:PR-tight} \\
	& \quad \max_{1\leq i\leq m} |\bm{a}_i^{\top}(\bm{x}_t - \bm{x}_\star)| \lesssim \sqrt{\log m} \|\bm{x}_{\star}\|_2  \label{eq:PR-incoherent}
\end{align}
\end{subequations}
for all $t\geq 0$
with probability $1-O(n^{-10})$. Here, $c_3,c_4>0$ are some constants, and  we assume $\|\bm{x}_{0}-\bm{x}_{\star}\|_2 \leq \|\bm{x}_{0}+\bm{x}_{\star}\|_2$. 
\end{theorem}
In words, \eqref{eq:PR-incoherent} reveals that all iterates are incoherent w.r.t.~the sampling vectors, and hence fall within the nice region characterized in Lemma \ref{lemma:wf_hessian}. With this observation in mind,  it is shown in \eqref{eq:PR-tight}  that vanilla GD converges in $O(\log n\log \frac{1}{\varepsilon})$ iterations. This significantly improves upon the computational bound in Theorem \ref{thm:GD-PR-loose} derived based on the smoothness property without restricting attention to the incoherence region. 

Similarly, for quadratic sensing~\eqref{eq:min-PR_lowrank} where the GD update rule is given by
\begin{equation}\label{eq:gradient-QS-formula}
\bX_{t+1} = \bX_{t} -  \frac{\eta_{t}}{m} \sum_{i=1}^{m } \left( \big\Vert \ba_{i}^{\top}\bX_{t} \big\Vert_{2}^{2} - y_{i} \right) \ba_{i}\ba_{i}^{\top}\bX_{t},
\end{equation}
we have the following result, which generalizes Theorem~\ref{thm:GD-PR-improved} to the low-rank setting.
\begin{theorem}[$\mathsf{GD~for~quadratic~sensing}$ \cite{li2018nonconvex}]
\label{thm:GD-QS-improved}
	Consider the problem~\eqref{eq:min-PR_lowrank}. Suppose the sample size satisfies $m\geq c_0  n r^{4} \kappa^{3} \log{ n} $ for some large constant $c_0$, then with probability $1-O(mn^{-10})$, the GD iterates with proper initialization (see, e.g., spectral initialization in Section \ref{sec:vanilla-spectral}) and $\eta_t =\eta \equiv  1/(c_1( r\kappa + \log{n})^{2} \sigma_1(\bM_{\star}))$ obey
\begin{align}
	& 	\mathsf{dist}( \bm{X}_{t}, \bm{X}_{\star})  \leq\left(1- \frac{\eta\sigma_1(\bM_{\star})}{2}\right)^{t}  \left\|\bX_{\star}\right\|_{\mathrm{F}} \label{eq:QS-tight} 
\end{align}
for all $t\geq 0$. Here, $c_1>0$ is some absolute constant. 
\end{theorem}
This theorem demonstrates that vanilla GD converges within $O\big(\max\{r,\log n\}^2 \log \frac{1}{\varepsilon} \big)$ iterations for quadratic sensing of a rank-$r$ matrix. This significantly improves upon the computational bounds in \cite{sanghavi2017local} which do not consider the incoherence region.

The next example is matrix completion \eqref{eq:MC-empirical-risk}, for which the GD update rule reads
\begin{align}
	\bm{X}_{t+1}=\bm{X}_t - \frac{\eta_t}{p}\mathcal{P}_{\Omega}\left(\bm{X}_t\bm{X}_t^{\top}-\bm{M}_{\star}\right)\bm{X}_t 
	\label{eq:gradient-MC-formula-clean}
\end{align}
with $\mathcal{P}_{\Omega}$ defined in \eqref{defn:Pomega}. The theory for this update rule is: 

\begin{theorem}[$\mathsf{GD~for~matrix~completion}$ \cite{ma2017implicit,chen2019nvx}]
\label{thm:main-MC}
Consider the problem \eqref{eq:MC-empirical-risk}.  Suppose that
the sample size satisfies $n^{2}p\geq c_0\mu^{2}r^{2}n\log n$ for
	some large constant $c_0>0$, and that the condition number $\kappa$ of $\bm{M}_{\star}= \bm{X}_{\star}\bm{X}_{\star}^{\top}$ is a fixed constant. 
With probability at least $1-O\left(n^{-3}\right)$, the GD iterates \eqref{eq:gradient-MC-formula-clean} with proper initialization (see, e.g., spectral initialization in Section \ref{sec:vanilla-spectral})
 satisfy 
\begin{subequations} \label{eq:induction_original_MC_thm}
\begin{align}
	\big\|\bm{X}_{t} {\bm{H}}_{\bm{X}_{t}}-\bm{X}_{\star}\big\|_{\mathrm{F}} & \lesssim \left( 1 - c_1 \right)^{t}\mu r\frac{1}{\sqrt{np}} \big\|\bm{X}_{\star}\big\|_{\mathrm{F}} \label{eq:induction_original_ell_2-MC_thm}\\
	\big\|\bm{X}_{t} {\bm{H}}_{\bm{X}_{t}}-\bm{X}_{\star}\big\|_{2,\infty} & \lesssim \left( 1 - c_1 \right)^{t} \mu r\sqrt{\frac{\log n}{np}} \big\|\bm{X}_{\star}\big\|_{2,\infty} \label{eq:induction_original_ell_infty-MC_thm}
\end{align}
\end{subequations}
for all $t\geq 0$, with $\bm{H}_{\bm{X}_t}$ defined in \eqref{defn:HX}. Here, $c_{1}>0$ is some constant, and $\eta_{t}\equiv {c_2} / (\kappa \sigma_{1}(\bm{M}_{\star}))$ for some constant $c_2>0$. \end{theorem}
This theorem demonstrates that vanilla GD converges within $O\big( \log \frac{1}{\varepsilon} \big)$ iterations.  The key enabler of such a convergence rate is the property \eqref{eq:induction_original_ell_infty-MC_thm}, which basically implies that the GD iterates stay incoherent with the standard basis vectors.  A byproduct is that:  GD converges not only in Euclidean norm, it also converges in other more refined error metrics, e.g.~the one measured by the $\ell_2/\ell_{\infty}$ norm.

The last example is  blind deconvolution. To measure the discrepancy between any
$\bm{z}:=${\scriptsize$\begin{bmatrix}
\bm{h} \\
\bm{x}
\end{bmatrix}$} 
and
$\bm{z}_{\star}:=${\scriptsize$\begin{bmatrix}
	\bm{h}_{\star} \\
	\bm{x}_{\star}
\end{bmatrix}$}, we define 
\begin{equation}
	\mathsf{dist}_{\mathsf{bd}}\left(\bm{z},\bm{z}_{\star}\right):= \min_{\alpha\in\mathbb{C}} \sqrt{\Big\Vert \frac{1}{\overline{\alpha}}\bm{h}-\bm{h}_{\star}\Big\Vert _{2}^{2}+\left\Vert \alpha\bm{x}-\bm{x}_{\star}\right\Vert _{2}^{2}},\label{eq:defn-dist-BD}
\end{equation}
which accounts for unrecoverable global scaling and phase. The gradient method, also called {\em Wirtinger flow (WF)}, is 
\begin{subequations}\label{eq:gradient-update-Bd-explicit}
\begin{align}
\bm{h}_{t+1} & =\bm{h}_{t}-\frac{\eta_t}{\left\Vert \bm{x}_{t}\right\Vert _{2}^{2}}\sum_{j=1}^{m}\left(\bm{b}_{j}^{\conj}\bm{h}_{t}\bm{x}_t^{\conj}\bm{a}_{j}-y_{j}\right)\bm{b}_{j}\bm{a}_{j}^{\conj}\bm{x}_{t},\label{eq:gradient-update-h-Bd}\\
\bm{x}_{t+1} & =\bm{x}_{t}-\frac{\eta_t}{\left\Vert \bm{h}_{t}\right\Vert _{2}^{2}}\sum_{j=1}^{m}\overline{(\bm{b}_{j}^{\conj}\bm{h}_{t}\bm{x}_t^{\conj}\bm{a}_{j}-y_{j})}\,\bm{a}_{j}\bm{b}_{j}^{\conj}\bm{h}_{t},\label{eq:gradient-update-x-BD}
\end{align}
\end{subequations}
which enjoys the following theoretical support. 
\begin{theorem}[$\mathsf{WF~for~blind~deconvolution}$ \cite{ma2017implicit}]
	\label{thm:main-BD}
	Consider the problem \eqref{eq:fmin-BD}. 
	Suppose the sample size $m\geq c_0\mu^{2}\max \{K,N\}\mathsf{poly}\log m$ for some large constant
$c_0>0$. Then there is some constant $c_1>0$ such that with probability exceeding $1- O(\min\{K,N\}^{-5})$,
the iterates \eqref{eq:gradient-update-Bd-explicit} with proper initialization (see, e.g., spectral initialization in Section \ref{sec:vanilla-spectral}) and $\eta_t \equiv c_1$ satisfy 
%
\label{eq:BD_thm}
\begin{align}
\mathsf{dist}_{\mathsf{bd}}\left(\bm{z}_{t},\bm{z}_{\star}\right) & \lesssim \left(1-\frac{\eta}{16}\right)^{t}\frac{1}{\log^{2}m}\left\Vert\bm{z}_{\star}\right\Vert_{2}, \quad \forall t\geq 0. \label{eq:BD-thm-ell-2}
\end{align}
\end{theorem}
For conciseness, we only state that the estimation error converges in $O\big( \log\frac{1}{\varepsilon} \big)$ iterations. The incoherence conditions also provably hold across all iterations; see \cite{ma2017implicit} for details. Similar results have been derived for the blind demixing case as well \cite{shi2018demising}.

Finally, we remark that the desired incoherence conditions  cannot be established via generic optimization theory.  Rather, these are proved by exploiting delicate statistical properties underlying the models of interest. The key technique is called a ``{\em{leave-one-out}}'' argument, which is rooted in probability and random matrix theory and finds applications to many problems \cite{stein1972bound,chen2010normal,el2015impact,zhong2017near,sur2017likelihood,chen2017spectral,abbe2017entrywise,chen2018gradient,ding2018leave,chen2019noisy,chen2019inference,chen2019nvx}.  The interested readers can find a general recipe using this argument in \cite{ma2017implicit}. 
%
%
%

\subsection{Notes} 

Provably valid two-stage nonconvex algorithms for matrix factorization were pioneered by Keshavan et al.~\cite{keshavan2010matrix,keshavan2010noisy}, where the authors studied the spectral method followed by (regularized) gradient descent on Grassmann manifolds. Partly due to the popularity of convex programming, the local refinement stage of \cite{keshavan2010matrix,keshavan2010noisy} received less attention than convex relaxation and spectral methods  around that time.  A recent work that  popularized the gradient  methods for matrix factorization is Cand\`es et al.~\cite{candes2015phase}, which provided the first  convergence guarantees for gradient descent (or Wirtinger flow) for phase retrieval. Local convergence of (regularized) GD was later established for matrix completion (without resorting to Grassmann manifolds) \cite{sun2016guaranteed},  matrix sensing \cite{tu2015low,zheng2015convergent},  and blind deconvolution under subspace prior \cite{li2016deconvolution}.  These works were all based on regularity conditions within a local ball.  The resulting iteration complexities for phase retrieval, matrix completion, and blind deconvolution were all sub-optimal, which  scaled at least linearly with the problem size. Near-optimal computational guarantees were first derived by \cite{ma2017implicit} via a leave-one-out analysis. Notably, all of these works are local results and rely on proper initialization. Later on, GD was shown to converge within a logarithmic number of iterations for phase retrieval, even with random initialization \cite{chen2018gradient}. 



%


\section{Variants of gradient descent}
\label{sec:GD_variants}
 

This section introduces several important variants of gradient descent that serve different purposes, including improving computational performance, enforcing additional structures of the estimates, and removing the effects of outliers, amongst others. 
Due to space limitation, our description of these algorithms cannot be as detailed as that of vanilla GD. Fortunately,  many of the insights and analysis techniques introduced in Section  \ref{sec:gd} are still applicable, which already shed light on how to understand and analyze these variants.   In addition,  we caution that all of the theory presented herein is developed for the idealistic models described in Section \ref{sec:examples}, which might sometimes not capture realistic measurement models.  Practitioners should perform comprehensive comparisons of these algorithms on real data, before deciding on which one to employ in practice.


\subsection{Projected gradient descent}
\label{sec:projected-GD}

Projected gradient descent  modifies vanilla GD \eqref{eq:GD-general} by adding a projection step in order to enforce additional structures of the iterates, that is
\begin{equation}
\bm{x}_{t+1}=\mathcal{P}_{\mathcal{C}}\left(\bm{x}_{t}-\eta_{t}\nabla f\big(\bm{x}_{t}\big)\right),\label{eq:PGD}
\end{equation}
where the constraint set $\mathcal{C}$ can be either convex or nonconvex. For many important sets $\mathcal{C}$ encountered in practice, the projection step can be implemented efficiently, sometimes even with a closed-form solution. There are two common purposes for including a projection step: 1) to enforce the iterates to stay in a region with benign geometry, whose importance has been explained in Section~\ref{sec:restricted_geometry}; 2) to encourage additional low-dimensional structures of the iterates that may be available from prior knowledge.  

\subsubsection{Projection for computational benefits}

Here, the projection is to ensure the running iterates stay incoherent with the sampling basis, a property that is crucial to guarantee the algorithm descends properly in every iteration (see Section~\ref{sec:restricted_geometry}). One notable example serving this purpose is projected GD for matrix completion \cite{chen2015fast,zheng2016convergence,yi2016fast}, where in the positive semidefinite case (i.e.~$\bm{M}_{\star}=\bm{X}_{\star}\bm{X}_{\star}^{\top}$), one runs projected GD w.r.t.~the loss function $f(\cdot)$ in \eqref{eq:MC-empirical-risk}:
	\begin{align}\label{eq:PGDgrad}
	\bX_{t+1}& = \mathcal{P}_{\mathcal{C}} \big(\bX_{t} - \eta_{t} \nabla f(\bX_{t}) \big),  
	\end{align}
	where $ \eta_{t} $ is the step size and $ \mathcal{P}_{\mathcal{C}}$ denote the Euclidean projection onto the set of incoherent matrices:
\begin{align}\label{eq:PGD_set}
\mathcal{C}  := \Big\{\bX \in \mathbb{R}^{n \times r} ~
\big|~ \|\bX \|_{2,\infty} \le
\sqrt{\frac{ c \mu r}{n}}\|\bX_{0}\| \Big\},
\end{align}
with $\bX_0$ being the initialization and $c$ is a predetermined constant (e.g. $c=2$). This projection guarantees that the iterates stay in a nice incoherent region w.r.t.~the sampling basis (similar to the one prescribed in Lemma~\ref{lemma:hessian-MC}), thus achieving fast convergence. Moreover, this projection can be implemented via a row-wise ``clipping'' operation, given as
\begin{align*}
[\mathcal{P}_{\mathcal{C}}(\bX)]_{i,\cdot} =  \min \left\{1 ,
\sqrt{\frac{c\mu r}{n}} \frac{ \|\bX_{0}\| } { \|\bX_{i,\cdot}\|_2 } \right\} \cdot \bX_{i,\cdot}  ,
\end{align*}
for $i = 1,2,\ldots,n$. The convergence guarantee for this update rule is given below, which offers slightly different prescriptions in terms of sample complexity and convergence rate from Theorem~\ref{thm:main-MC} using vanilla GD.
\begin{theorem}[$\mathsf{Projected~GD~for~matrix~completion}$ \cite{chen2015fast,yi2016fast}]
\label{thm:main-MC-PGD}
Suppose that
the sample size satisfies $n^{2}p\geq c_0\mu^{2}r^{2}n\log n$ for
	some large constant $c_0>0$, and that the condition number $\kappa$ of $\bm{M}_{\star}= \bm{X}_{\star}\bm{X}_{\star}^{\top}$ is a fixed constant. 
With probability at least $1-O\left(n^{-1}\right)$, the projected GD iterates
	\eqref{eq:PGDgrad} satisfy 
\begin{align}
\mathsf{dist}^2( \bm{X}_{t}, \bm{X}_{\star} ) \leq \left(1-\frac{c_1}{\mu r }\right)^{t} \mathsf{dist}^2( \bm{X}_{0}, \bm{X}_{\star} ), 
\end{align} 
for all $t\geq 0$, provided that  $\mathsf{dist}^2( \bm{X}_{0}, \bm{X}_{\star} )\leq c_3 \sigma_{r}(\bm{M}_{\star})$ and $\eta_{t}\equiv \eta := {c_2} / (\mu r \sigma_{1}(\bm{M}_{\star}))$ for some constant $c_1,c_2,c_3>0$. \end{theorem}
 
This theorem says that projected GD takes $O(\mu r \log\frac{1}{\varepsilon})$ iterations to yield $\varepsilon$-accuracy (in a relative sense). 

\begin{remark}
The results can be extended to the more general asymmetric case by applying similar modifications mentioned in Remark \ref{rmk:GD_asymetric}; see \cite{zheng2016convergence,yi2016fast}.
\end{remark}	

\subsubsection{Projection for incorporating structural priors} \label{sec:pgd_sparse_pr}
In many problems of practical interest, we might be given some prior knowledge about the signal of interest, encoded by a constraint set $\bm{x}_{\star}\in\mathcal{C}$. Therefore, it is natural to apply projection to enforce the desired structural constraints. One such example is sparse phase retrieval \cite{soltanolkotabi2017structured,cai2016optimal}, where it is known {\em a priori} that $\bm{x}_{\star}$ in \eqref{eq:PR-samples} is $k$-sparse, where $k\ll n$. 
If we have prior knowledge about $\|\bx_{\star}\|_1$, then we can pick the constraint set $\mathcal{C}$ as follows to promote sparsity 
\begin{equation}\label{eq:SPR_constraint}
\mathcal{C} = \big\{\bx\in\mathbb{R}^n : \quad \|\bx\|_1 \leq \|\bx_{\star}\|_1 \big\},
\end{equation}
as a sparse signal often (although not always) has low $\ell_1$ norm. 
With this convex constraint set in place, applying projected GD w.r.t.~the loss function \eqref{eq:min-PR} can be efficiently implemented \cite{duchi2008efficient}. The theoretical guarantee of projected GD for sparse phase retrieval is given below.
\begin{theorem}[$\mathsf{Projected~GD~for~sparse~phase~retrieval}$ \cite{soltanolkotabi2017structured}]
\label{thm:PGD-SPR-loose}
Consider the sparse phase retrieval problem where $\bx_{\star}$ is $k$-sparse. Suppose that $m\geq c_1 k \log n$ for some large constant $c_1>0$. The projected GD iterates w.r.t. \eqref{eq:min-PR} and the constraint set \eqref{eq:SPR_constraint} obey
\begin{equation}
\|\bm{x}_{t}-\bm{x}_{\star}\|_{2}\leq\left(1-\frac{1}{2c_{2}n}\right)^{t}\|\bm{x}_{0}-\bm{x}_{\star}\|_{2},\quad t\geq 0, \label{eq:PR-loose}
\end{equation}
with probability at least $1-O(n^{-1})$, provided that $\eta_t \equiv 1/(c_2n\|\bm{x}_{\star}\|_2^2)$ and $\|\bm{x}_0-\bm{x}_{\star}\|_2 \leq \|\bm{x}_{\star}\|_2  / 8$. 
\end{theorem}
 
Another possible projection constraint set for sparse phase retrieval is the (nonconvex) set of $k$-sparse vectors \cite{wang66sparse}, 
\begin{equation}
	\mathcal{C}= \big\{\bx\in\mathbb{R}^n :   \|\bx\|_0=k \big\}.
\end{equation}
This leads to a hard-thresholding operation, namely, $ \mathcal{P}_{\mathcal{C}}(\bx) $ becomes the best $k$-term approximation of $\bm{x}$ (obtained by keeping the $k$ largest entries (in magnitude) of $\bm{x}$ and setting the rest to 0). The readers are referred to \cite{wang66sparse} for details. See also \cite{cai2016optimal} for a thresholded GD algorithm --- which enforces adaptive thresholding rather than projection to promote sparsity --- for solving the sparse phase retrieval problem.

We caution, however, that Theorem \ref{thm:PGD-SPR-loose} does not imply that the sample complexity for projected GD (or thresholded GD) is $O(k \log n)$. So far there is no tractable  procedure that can provably guarantee a sufficiently good initial point $\bm{x}_0$ when $m\lesssim k\log n$ (see a discussion of the spectral initialization method in Section \ref{sec:spectral_init_spr}). Rather, all computationally feasible algorithms (both convex and nonconvex) analyzed so far require sample complexity at least on the order of $k^2 \log n$ under i.i.d.~Gaussian designs, unless $k$ is sufficiently large or other structural information is available \cite{li2013sparse,oymak2012simultaneously,chen2015exact,cai2016optimal,jagatap2017fast}.  All in all, the computational bottleneck for sparse phase retrieval lies in the initialization stage.


%

\subsection{Truncated gradient descent}
\label{sec:truncated-GD}


Truncated gradient descent  proceeds by trimming away a subset of
the measurements when forming the descent direction, typically performed adaptively. We can express it as
\begin{equation}
\bm{x}_{t+1}=\bm{x}_{t}-\eta_{t}\mathcal{T}\left(\nabla f\big(\bm{x}_{t}\big)\right),\label{eq:TGD-general}
\end{equation}
where $\mathcal{T}$ is an operator that effectively drops samples that bear undesirable influences over the search directions.

There are two common purposes for enforcing a truncation step: (1) to remove samples whose associated design vectors are too coherent with the current iterate \cite{chen2015solving,kolte2016phase,wang2017solving}, in order to  accelerate convergence and improve sample complexity; (2) to remove samples that may be adversarial outliers, in the hope of improving  robustness of the algorithm \cite{zhang2016provable,yi2016fast,li2017nonconvex}.

%
\subsubsection{Truncation for  computational and statistical benefits}

We use phase retrieval to illustrate this benefit.  All results discussed so far require a sample size that exceeds $m\gtrsim n \log n$. When it comes to the  sample-limited regime where $m\asymp n$, there is no guarantee for strong convexity (or regularity condition) to hold.  This presents significant challenges for nonconvex methods, in a regime of critical importance for practitioners.  

To better understand the challenge, recall  the GD rule \eqref{eq:PR-GD}. When $m$ is exceedingly large, the negative gradient concentrates around the population-level gradient, which forms a reliable search direction. However, when $m\asymp n$, the gradient --- which depends on 4th moments of $\{\bm{a}_i\}$ and is heavy-tailed --- may deviate significantly from the mean,  thus resulting in unstable search directions.

To stabilize the search directions, one strategy is to trim away those gradient components $\{\nabla f_i(\bx_t): = \big((\bm{a}_{i}^{\top}\bm{x}_t)^{2}-y_{i}\big) \bm{a}_{i}\bm{a}_{i}^{\top}\bm{x}_{t}  \}$ whose size deviate too much from the typical size. Specifically, the  truncation rule proposed in \cite{chen2015solving} is:\footnote{Note that the original algorithm proposed in \cite{chen2015solving} is designed w.r.t.~the Poisson loss, although all theory goes through for the current loss.} 
\begin{eqnarray}
	\boldsymbol{x}_{t+1} = \boldsymbol{x}_{t}- \eta_{t} \underset{:= \nabla f_{\mathrm{tr}}(\bx_t) }{\underbrace{ \frac{1}{m} \sum_{i=1}^m \nabla f_i(\bx_t) 
	{\ind}_{\mathcal{E}_{1}^{i}\left(\boldsymbol{x}_t\right)\cap\mathcal{E}_{2}^{i}\left(\boldsymbol{x}_t\right)} }} , \quad t\geq 0	\label{eq:TWF-update}
\end{eqnarray}
 for some trimming criteria  defined as
\begin{align*}
	&\mathcal{E}_{1}^{i} (\boldsymbol{x})  :=  \bigg\{ \alpha_{\text{lb}}\leq\frac{\left|\boldsymbol{a}_i^{\top}\boldsymbol{x} \right|}{\Vert \boldsymbol{x} \Vert_2 }\leq\alpha_{\text{ub}}\bigg\} ,
	\\
	&\mathcal{E}_{2}^{i} (\boldsymbol{x})  :=  \bigg\{ |y_{i}-|\boldsymbol{a}_i^{\top}\boldsymbol{x} |^{2} |\leq\frac{\alpha_{h}}{m} \Big(\sum_{j=1}^{m}\big|y_{i}-(\bm{a}_{j}^{\top}\bm{x})^{2}\big| \Big) \frac{\left|\boldsymbol{a}_i^{\top}\boldsymbol{x}\right|} {\Vert \boldsymbol{x}\Vert_2 } \bigg\} ,
\end{align*}
%
where $\alpha_{\text{lb}}$, $\alpha_{\text{ub}}$, $\alpha_h$ are predetermined thresholds.  This trimming rule --- called {\em Truncated Wirtinger flow} --- effectively removes the ``heavy tails'', thus leading to much better concentration and hence enhanced performance. 
\begin{theorem}[$\mathsf{Truncated~GD~for~phase~retrieval}$ \cite{chen2015solving}]
\label{thm:TWF-PR}
Consider the problem \eqref{eq:min-PR}.  With probability $1-O(n^{-10})$, the iterates \eqref{eq:TWF-update}  obey
\begin{equation}
	\|\bm{x}_{t}-\bm{x}_{\star}\|_{2}\leq \rho ^{t}\|\bm{x}_{0}-\bm{x}_{\star}\|_{2},\quad t\geq 0, \label{eq:TWF-PR}
\end{equation}
	for some constant $0<\rho<1$, provided that $m\geq c_1 n$, $\|\bm{x}_0 - \bm{x}_{\star}\|_2\leq c_2 \|\bm{x}_{\star}\|_2$ and $\eta_t \equiv c_3/\|\bm{x}_{\star}\|_2^2$   for some constants $c_1,c_2,c_3>0$. 
\end{theorem}
\begin{remark}
	In this case, the truncated gradient is clearly not smooth, and hence we need to resort to the regularity condition (see Definition \ref{def:reg-condition}) with $\bg(\bx)= \nabla f_{\mathrm{tr}}(\bx)$. Specifically, the proof of Theorem~\ref{thm:TWF-PR} consists of showing  
\[
2\langle\nabla f_{\mathrm{tr}}(\bx),\bm{x}-\bm{x}_{\star}\rangle\geq\mu\|\bm{x}-\bm{x}_{\star}\|_{2}^{2}+\lambda\|\nabla f_{\mathrm{tr}}(\bx)\|_{2}^{2}
\]
for all $\bx$ within a local ball around $\bx_{\star}$, where $\lambda, \mu\asymp 1$. See \cite{chen2015solving} for details. 
\end{remark}

In comparison to vanilla GD, the truncated version provably achieves two benefits:
\begin{itemize}
	\item {\em Optimal sample complexity:} given that one needs at least $n$ samples to recover $n$ unknowns,  the sample complexity $m \asymp n$ is orderwise optimal;
	\item {\em Optimal computational complexity:} truncated WF yields $\varepsilon$ accuracy in $O\big( \log \frac{1}{\varepsilon} \big)$ iterations. Since each iteration takes time proportional to that taken to read the data, the computational complexity is nearly optimal. 
\end{itemize}
At the same time, this approach is particularly stable in the presence of noise, which enjoys a statistical guarantee that is minimax optimal. The readers are referred to \cite{chen2015solving} for precise statements.

\subsubsection{Truncation for removing sparse outliers}
In many problems, the collected measurements may suffer from corruptions of sparse outliers, and the gradient descent iterates need to be carefully monitored to remove the undesired effects of outliers (which may take arbitrary values). Take robust PCA~\eqref{eq:RPCA-empirical-risk} as an example, in which a fraction $\alpha$ of revealed entries are corrupted by outliers. At the $t$th iterate, one can first try to identify the support of the sparse matrix $\bm{S}_{\star}$ by {\em hard thresholding} the residual, namely,
\begin{equation}
	\label{eq:hard_threshodling}
\bm{S}_{t+1}=\mathcal{H}_{c\alpha np} \big( \mathcal{P}_{\Omega}(\bm{\Gamma}_{\star}-\bX_t\bX_t^{\top}) \big). 
\end{equation}
Here, $c>1$ is some predetermined constant (e.g.~$c=3$), and the operator $\mathcal{H}_l(\cdot)$ is defined as
\[
	[\mathcal{H}_{l}(\bA)]_{j,k} := \begin{cases}  A_{j,k},  &\text{if }|A_{j,k}|\geq |A_{j,\cdot}^{(l)}| \text{ and }
							|A_{j,k}|\geq |A_{\cdot,k}^{(l)}|,  \\
							0, & \text{otherwise,}
					  \end{cases}
\]
where 
$A_{j,\cdot}^{(l)}$ (resp.~$A_{\cdot,k}^{(l)}$) denotes the $l$th largest entry (in magnitude) in the $j$th row (resp.~column) of $\bm{A}$.
The  idea is simple: an entry is likely to be an outlier if it is simultaneously among the largest entries in the corresponding row and column. The thresholded residual $\bm{S}_{t+1}$ then becomes our estimate of the sparse outlier matrix $\bm{S}_{\star}$ in the $(t+1)$-th iteration.  With this in place, we update the estimate for the low-rank factor by applying projected GD
\begin{equation}
\bm{X}_{t+1}=\mathcal{P}_{\mathcal{C}}\big(\bm{X}_{t}-\eta_t \nabla_{\bX} f\left(\bm{X}_{t}, \bm{S}_{t+1}\right)\big), 
\end{equation}
where $\mathcal{C}$ is the same as \eqref{eq:PGD_set} to enforce the incoherence condition. This method has the  following theoretical guarantee: 
\begin{theorem}[$\mathsf{Nonconvex~robust~PCA}$ \cite{yi2016fast}]
  \label{thm:RPCA-nonconvex}
Assume that the condition number $\kappa$ of $\bm{M}_{\star}= \bm{X}_{\star}\bm{X}_{\star}^{\top}$ is a fixed constant. 
Suppose that the sample size and the sparsity of the outlier satisfy $n^{2}p\geq c_0\mu^{2}r^{2}n\log n$ and $\alpha\leq c_1/( \mu r)$ for
	some constants $c_0,c_1>0$.   
With probability at least $1-O\left(n^{-1}\right)$, the iterates satisfy 
\begin{align}
		\mathsf{dist}^2( \bm{X}_{t}, \bm{X}_{\star} ) \leq \left(1-\frac{c_2}{\mu r  }\right)^{t} \mathsf{dist}^2( \bm{X}_{0}, \bm{X}_{\star} ), 
\end{align} 
for all $t\geq 0$, provided that  $\mathsf{dist}^2( \bm{X}_{0}, \bm{X}_{\star} )\leq c_3 \sigma_{r}(\bm{M}_{\star})$. Here, $0<c_2,c_3<1$ are some constants, and $\eta_{t}\equiv {c_4} / (\mu r \sigma_{1}(\bm{M}_{\star}))$ for some constant $c_4>0$. 
\end{theorem}
 
 \begin{remark}
 In the full data case, the convergence rate can be improved to $1-c_2$ for some constant $0<c_2<1$. 
 \end{remark}

 This theorem essentially says that: as long as the fraction of entries corrupted by  outliers does not exceed $O(1/\mu r)$, then the nonconvex algorithm described above provably recovers the true low-rank matrix in about $O(\mu r)$ iterations (up to some logarithmic factor). When $r=O(1)$, it means that the nonconvex algorithm succeeds even when a constant fraction of entries are corrupted.

Another truncation strategy to remove outliers is based on the sample median, as the median is known to be robust against arbitrary outliers \cite{zhang2016provable,li2017nonconvex}. We illustrate this median-truncation approach through an example of robust phase retrieval \cite{zhang2016provable}, where we assume a subset of samples in \eqref{eq:PR-samples} is corrupted arbitrarily, with their index set denoted by $\mathcal{S}$ with $|\mathcal{S}|=\alpha  m$. Mathematically, the measurement  model in the presence of outliers is given by
\begin{equation} \label{eq:robust_pr_model}
	y_i = \begin{cases}
(\ba_i^\top\bx_{\star})^2, \quad  & i \notin \mathcal{S}, \\
\mathsf{arbitrary} , & i \in \mathcal{S}. 
	\end{cases}
\end{equation}
The goal is to still recover $\bm{x}_{\star}$ in the presence of many outliers (e.g.~a constant fraction of measurements are outliers).
 
It is obvious that the original GD iterates \eqref{eq:PR-GD} are not robust, since the residual $$r_{t,i} := (\bm{a}_{i}^{\top}\bm{x}_t)^{2}-y_{i}$$ can be perturbed arbitrarily if $i\in\mathcal{S}$. Hence, we instead include only a subset of the samples when forming the search direction, yielding a truncated GD update rule
\begin{equation}
	\bm{x}_{t+1}=\bm{x}_{t}- \frac{\eta_{t}}{m}\sum_{i\in\mathcal{T}_t} \big((\bm{a}_{i}^{\top}\bm{x}_t)^{2}-y_{i}\big)\bm{a}_{i}\bm{a}_{i}^{\top}\bm{x}_{t} .
	\label{eq:PR-GD-median}
\end{equation}
Here, $\mathcal{T}_t$ only includes samples whose residual size $|r_{t,i}|$ does not deviate much from the {\em median} of $\{|r_{t,j}|\}_{1\leq j\leq m}$:
\begin{align}
	\mathcal{T}_{t} &:= \{i:   |r_{t,i}|   \lesssim \mathsf{median}(\{ |r_{t,j}| \}_{1\leq j\leq m}) \}, 
\end{align}
where $\mathsf{median}(\cdot)$ denotes the sample median.  As the iterates get close to the ground truth, we expect that the residuals of the clean samples will decrease and cluster, while the residuals remain large for outliers. In this situation,  the median provides a robust means to tell them apart. One has the following theory, which reveals the success of the median-truncated GD even when a constant fraction of measurements are arbitrarily corrupted.

\begin{theorem}[$\mathsf{Median\text{-}truncated~GD~for~robust~phase~retrieval}$ \cite{zhang2016provable}]
\label{thm:GD-RobustPR}
Consider the problem \eqref{eq:robust_pr_model} with a fraction $\alpha$ of arbitrary outliers.
There exist some constants  $c_0, c_1,c_2>0$ and $0<\rho<1$ such that if 
	$m\geq c_{0}n\log n$, $\alpha\leq c_1$, and $\|\bm{x}_0 - \bm{x}_{\star}\|_2\leq c_2 \|\bm{x}_{\star}\|_2$, then with probability at least $1-O(n^{-1})$, the median-truncated GD iterates satisfy 
%
\begin{align}
	& \|\bm{x}_{t}-\bm{x}_{\star}\|_{2} \leq \rho^{t}\|\bm{x}_{0}-\bm{x}_{\star}\|_{2},\qquad t\geq 0.  \label{eq:medianPR-tight}  
\end{align}
%
\end{theorem}

\subsection{Generalized gradient descent} 

In all the examples discussed so far, the loss function $f(\cdot)$ has been a smooth function. When $f(\cdot)$ is nonsmooth and non-differentiable, it is possible to continue to apply GD using the {\em generalized gradient} (e.g.~subgradient) \cite{clarke1975generalized}. As an example, consider again the phase retrieval problem but with an alternative loss function, where we minimize the quadratic loss of the amplitude-based measurements, given as
\begin{equation}\label{eq:amplitude_phase_retrieval}
	f_{\mathrm{amp}}(\bx) =\frac{1}{2m}\sum_{i=1}^{m}\big(|\bm{a}_{i}^{\top}\bm{x}| - \sqrt{y_i} \big)^{2} .
\end{equation}
Clearly, $f_{\mathrm{amp}}(\bx)$ is nonsmooth, and its generalized gradient is given by, with a slight abuse of notation,
\begin{equation}\label{eq:GD_AMP}
	\nabla f_{\mathrm{amp}}(\bx) :=\frac{1}{m}\sum_{i=1}^m \left(\ba_i^{\top}\bx- \sqrt{y_i} \cdot\mathsf{sgn}(\ba_i^{\top}\bx)\right)\ba_i.
\end{equation}
We can simply execute GD w.r.t.~the generalized gradient:
\[
	\bm{x}_{t+1} = \bm{x}_t - \eta_t \nabla f_{\mathrm{amp}}(\bx_t), \qquad t=0,1,\cdots 	
\]
This amplitude-based loss function $ f_{\mathrm{amp}}(\cdot)$ often has better curvature around the truth, compared to the intensity-based loss function $f(\cdot)$ defined in \eqref{eq:min-PR_lowrank}; see \cite{zhang2017reshaped,wang2017solving,davis2017nonsmooth} for detailed discussions. The theory is given as follows.
\begin{theorem}[$\mathsf{GD~for~amplitude\text{-}based~phase~retrieval}$ \cite{zhang2017reshaped}]
\label{thm:GD-RWF-improved}
Consider the problem \eqref{eq:min-PR}.
There exist some constants  $c_0, c_1, c_2>0$ and $0<\rho<1$ such that if 
	$m\geq c_{0}n$ and $\eta_t\equiv c_2$, then with high probability, 
\begin{align}
	& \|\bm{x}_{t}-\bm{x}_{\star}\|_{2} \leq \rho^{t}\|\bm{x}_{0}-\bm{x}_{\star}\|_{2},\qquad t=0,1,\cdots  \label{eq:medianPR-tight}  
\end{align}
as long as $\|\bm{x}_{0}-\bm{x}_{\star}\|_{2}\leq  \|\bm{x}_{\star}\|_{2}/10$.
\end{theorem}
In comparison to Theorem~\ref{thm:TWF-PR}, the generalized GD w.r.t.~$f_{\mathrm{amp}}(\cdot)$ achieves both order-optimal sample and computational complexities.  
Notably,  a very similar theory was obtained in \cite{wang2017solving} for a truncated version of the generalized GD (called {\em Truncated Amplitude Flow} therein), where the algorithm also employs the gradient update w.r.t.~$f_{\mathrm{amp}}(\cdot)$ but discards high-leverage data in a way similar to truncated GD discussed in Section \ref{sec:truncated-GD}.  However, in contrast to the intensity-based loss $f(\cdot)$ defined in \eqref{eq:min-PR_lowrank},  the truncation step is not crucial and can be safely removed when dealing with the amplitude-based $f_{\mathrm{amp}}(\cdot)$. 
A main reason is that for any fixed $\bm{x}$, $f_{\mathrm{amp}}(\cdot)$ involves only the first and second moments of the (sub)-Gaussian random variables $\{\ba_{i}^{\top}\bm{x}\}$. As such, it exhibits much sharper measure concentration --- and hence much better controlled gradient components --- compared  to the heavy-tailed $f(\cdot)$, which involves fourth moments of $\{\ba_{i}^{\top}\bm{x}\}$. 
This observation in turn  implies the importance of designing loss functions for nonconvex statistical estimation.



\subsection{Projected power method for constrained PCA}
\label{sec:PPM}

Many applications require solving a constrained quadratic maximization (or constrained PCA) problem:
\begin{subequations}\label{eq:quadratic-min}
\begin{align}
	\text{maximize} \quad  & f(\bx) = \bm{x}^{\top} \bm{L} \bm{x} ,  \\
	\text{subject to}  \quad    & \bm{x} \in \mathcal{C} ,
\end{align}
\end{subequations}
where $\mathcal{C}$ encodes the set of feasible points. This problem becomes nonconvex if either $\bm{L}$ is not negative semidefinite or if $\mathcal{C}$ is nonconvex. To demonstrate the value of studying this problem,  we introduce two important examples. 
\begin{itemize}
\item {\em Phase synchronization \cite{boumal2016nonconvex,liu2017estimation}.} 
Suppose we wish to recover $n$ unknown phases $\phi_{1},\cdots,\phi_{n}\in[0,2\pi]$
given their pairwise relative phases. Alternatively, by setting $(x_{\star})_{i}=e^{\phi_{i}}$,
		this problem reduces to estimating $\bm{x}_{\star}=[(x_{\star})_{i}]_{1\leq i\leq n}$
from $\bm{x}_{\star}\bm{x}_{\star}^{\conj}$ --- a matrix that
encodes all pairwise phase differences $(x_{\star})_{i}(x_{\star})_{j}^{\conj}=e^{\phi_{i}-\phi_{j}}$.
To account for the noisy nature of practical measurements, suppose
that what we observe is
$\bm{L}=\bm{x}_{\star}\bm{x}_{\star}^{\conj}+\sigma\bm{W}$,
where $\bm{\bm{W}}$ is a Hermitian matrix. Here, $\{W_{i,j}\}_{i\leq j}$ are i.i.d.~standard complex Gaussians. The quantity $\sigma$ indicates the noise level, which determines the hardness of the problem. A natural way to attempt
recovery is to solve the following problem
\begin{align*}
\text{maximize}_{\bm{x}}\quad & \bm{x}^{\conj}\bm{L}\bm{x}\quad\text{subject to}~|x_{i}|=1,~ 1\leq i\leq n.
\end{align*}
\item {\em Joint alignment \cite{chen2016information,chen2016projected}.} 
Imagine  we want to estimate $n$ \emph{discrete} variables
$\{(x_{\star})_{i}\}_{1\leq i\leq n}$, where each variable  can take $m$ possible values, namely, $(x_{\star})_{i}\in\{1,\cdots,m\}$.
Suppose that estimation needs to be performed based on pairwise difference samples $y_{i,j}=x_{i}-x_{j}+z_{i,j}~\mathsf{mod}~m$,
where the $z_{i,j}$'s are i.i.d.~noise and their distributions dictate the recovery limits.
To facilitate computation, one strategy
is to lift each discrete variable $x_{i}$ into a $m$-dimensional
vector $\bm{x}_{i}\in\{\bm{e}_{1},\cdots,\bm{e}_{m}\}$. We then introduce
a matrix $\bm{L}$ that properly encodes all log-likelihood information.
After simple manipulation (see \cite{chen2016projected} for details),
maximum likelihood estimation can be cast as follows
\begin{align*}
\text{maximize}_{\bm{x}}\quad & \bm{x}^{\top}\bm{L}\bm{x}\\
\text{subject to}\quad & \bm{x}=\left[\bx_i\right]_{1\leq i\leq n};\text{ }\bm{x}_{i}\in\{\bm{e}_{1},\cdots,\bm{e}_{m}\},~\forall i. \nonumber
\end{align*}
\end{itemize}
More examples of constrained PCA include an alternative formulation of phase retrieval \cite{waldspurger2015phase}, sparse PCA \cite{yuan2013truncated}, and multi-channel blind deconvolution with sparsity priors \cite{li2017blind}.

To solve \eqref{eq:quadratic-min}, two algorithms naturally come into mind. The first one is projected GD, which follows the update rule
\[
	\bm{x}_{t+1}=\mathcal{P}_{\mathcal{C}}(\bm{x}_{t}+\eta_{t}\bm{L}\bm{x}_{t}). 
\]
Another possibility is called the {\em projected power method (PPM)} \cite{chen2016projected,yuan2013truncated}, which drops the current iterate and performs projection only over the gradient component: 
\begin{align}
	\bm{x}_{t+1}=\mathcal{P}_{\mathcal{C}}(\eta_{t}\bm{L}\bm{x}_{t}).  \label{eq:PPM}
\end{align}
While this is perhaps best motivated by its connection to the canonical eigenvector problem (which is often solved by the power method), we remark on its close resemblance to  projected GD. 
In fact, for many constrained sets $\mathcal{C}$ (e.g.~the ones in phase synchronization and joint alignment), \eqref{eq:PPM} is equivalent to projected GD when the step size $\eta_t \rightarrow \infty$.

As it turns out, the PPM provably achieves near-optimal sample and computational complexities for the preceding two examples. Due to the space limitation, the theory is described only in passing. 

\begin{itemize}
	\item {\em Phase synchronization.} With high probability, the PPM with proper initialization converges linearly to the global optimum, as long as the noise level $\sigma \lesssim \sqrt{n /\log n}$. This is information theoretically optimal up to some log factor \cite{boumal2016nonconvex,zhong2017near}.
	\item {\em Joint alignment.}  With high probability, the PPM with proper initialization converges linearly to the ground truth, as long as certain Kullback-Leibler divergence w.r.t.~the noise distribution exceeds the information-theoretic threshold. See details in \cite{chen2016projected}.  
\end{itemize}

\subsection{Gradient descent on manifolds}


In many problems of interest, it is desirable to impose additional constraints on the object of interest, which leads to a constrained optimization problem over manifolds. In the context of low-rank matrix factorization, to eliminate global {\em scaling ambiguity}, one might constrain the low-rank factors to live on a Grassmann manifold  or a Riemannian quotient manifold \cite{edelman1998geometry,absil2009optimization}. 

To fix ideas, take matrix completion as an example. When factorizing $\bM_{\star}=\bL_{\star}\bR_{\star}^{\top}$, we might assume $\bL_{\star}\in\mathcal{G}(n_1,r)$, where $\mathcal{G}(n_1,r)$ denotes the Grassmann manifold which parametrizes all $r$-dimensional linear subspaces of the $n_1$-dimensional space\footnote{More specifically, any point in $\mathcal{G}(n,r)$ is an equivalent class of a $n\times r$ orthonormal matrix. See \cite{edelman1998geometry} for details. }.
In words, we are searching for a $r$-dimensional subspace $\bL$ but ignores the global rotation.  
It is also assumed that
 $\bL_{\star}^{\top}\bL_{\star}=\bI_r$ to remove the global scaling ambiguity (otherwise $(c\bm{L},c^{-1}\bm{R})$ is always equivalent to $(\bm{L},\bm{R})$ for any $c\neq 0$). One might then try to minimize the loss function defined over the Grassmann manifold as follows
\begin{equation}
	\text{minimize}_{\bL\in \mathcal{G}(n_1, r)} \qquad F(\bL) ,
\end{equation}
where 
\begin{equation}
	F(\bL) := \min_{\bR \in \mathbb{R}^{n_2\times r}}\left\|\mathcal{P}_{\Omega}(\bM_{\star} - \bL \bR^{\top}) \right\|_{\mathrm{F}}^2. 
\end{equation}
As it turns out, it is  possible to apply GD to $F(\cdot)$ over the Grassmann manifold by moving along the geodesics; here, a geodesic is the shortest path between
two points on a manifold. See \cite{edelman1998geometry} for an excellent overview. In what follows, we provide a very brief exposure to highlight its difference from a nominal gradient descent in the Euclidean space.

We start by writing out the conventional gradient of $F(\cdot)$ w.r.t.~the $t$th iterate $\bL_t$ in the {\em Euclidean space} \cite{balzano2014subspace}:
\begin{equation}
	\nabla F(\bL_t) = - 2 \mathcal{P}_{\Omega}(\bM_{\star} - \bL_t \hat{\bR}_t^{\top}) \hat{\bR}_t ,
\end{equation}
where $\hat{\bR}_t = \argmin_{\bR} \left\|\mathcal{P}_{\Omega}(\bM_{\star} - \bL_t \bR^{\top}) \right\|_{\mathrm{F}}^2$ is the least-squares solution. The gradient on the Grassmann manifold, denoted by $\nabla_{\mathcal{G}} F(\cdot)$, is then given by
\begin{align}
	\nabla_{\mathcal{G}} F(\bL_t) &= \left(\bI_{n_1} - \bL_t\bL_t^{\top}\right) \nabla F(\bL_t) . \nonumber 
\end{align}
%
Let $ - \nabla_{\mathcal{G}} F(\bL_t) = \tilde{\bU}_t\tilde{\bSigma}_t\tilde{\bV}_t^{\top}$ be its compact SVD, then the geodesic on the Grassmann manifold along the direction $- \nabla_{\mathcal{G}} F(\bL_t)$ is given by 
\begin{equation}
	\label{eq:geodesic}
	\bL_t(\eta) =\left[ \bL_t\tilde{\bV}_t \cos(\tilde{\bSigma}_t \eta) + \tilde{\bU}_t\sin(\tilde{\bSigma}_t \eta)  \right]\tilde{\bV}_t^{\top}.
\end{equation}
We can then update the iterates as  
\begin{equation}
	\bL_{t+1} = \bL_t(\eta_t) 
\end{equation}
for some properly chosen step size $\eta_t$. For the rank-1 case where $r=1$, the update rule~\eqref{eq:geodesic} can be simplified to
\begin{equation}
\bL_{t+1} = \cos(\sigma \eta_t)\bL_t -   \frac{\sin(\sigma \eta_t) }{\| \nabla_{\mathcal{G}} F(\bL_t)\|_2} \nabla_{\mathcal{G}} F(\bL_t),
\end{equation}
with $\sigma := \| \nabla_{\mathcal{G}} F(\bL_t) \|_2$. As can be verified, $\bL_{t+1}$ automatically stays on the unit sphere obeying $\bL_{t+1}^{\top}\bL_{t+1}= 1$.

One of the earliest provable nonconvex methods for matrix completion --- the $\mathsf{OptSpace}$ algorithm by Keshavan et al.~\cite{keshavan2010matrix,keshavan2010noisy} --- performs  gradient descent on the Grassmann manifold, tailored to the loss function:
$$ F(\bL, \bR ) := \min_{\bS \in \mathbb{R}^{r \times r}}\left\|\mathcal{P}_{\Omega}(\bM_{\star} - \bL\bS \bR^{\top}) \right\|_{\mathrm{F}}^2,$$
where $\bL\in \mathcal{G}(n_1, r)$ and $\bR \in \mathcal{G}(n_2, r)$, with some additional regularization terms to promote incoherence (see Section \ref{sec:Reg-GD}). It is shown by \cite{keshavan2010matrix,keshavan2010noisy} that GD on the Grassman manifold converges to the truth with high probability if $n^2p\gtrsim \mu^2\kappa^6 r^2 n\log n$, provided that a proper initialization is given.

Other gradient descent approaches on manifolds include \cite{dai2012geometric,wei2016guarantees,balzano2014subspace,wei2016guarantees_mc,vandereycken2013low,uschmajew2018critical,cai2018solving}. See
\cite{cai2018exploiting} for an extensive overview of recent developments along this line.

\subsection{Stochastic gradient descent}

Many problems have to deal with an empirical loss function that is an average of the sample losses, namely, 
\begin{equation}
f(\bx)  = \frac{1}{m}\sum_{i=1}^m f_i(\bx).
\end{equation}
When the sample size is large, it is computationally expensive to apply the gradient update rule --- which goes through all data samples --- in every iteration. Instead, one might apply stochastic gradient descent (SGD) \cite{robbins1985stochastic,nemirovski2009robust,bottou2018optimization}, where in each iteration, only a single sample or a small subset of samples are used to form the search direction. Specifically, the SGD follows the update rule
\begin{equation}\label{eq:SGD}
\bx_{t+1}  = \bx_t - \frac{\eta_t}{m}\sum_{i \in\Omega_t} \nabla f_i(\bx),
\end{equation}
where $\Omega_t\in\{1,\ldots,m\}$ is a subset of cardinality $k$ selected uniformly at random. Here, $k\geq 1$ is known as the mini-batch size. As one can expect, the mini-batch size $k$ plays an important role in the trade-off between the computational cost per iteration and the convergence rate. A properly chosen mini-batch size will optimize the total computational cost given the practical constraints. Please see \cite{wei2015solving,li2015phase,chi2016kaczmarz,zhang2017reshaped,tan2017phase,monardo2019solving} for the application of SGD in phase retrieval (which has an interesting connection with the Kaczmarz method), and \cite{jin2016provable} for its application in matrix factorization.

\section{Beyond gradient methods}
\label{sec:other_algorithms}

Gradient descent is certainly not the only method that can be employed to solve the problem~\eqref{eq:core_problem}. Indeed, many other algorithms have been proposed, which come with different levels of theoretical guarantees. Due to the space limitation, this section only reviews two popular alternatives to gradient methods discussed so far. For simplicity, we  consider the following unconstrained problem (with slight abuse of notation)
\begin{equation}
	\label{eq:empirical-risk-min-altmin}
	\underset{{\bm{L}\in \mathbb{R}^{n_1\times r},\,\bm{R}\in \mathbb{R}^{n_2\times r}}}{\text{minimize}} \quad f(\bm{L}, \bm{R}) .
\end{equation}

\subsection{Alternating minimization}

To optimize the core problem \eqref{eq:empirical-risk-min-altmin}, alternating minimization (AltMin) alternates between solving the following two subproblems: for $t=1,2\ldots,$
\begin{subequations}
	\label{eq:altmin_general}
\begin{align}
\bR_{t}  & = \argmin_{\bR\in\mathbb{R}^{n_2\times r}} \; f(\bL_{t-1}, \bR), \label{eq:altmin_general_L}\\
\bL_{t} &  = \argmin_{\bL\in\mathbb{R}^{n_1\times r}} \; f(\bL, \bR_{t}), \label{eq:altmin_general_R}
\end{align}
\end{subequations}
where $\bR_t$ and $\bL_t$ are updated {\em sequentially}. Here, $\bL_0$ is an appropriate initialization. For many problems discussed here, both \eqref{eq:altmin_general_L} and \eqref{eq:altmin_general_R} are convex problems and can be solved efficiently.

\subsubsection{Matrix sensing} Consider the loss function \eqref{eq:min-matrix-sensing-rank-r-asymmetric}. In each iteration, AltMin proceeds as follows \cite{jain2013low}: for $t=1,2\ldots,$ 
\begin{align*}
\bR_{t}  & = \argmin_{\bR\in\mathbb{R}^{n_2\times r}} \; \left\| \mathcal{A}(\bL_{t-1}\bR^{\top} - \bM_{\star})  \right\|_{\mathrm{F}}^2, \\
\bL_{t} &  = \argmin_{\bL\in\mathbb{R}^{n_1\times r}} \; \left\| \mathcal{A}(\bL\bR_t^{\top} - \bM_{\star})  \right\|_{\mathrm{F}}^2.
\end{align*}
Each substep consists of a linear least-squares problem, which can often be solved efficiently via the conjugate gradient algorithm \cite{trefethen1997numerical}. To illustrate why this forms a promising scheme, we look at the following simple example. 

\begin{example}
Consider the case where $\mathcal{A}$ is identity (i.e.~$\mathcal{A}(\bm{M})=\bm{M}$).  We claim that given almost any initialization, AltMin converges to the truth after two updates. To see this, we first note that the output of the first iteration can be written as
$$ \bR_1 = \bR_{\star} \bL_{\star}^{\top}  \bL_{0}(\bL_{0}^{\top}\bL_0)^{-1}. $$
As long as both $\bL_{\star}^{\top}  \bL_{0}$ and $\bL_{0}^{\top}\bL_0$ are full-rank, the column space of $\bR_{\star}$ matches perfectly with that of $\bR_{1}$. Armed with this fact, the subsequent least squares problem (i.e.~the update for $\bm{L}_1$) is exact, in the sense that  $\bm{L}_1\bm{R}_1^{\top} = \bm{M}_{\star}= \bm{L}_{\star}\bm{R}_{\star}^{\top}$. 

\end{example}

With the above identity example in mind, we are  hopeful that AltMin converges fast if $\mathcal{A}$ is nearly isometric. Towards this, one has the following theory. 
\begin{theorem}[$\mathsf{AltMin~for~matrix~sensing}$ \cite{jain2013low}]
\label{thm:AltMin-matrix-sensing}
Consider the problem \eqref{eq:min-matrix-sensing-rank-r-asymmetric} and suppose the operator \eqref{eq:defn-A-sensing} satisfies $2r$-RIP with RIP constant $\delta_{2r}\leq 1 /(100\kappa^2 r)$. If we initialize $\bL_0$ by the  $r$ leading left singular vectors of $\mathcal{A}^*(\by)$, then AltMin achieves
	\[ \left\| \bM_{\star} - \bL_t \bR_t^{\top} \right\|_{\mathrm{F}} \leq \varepsilon	\]
for all $t\geq 2\log(\|\bM_{\star}\|_{\mathrm{F}}/\varepsilon)$.
\end{theorem}
In comparison to the performance of GD in Theorem~\ref{thm:convergence-rankr-sensing}, AltMin enjoys a better iteration complexity w.r.t.~the condition number $\kappa$;  that is, it obtains $\varepsilon$-accuracy within $O(\log(1/\varepsilon))$ iterations, compared to $O(\kappa\log(1/\varepsilon))$ iterations for GD. In addition, the requirement on the RIP constant depends quadratically on $\kappa$, leading to a sub-optimal sample complexity. To address this issue, Jain et al.~\cite{jain2013low} further developed a stage-wise AltMin algorithm, which only requires $\delta_{2r}=O (1/r^2)$.
Intuitively, if there is a singular value that is much larger than the remaining ones, then one can treat $\bm{M}_{\star}$ as a (noisy) rank-1 matrix and compute this rank-1 component via AltMin. Following this strategy, one successively applies AltMin to recover the dominant rank-1 component in the residual matrix,  unless it is already well-conditioned.  See \cite{jain2013low} for details.


\subsubsection{Phase retrieval}
Consider the phase retrieval problem. It is helpful to think of the amplitude measurements as bilinear measurements of the signs $\boldsymbol{b} =\{b_i\in\{\pm 1\}\}_{1\leq i\leq m}$ and the signal $\bx_{\star}$, namely, 
\begin{equation}
	\sqrt{y_i} = |\bm{a}_{i}^{\top}\bm{x}_{\star}| = \underset{:=b_i }{\underbrace{\mathsf{sgn}(\bm{a}_{i}^{\top}\bm{x}_{\star})}} \, \bm{a}_{i}^{\top}\bm{x}_{\star}.
\end{equation} 
This leads to a simple yet useful alternative formulation for the amplitude loss minimization problem 
\begin{align*}
	\underset{\bm{x}\in \mathbb{R}^n}{\text{minimize}}~	f_{\mathrm{amp}}(\bx) = \underset{\bm{x}\in \mathbb{R}^n, b_i\in\{\pm 1\}}{\text{minimize}} ~f(\boldsymbol{b}, \bx),
\end{align*}
where we abuse the notation by letting
\begin{equation}
	f(\boldsymbol{b}, \bx)  := \frac{1}{2m}\sum_{i=1}^{m}\big(b_i \bm{a}_{i}^{\top}\bm{x} -\sqrt{y_i}\big)^{2} . 
\end{equation}
Therefore, by applying AltMin to the loss function $ f(\boldsymbol{b}, \bx)$, we obtain the following update rule \cite{netrapalli2015phase,waldspurger2016phase}: for each $t=1,2,\ldots$
\begin{subequations}
\begin{align}
	\bb_{t} & = \argmin_{b_i:|b_i|=1, \forall i} f(\boldsymbol{b}, \bx_{t-1} ) = \mathsf{sgn}(\bA\bx_{t-1}),\\
	\bx_{t+1} &= \argmin_{\bx \in \mathbb{R}^n} f(\boldsymbol{b}_t, \bx_{t-1} ) = \bA^{\dag}\mathsf{diag}(\bb_{t}) \sqrt{\by} ,
	\label{eq:alt-min-PR-x}
\end{align}
\end{subequations}
where $\bx_0$ is an appropriate initial estimate, $\bm{A}^{\dagger}$ is the pseudo-inverse of $\bA:=[\bm{a}_1,\cdots, \bm{a}_m]^{\top}$, and $\sqrt{\by}:=[\sqrt{y_i}]_{1\leq i\leq m}$. The step \eqref{eq:alt-min-PR-x} can again be efficiently solved using the conjugate gradient method \cite{trefethen1997numerical}.  
This is exactly the Error Reduction (ER) algorithm proposed by Gerchberg and Saxton \cite{gerchberg1972practical,fienup1982phase} in the 1970s. Given a reasonably good initialization,  this algorithm converges linearly under the Gaussian design.
\begin{theorem}[$\mathsf{AltMin~(ER)~for~phase~retrieval}$ \cite{waldspurger2016phase}]
\label{thm:AltMin-PR}
Consider the problem \eqref{eq:min-PR}.
There exist some constants  $c_0, \cdots, c_4>0$ and $0<\rho<1$ such that if 
	$m\geq c_{0}n$, then with probability at least $1-c_2\exp(-c_3 m)$, the estimates of AltMin (ER) satisfy
\begin{align}
	& \|\bm{x}_{t}-\bm{x}_{\star}\|_{2} \leq \rho^{t}\|\bm{x}_{0}-\bm{x}_{\star}\|_{2},\qquad t=0,1,\cdots  
	\label{eq:medianPR-tight}  
\end{align}
as long as $\|\bm{x}_{0}-\bm{x}_{\star}\|_{2}\leq  c_4 \|\bm{x}_{\star}\|_{2}$.
\end{theorem}
\begin{remark}
	AltMin for phase retrieval was first analyzed by \cite{netrapalli2015phase} but for a sample-splitting variant; that is, each iteration employs fresh samples, which facilitates analysis but is not the version used in practice. The theoretical guarantee for the original sample-reuse version was derived by \cite{waldspurger2016phase}. 
\end{remark}
In view of Theorem \ref{thm:AltMin-PR}, alternating minimization, if carefully initialized, achieves optimal sample and computational complexities (up to some logarithmic factor) all at once.    This in turn explains its appealing performance in practice.

\subsubsection{Matrix completion} 
Consider the matrix completion problem in \eqref{eq:MC-empirical-risk-asym}. Starting with a proper initialization $(\bL_0,\bR_0)$, AltMin proceeds as follows: for $t=1,2,\ldots$
\begin{subequations} 
	\label{eq:original_altmin_mc}
\begin{align}
\bR_{t}  & = \argmin_{\bR\in\mathbb{R}^{n_2\times r}} \; \left\| \mathcal{P}_{\Omega}(\bL_{t-1}\bR^{\top} - \bM_{\star})  \right\|_{\mathrm{F}}^2, \\
\bL_{t} &  = \argmin_{\bL\in\mathbb{R}^{n_1\times r} } \; \left\| \mathcal{P}_{\Omega}(\bL\bR_t^{\top} - \bM_{\star})  \right\|_{\mathrm{F}}^2,
\end{align}
\end{subequations}
where $\mathcal{P}_{\Omega}$ is defined in \eqref{defn:Pomega}. 
Despite its popularity in practice \cite{hastie2015matrix}, a clean analysis of the above update rule is still missing to date. 
Several modifications have been proposed and analyzed in the literature, primarily to bypass mathematical difficulty:
\begin{itemize}
\item {\bf Sample splitting.} Instead of reusing the same set of samples across all iterations, this approach draws a fresh set of samples at every iteration and performs AltMin on the new samples
\cite{keshavan2012efficient,jain2013low,hardt2014understanding,hardt2014fast,zhao2015nonconvex}:
\begin{align*}
\bR_{t}  & = \argmin_{\bR\in\mathbb{R}^{n_2\times r} } \; \left\| \mathcal{P}_{\Omega_t}(\bL_{t-1}\bR^{\top} - \bM_{\star})  \right\|_{\mathrm{F}}^2, \\
\bL_{t} &  = \argmin_{\bL\in\mathbb{R}^{n_1\times r}} \; \left\| \mathcal{P}_{\Omega_t}(\bL\bR_t^{\top} - \bM_{\star})  \right\|_{\mathrm{F}}^2,
\end{align*}
where $\Omega_t$ denotes the sampling set used in the $t$th iteration, which is assumed to be statistically independent across iterations. It is proven in \cite{jain2013low} that under an appropriate initialization, the output satisfies $ \left\| \bM_{\star} - \bL_t \bR_t^{\top} \right\|_{\mathrm{F}} \leq \varepsilon$ after $t\gtrsim \log(\|\bM_{\star}\|_{\mathrm{F}}/\varepsilon)$ iterations,  provided that the sample complexity exceeds $n^2 p \gtrsim \mu^4 \kappa^6 n r^7 \log n\, \log(r\|\bM_{\star}\|_{\mathrm{F}}/\varepsilon)$. Such a sample-splitting operation ensures statistical independence across iterations, which helps  to control the incoherence of the iterates.  However, this necessarily results in undesirable  dependency between the sample complexity and the target accuracy; for example, an infinite number of samples is needed if the goal is to achieve exact recovery. 

\item {\bf Regularization.} Another strategy is to apply AltMin to the regularized loss function in \eqref{eq:regularized_mc_loss} \cite{sun2016guaranteed}:
\begin{subequations}
\begin{align}
	\bm{R}_{t} & =\argmin_{\bm{R}\in\mathbb{R}^{n_{2}\times r}} f_{\mathrm{reg}}(\bm{L}_{t-1},\bm{R}), \label{eq:alt-min-reg-R} \\
	\bm{L}_{t} & =\argmin_{\bm{L}\in\mathbb{R}^{n_{2}\times r}} f_{\mathrm{reg}}(\bm{L},\bm{R}_{t})	\label{eq:alt-min-reg-L}.
\end{align}
\end{subequations}
In \cite{sun2016guaranteed}, it is shown that the AltMin without resampling converges to $\bm{M}_{\star}$,  with the proviso that the sample complexity exceeds $n^2 p \gtrsim \mu^2 \kappa^6 nr^7 \log n$. Note that the subproblems \eqref{eq:alt-min-reg-R} and \eqref{eq:alt-min-reg-L} do not have closed-form solutions.  For properly chosen regularization functions, they might be solved using convex optimization algorithms. 
\end{itemize}

It remains an open problem to establish theoretical guarantees for the original form of AltMin \eqref{eq:original_altmin_mc}. Meanwhile, the existing sample complexity guarantees are quite sub-optimal in terms of the dependency on $r$ and $\kappa$,  and should not be taken as an indicator of the actual performance of AltMin.

\subsection{Singular value projection}
\label{sec:SVP}

Another popular approach to solve \eqref{eq:empirical-risk-min-altmin} is singular value projection (SVP) \cite{jain2010guaranteed,oymak2018sharp,ding2018leave}. In contrast to the algorithms discussed so far,  SVP performs gradient descent in the {\em full matrix space} and then applies a partial singular value decomposition (SVD) to retain the low-rank structure. Specifically, it adopts the update rule
	\begin{align}
		\label{eq:svp}
		\bM_{t+1}  &= \mathcal{P}_r \Big( \bM_t - \eta_t \nabla f( \bM_t ) \Big) , \qquad t=0,1,\cdots
	\end{align}
where 
\begin{equation*}
f(\bm{M}):=\begin{cases}
	\frac{1}{2} \| \mathcal{A} (\bm{M}) - \mathcal{A} (\bm{M}_{\star}) \|_{\mathrm{F}}^2 , & \text{matrix sensing},\\
\frac{1}{2p}\|\mathcal{P}_{\Omega}(\bm{M})-\mathcal{P}_{\Omega}(\bm{M}_{\star})\|_{\mathrm{F}}^{2}, & \text{matrix completion.}
\end{cases}
\end{equation*}
Here, 
$\eta_t$ is the step size, and $\mathcal{P}_r(\bZ)$ returns the best rank-$r$ approximation of $\bZ$. Given that the iterates $\bM_t$ are always low-rank, one can store $\bM_t$ in a memory-efficient manner by storing its compact SVD. 

The SVP algorithm is a popular approach for matrix sensing and matrix completion, where the partial SVD can be calculated using Krylov subspace methods (e.g.~Lanczos algorithm) \cite{trefethen1997numerical} or the randomized linear algebra algorithms \cite{halko2011finding}. The following theorem establishes performance guarantees for SVP.
\begin{theorem}[$\mathsf{SVP~for~matrix~sensing}$ \cite{jain2010guaranteed}]
\label{thm:SVP-matrix-sensing}
	Consider the problem \eqref{eq:empirical-risk-min-altmin} and suppose the operator \eqref{eq:defn-A-sensing} satisfies $2r$-RIP for  RIP constant $\delta_{2r}\leq 1/3$. If we initialize $\bM_0=\bm{0}$ and adopt a step size $\eta_t\equiv 1 / (1+\delta_{2r})$, then SVP achieves
	\[ 
		\left\| \bM_t - \bM_{\star}  \right\|_{\mathrm{F}} \leq \varepsilon	
	\]
as long as $t\geq c_1\log(\|\bM_{\star}\|_{\mathrm{F}}/\varepsilon)$ for some constant $c_1>0$.
\end{theorem}

\begin{theorem}[$\mathsf{SVP~for~matrix~completion}$ \cite{ding2018leave}]
\label{thm:SVP-matrix-completion}
Consider the problem \eqref{eq:MC-empirical-risk} and set $\bM_0=\boldsymbol{0}$.  
	Suppose that $n^{2}p\geq c_0\kappa^{6}\mu^4 r^6 n\log n$ for some large
constant $c_0>0$, and that the step size is set as $\eta_t \equiv 1$. Then with probability exceeding $1-O\left(n^{-10}\right)$, the SVP iterates achieve
	\[ 
		\left\| \bM_t - \bM_{\star}  \right\|_{\infty} \leq \varepsilon	
	\]
	as long as $t\geq c_1\log(\|\bM_{\star}\|/\varepsilon)$ for some constant $c_1>0$.
\end{theorem}
Theorem~\ref{thm:SVP-matrix-sensing} and Theorem~\ref{thm:SVP-matrix-completion} indicate that SVP converges linearly as soon as the sample size is sufficiently large. 

Further, the SVP operation is particularly helpful in enabling optimal uncertainty quantification and  inference for noisy matrix completion (e.g.~constructing a valid and short confidence interval for an entry of the unknown matrix).  The interested readers are referred to \cite{chen2019inference} for details.

\subsection{Further pointers to other algorithms}
A few other nonconvex matrix factorization algorithms have been left out due to space, including but not limited to normalized iterative hard thresholding (NIHT) \cite{tanner2013normalized},  atomic decomposition for minimum rank approximation (Admira)
\cite{lee2010admira}, composite optimization (e.g.~prox-linear algorithm) \cite{duchi2017solving,duchi2017stochastic,charisopoulos2019lowrank,charisopoulos2019composite}, approximate message passing \cite{donoho2013phase,ma2018optimization,ma2018approximate}, block coordinate descent \cite{sun2016guaranteed},  coordinate descent \cite{zeng2017coordinate}, and conjugate gradient \cite{pinilla2018phase}. The readers are referred to these papers for detailed descriptions.


\section{Initialization via spectral methods}
\label{sec:spectral-initialization}

The theoretical performance guarantees presented in the last three sections rely heavily on proper initialization. One popular  scheme that often generates a reasonably good initial estimate  is called the spectral method. Informally, this strategy starts by arranging the data samples into a matrix $\bm{Y}$ of the form 
\begin{equation}
	\label{eq:perturbed}
	\bm{Y} = \bm{Y}_{\star} + \bm{\Delta},
\end{equation}
where $\bm{Y}_{\star}$ represents certain large-sample limit  whose eigenspace\,/\,singular subspaces reveal the truth, and $\bm{\Delta}$ captures the fluctuation due to the finite-sample effect. One  then attempts to estimate the truth by computing the eigenspace\,/\,singular subspace of $\bm{Y}$, provided that the finite-sample fluctuation is well-controlled. 
This simple strategy has  proven to be quite powerful and versatile in providing a ``warm start'' for many nonconvex matrix factorization algorithms. 

\subsection{Preliminaries: matrix perturbation theory}
\label{sec:perturbation}

Understanding the performance of the spectral method requires some elementary toolkits regarding eigenspace\,/\,singular subspace perturbations, which we review in this subsection.  


To begin with, let $\bm{Y}_{\star}\in \mathbb{R}^{n\times n}$ be a symmetric matrix, whose eigenvalues are real-valued.  In many cases, we only have access to a perturbed version  $\bm{Y}$ of $\bm{Y}_{\star}$ (cf.~\eqref{eq:perturbed}), where the perturbation $\bm{\Delta}$ is a ``small'' symmetric matrix. How do the eigenvectors of $\bm{Y}$ change as a result of such a perturbation? 



As it turns out, the eigenspace of $\bm{Y}$ is a stable estimate of the eigenspace of $\bm{Y}_{\star}$, with the proviso that the perturbation is sufficiently small in size. This was first established in the celebrated Davis-Kahan $\sin\bm{\Theta}$ Theorem \cite{davis1970rotation}. Specifically,  
let the eigenvalues of $\bm{Y}_{\star}$ be partitioned into two groups
\[
\underbrace{\lambda_1(\bm{Y}_{\star}) \ge  \ldots \lambda_r(\bm{Y}_{\star})}_{\text{Group 1}} > \underbrace{\lambda_{r+1}(\bm{Y}_{\star}) \ge \ldots \lambda_n(\bm{Y}_{\star})}_{\text{Group 2}},
\]
where $1 \le r < n$. We assume that the eigen-gap between the two groups, $\lambda_r(\bm{Y}_{\star}) - \lambda_{r+1}(\bm{Y}_{\star})$, is strictly positive. For example, if $\bm{Y}_{\star}\succeq \bm{0}$ and has rank $r$, then all the eigenvalues in the second group are identically zero. 

Suppose we wish to estimate the eigenspace associated with the first group.  Denote by $\bm{U}_{\star} \in \mathbb{R}^{n\times r}$ (resp.~$\bm{U} \in \mathbb{R}^{n\times r}$) an orthonormal matrix  whose columns are the first $r$ eigenvectors of $\bm{Y}_{\star}$ (resp.~$\bm{Y}$).  In order to measure  the distance between the two subspaces spanned by $\bm{U}_{\star}$ and $\bm{U}$, we introduce the following metric that accounts for global orthonormal transformation
%
\begin{equation}
	\label{eq:dist_sin}
	\mathsf{dist}_{\mathrm{p}} (\bm{U}, \bm{U}_{\star}) := \norm{\bm{U} \bm{U}^\top - \bm{U}_{\star} \bm{U}_{\star}^\top}.
\end{equation}
%
%
%

\begin{theorem}[$\mathsf{Davis}$-$\mathsf{Kahan}$ $\sin\bm{\Theta}$ $\mathsf{Theorem}$ \cite{davis1970rotation}]
	\label{thm:DK} 
	If $\| \bm{\Delta} \| < \lambda_r(\bm{Y}_{\star}) - \lambda_{r+1}(\bm{Y}_{\star})$,
then
\begin{equation}
	\label{eq:DK}
	\mathsf{dist}_{\mathrm{p}} (\bm{U}, \bm{U}_{\star}) \le \frac{\| \bm{\Delta} \|}{\lambda_r(\bm{Y}_{\star}) - \lambda_{r+1}(\bm{Y}_{\star}) - \| \bm{\Delta} \|}.
\end{equation}
\end{theorem}
\begin{remark}
The bound we present in \eref{DK} is in fact a simplified, but slightly more user-friendly version of the original Davis-Kahan inequality. A more general result states that, if $\lambda_r(\bm{Y}_{\star}) - \lambda_{r+1}(\bm{Y}) > 0$, then
\begin{equation}
	\label{eq:DK_original}
	\mathsf{dist}_{\mathrm{p}} (\bm{U}, \bm{U}_{\star}) \le \frac{\| \bm{\Delta} \|}{\lambda_r(\bm{Y}_{\star}) - \lambda_{r+1}(\bm{Y})}.
\end{equation}
	These results are referred to as the $\sin \bm{\Theta}$ theorem because the distance metric $\mathsf{dist}_{\mathrm{p}} (\bm{U}, \bm{U}_{\star})$ is identical to $\max_{1\leq i\leq r} | \sin \theta_i|$, where $\{\theta_i\}_{1\leq i\leq r}$ are the so-called \emph{principal angles} \cite{Bjorck1973principal} between the two subspaces spanned by $\bm{U}$ and $\bm{U}_{\star}$, respectively.  
\end{remark}

%
%
%

Furthermore, to deal with asymmetric matrices, similar perturbation bounds can be obtained. Suppose that $\bm{Y}_{\star},\bm{Y},\bm{\Delta} \in \R^{n_1 \times n_2}$ in \eqref{eq:perturbed}. Let $\bm{U}_{\star}$ (resp.~$\bm{U}$) and $\bm{V}_{\star}$ (resp.~$\bm{V}$) consist respectively of the first $r$ left and right singular vectors of $\bm{Y}_{\star}$ (resp.~$\bm{Y}$). Then we have the celebrated Wedin $\sin\bm{\Theta}$ Theorem \cite{wedin1972perturbation} concerning perturbed singular subspaces. 

\begin{theorem}[$\mathsf{Wedin}$ $\sin\bm{\Theta}$ $\mathsf{Theorem}$ \cite{wedin1972perturbation}]
\label{thm:Wedin}
If $\| \bm{\Delta} \| < \sigma_r(\bm{Y}_{\star}) - \sigma_{r+1}(\bm{Y}_{\star})$,
then
\begin{align*}
	&\max\big\{ \mathsf{dist}_{\mathrm{p}} (\bm{U}, \bm{U}_{\star}), \, \mathsf{dist}_{\mathrm{p}} (\bm{V}, \bm{V}_{\star}) \big\} \mylinebreak
	\myalign \myquad \myquad \leq \frac{ \| \bm{\Delta} \| }{ \sigma_r(\bm{Y}_{\star}) - \sigma_{r+1}(\bm{Y}_{\star}) - \|\bm{\Delta}\| }.
\end{align*}
\end{theorem}

In addition, one might naturally wonder how the eigenvalues\,/\,singular values are affected by the perturbation.  To this end, Weyl's inequality provides a simple answer: 
\begin{align}
\big|\lambda_{i}(\bm{Y})-\lambda_{i}(\bm{Y}_{\star})\big| & \leq\|\bm{\Delta}\|,\qquad1\leq i\leq n,\\
\big|\sigma_{i}(\bm{Y})-\sigma_{i}(\bm{Y}_{\star})\big| & \leq\|\bm{\Delta}\|,\qquad1\leq i\leq \min\{n_1,n_2\}.
\end{align}

In summary, both eigenspace (resp.~singular subspace) perturbation
and eigenvalue (resp.~singular value) perturbation rely heavily on
the spectral norm of the perturbation $\bm{\Delta}$. 


\subsection{Spectral methods}
\label{sec:vanilla-spectral}

With the matrix perturbation theory in place, we are positioned to present spectral methods for various low-rank matrix factorization problems. As we shall see, these methods are all variations on a common recipe.

\subsubsection{Matrix sensing}

We start with the prototypical problem of matrix sensing \eqref{eq:matrix-sensing-samples} as described in Section~\ref{sec:matrix_sensing}. Let us construct a \emph{surrogate matrix} as follows
\begin{equation}
	\label{eq:surrogate}
	\bm{Y} = \frac{1}{m}\sum_{i=1}^m y_i \bm{A}_i = \mathcal{A}^* \mathcal{A} (\bm{M}_{\star}),
\end{equation}
where $\mathcal{A}$ is the linear operator defined in \eref{defn-A-sensing} and $\mathcal{A}^\ast$ denotes its adjoint. The generic version of the spectral method then proceeds by computing (i) two matrices $\bm{U}\in \mathbb{R}^{n_1\times r}$ and $\bm{V}\in \mathbb{R}^{n_2\times r}$ whose columns consist of the top-$r$ left and right singular vectors of $\bm{Y}$, respectively, and (ii) a diagonal matrix $\bm{\Sigma}\in \mathbb{R}^{r\times r}$ that contains the corresponding top-$r$ singular values. 
In the hope that $\bm{U}$, $\bm{V}$ and $\bm{\Sigma}$ are reasonably reliable estimates of $\bm{U}_{\star}$, $\bm{V}_{\star}$ and $\bm{\Sigma}_{\star}$, respectively, we take
\begin{align} \label{eq:initial_guess}
	\bm{L}_0 = \bm{U}\bm{\Sigma}^{1/2} \qquad \text{and} \qquad \bm{R}_0 = \bm{V}\bm{\Sigma}^{1/2}
\end{align}
as estimates of the low-rank factors $\bm{L}_{\star}$ and $\bm{R}_{\star}$ in \eqref{eq:ground_truth_factors}.  If $\bm{M}_{\star}\in \mathbb{R}^{n\times n}$ is known to be positive semidefinite, then we can also let $\bm{U}\in \mathbb{R}^{n\times r}$ be a matrix consisting of the top-$r$ leading eigenvectors, with $\bm{\Sigma}$ being a diagonal matrix containing all top-$r$ eigenvalues.

Why would this be a good strategy?  In view of Section \ref{sec:perturbation}, the three matrices $\bm{U}$, $\bm{V}$ and $\bm{\Sigma}$ become reliable estimates if $\|\bm{Y}-\bm{M}_{\star}\|$ can be well-controlled.    
A simple way to control $\|\bm{Y}-\bm{M}_{\star}\|$ arises when  $\mathcal{A}$ satisfies the RIP in Definition~\ref{defn:RIPs}.  
\begin{lemma}
	\label{lem:perturbation-sensing}
	Suppose that $\bm{M}_{\star}$ is a rank-$r$ matrix, and assume that $\mathcal{A}$ satisfies $2r$-RIP with RIP constant $\delta_{2r} < 1$. Then
\begin{equation}
	\label{eq:msense_pert}
	\|\bm{Y}-\bm{M}_{\star}\|  \leq \delta_{2r} \| \bm{M}_{\star} \|_{\mathrm{F}} \leq \delta_{2r} \sqrt{r} \| \bm{M}_{\star} \| .
\end{equation}
\end{lemma}

\begin{proof}[Proof of Lemma~\ref{lem:perturbation-sensing}]
Let $\vx, \vy$ be two arbitrary vectors, then
\[
\begin{aligned}
	 \vx^\top (\mMp- \bm{Y})\vy 
	&=   \vx^\top \left(\mMp - \mathcal{A}^* \mathcal{A}(\mMp) \right) \vy \\
	&=   \inprod{\vx \vy^\top, \mMp - \mathcal{A}^*\mathcal{A}(\mMp)} \\
	&=   \inprod{\vx \vy^\top, \mMp} - \inprod{\mathcal{A}(\vx \vy^{\top}), \mathcal{A}(\mMp)} \\
	&\le \delta_{2r} ( \|\vx\|_2 \cdot \| \vy \|_2 ) \norm{\mMp}_{\mathrm{F}},
\end{aligned}
\]
where the last inequality is due to Lemma~\ref{lemmq:RIP-cross}. Using a variational characterization of $\|\cdot\|$, we have
\begin{align}
	\norm{\mMp - \bm{Y}} &= \max_{\| \vx \|_2 = \|\vy\|_2 = 1} \vx^{\top} (\boldsymbol{M}_{\star} - \bm{Y}) \vy \le \delta_{2r} \norm{\mMp}_{\mathrm{F}}  \leq  \delta_{2r}  \sqrt{r} \norm{\mMp} ,\nonumber
\end{align}
from which \eqref{eq:msense_pert} follows.
 \end{proof}

In what follows, we first illustrate how to control the estimation error in the {\em rank-1} case where $\bm{M}_{\star}=\lambda_{\star}\bm{u}_{\star}\bm{u}_{\star}^{\top}\succeq \bm{0}$. In this case, the leading eigenvector $\bm{u}$ of $\bm{Y}$ obeys 
\begin{align}
	\mathsf{dist}_{\mathrm{p}}(\bm{u},\bm{u}_{\star}) & \overset{\mathrm{(i)}}{\leq}\frac{\|\bm{Y}-\bm{M}_{\star}\|}{\sigma_{1}(\bm{M}_{\star})-\|\bm{Y}-\bm{M}_{\star}\|} \overset{\mathrm{(ii)}}{\leq} \frac{2\|\bm{Y}-\bm{M}_{\star}\|}{\sigma_{1}(\bm{M}_{\star})} \leq 2\delta_2,  \label{eq:rank1-spectral-example-u}
\end{align}
where (i) comes from Theorem \ref{thm:DK}, and (ii) holds if $\|\bm{Y}-\bm{M}_{\star}\| \leq \sigma_{1}(\bm{M}_{\star})/2$ (which is guaranteed if $\delta_{2}\le 1/2$ according to Lemma \ref{lem:perturbation-sensing}). 
Similarly, we can invoke Weyl's inequality and Lemma \ref{lem:perturbation-sensing} to control the gap between the leading eigenvalue $\lambda$ of $\bm{Y}$ and   $\lambda_{\star}$:
\begin{equation}
	\label{eq:rank1-spectral-example-lambda}
	| \lambda - \lambda_{\star} | \leq \| \bm{Y} - \bm{M}_{\star} \| \leq \delta_2 \|\bm{M}_{\star}\|.
\end{equation}
Combining the preceding two bounds, we see that: if $\bm{u}^{\top} \bm{u}_{\star} \geq 0$, then 
\begin{align*}
\big\|\sqrt{\lambda}\bm{u}-\sqrt{\lambda_{\star}}\bm{u}_{\star}\big\|_{2} & \leq\big\|\sqrt{\lambda}(\bm{u}-\bm{u}_{\star})\big\|_{2}+\big\|(\sqrt{\lambda}-\sqrt{\lambda_{\star}})\bm{u}_{\star}\big\|_{2}\\
 & =\sqrt{\lambda}\mathsf{dist}(\bm{u},\bm{u}_{\star})+\frac{|\lambda-\lambda_{\star}|}{\sqrt{\lambda}+\sqrt{\lambda_{\star}}}\\
 & \lesssim\sqrt{\lambda_{\star}}\delta_{2},
\end{align*}
where the last inequality makes use of \eqref{eq:rank1-spectral-example-u}, \eqref{eq:rank1-spectral-example-lambda}, and \eqref{eq:dist-equivalent}. This characterizes the difference between our estimate $\sqrt{\lambda}\bm{u}$ and the true low-rank factor $\sqrt{\lambda_{\star}}\bm{u}_{\star}$. 

 
Moving beyond this simple rank-1 case, a more general (and often tighter) bound can be obtained by using a refined argument from \cite[Lemma 5.14]{tu2015low}. We present the theory below. The proof can be found in Appendix~\ref{appendix:proof-lem:perturbation-sensing}. 

\begin{theorem}[$\mathsf{Spectral~method~for~matrix~sensing}$ \cite{tu2015low}]
	\label{thm:spectral-sensing}
	Fix $\zeta > 0$. Suppose $\mathcal{A}$ satisfies $2r$-RIP with RIP constant $\delta_{2r} < c_0 \sqrt{\zeta} / (\sqrt{r} \kappa)$ for some sufficiently small constant $c_0>0$. Then the spectral estimate \eqref{eq:initial_guess} obeys
	\[
		\mathsf{dist}^2\left({\footnotesize\begin{bmatrix}
\bm{L}_{0} \\
\bm{R}_{0}
	\end{bmatrix}} , {\footnotesize\begin{bmatrix}
\bm{L}_{\star} \\
\bm{R}_{\star}
		\end{bmatrix} }\right)  \leq  \zeta \sigma_{r}(\bm{M}_{\star}) .
	\]		
\end{theorem}
\begin{remark}
	It is worth pointing out that in view of Fact~\ref{fact:gaussian_rip}, the vanilla spectral method needs $O(nr^2)$ samples to land in the local basin of attraction (in which linear convergence of GD is guaranteed according to  Theorem~\ref{thm:convergence-rankr-sensing}).
\end{remark}
 
As discussed in Section \ref{sec:gd}, the RIP  does not hold for the sensing matrices used in many problems. 
Nevertheless, one may still be able to show that the leading singular subspace of the surrogate matrix $\bm{Y}$  contains useful information about the truth $\mMp$. In the sequel, we will go over several examples to demonstrate this point.

\subsubsection{Phase retrieval}
\label{sec:PR_init}

Recall that the phase retrieval problem in \sref{PR} can be viewed as a matrix sensing problem, where we seek to recover a rank-1 matrix $\mMp = \vxp \vxp^\top$ with sensing matrices $\mA_i = \va_i \va_i^\top$. 
To obtain an initial guess $\bm{x}_0$ that is close to the truth $\bm{x}_{\star}$, we follow the recipe described in \eref{surrogate} by estimating the leading eigenvector $\bm{u}$ and leading eigenvalue $\lambda$ of  a surrogate matrix
\begin{equation}
	\label{eq:D_PR_org}
	\bm{Y} = \frac{1}{m} \sum_{i=1}^m y_i \va_i \va_i^\top.
\end{equation} 
The initial guess is then formed as\footnote{In the sample-limited regime with $m \asymp n$, one should replace $\sqrt{\lambda/3}$ in \eref{init-PR} by $\sqrt{{\sum_{i=1}^m y_i}/{m}}$. The latter provides a more accurate estimate. See the discussions in Section~\ref{sec:improved-spectral-truncated}.}  
\begin{align}
	\label{eq:init-PR}
	\bm{x}_0 = \sqrt{\lambda/3} \, \bm{u}.  
\end{align}

Unfortunately, the RIP does not hold for the sensing operator in phase retrieval, which precludes us from invoking Theorem \ref{thm:spectral-sensing}. There is, however, a simple and intuitive explanation regarding why  $\bm{x}_0$  is a reasonably good estimate of $\vx_{\star}$. Under the Gaussian design, the surrogate matrix $\bm{Y}$ in \eqref{eq:D_PR_org} can be viewed as the sample average of $m$ i.i.d.~random rank-one matrices $\set{y_i \va_i \va_i^{\top}}_{1\le i \le m}$. When the number of samples $m$ is large, this sample average should be ``close'' to its expectation, which is,
\begin{equation}\label{eq:mtx_LLN}
	\mathbb{E}[\bm{Y}] = \mathbb{E} \left[y_i \, \va_i \va_i^\top\right] = 2 \vxp \vxp^\top + \|\vxp\|_2^2\, \mI_n.
\end{equation}
The best rank-1 approximation of $\mathbb{E}[\bm{Y}]$ is precisely $3 \bm{x}_{\star}\bm{x}_{\star}^{\top} $. 
Now that  $\bm{Y}$ is an approximated version of $\mathbb{E}[\bm{Y}]$, we expect $\bm{x}_0$ in \eqref{eq:init-PR} to carry useful information about $\vxp$. 

The above intuitive arguments can be made precise. Applying standard matrix concentration inequalities \cite{vershynin2010nonasym} to the surrogate matrix in \eref{D_PR_org} and invoking the Davis-Kahan $\sin\bm{\Theta}$ theorem, one arrives at the following estimates:
\begin{theorem}[$\mathsf{Spectral~method~for~phase~retrieval}$ \cite{candes2015phase}]
	\label{thm:spectral_PR}
	Consider phase retrieval in \sref{PR}, where $\vxp \in \R^n$ is any given vector.  Fix any $\zeta > 0$, and suppose $m \ge c_0 n \log n$ for some sufficiently large constant $c_0>0$.  Then the spectral estimate \eqref{eq:init-PR} obeys
\[
	\min\set{ \| \vxp - \vx_0 \|_2,  \| \vxp + \vx_0 \|_2 } \le \zeta \| \vxp \|_2
\]
with probability at least $1 - O(n^{-2})$.
\end{theorem}

\subsubsection{Quadratic sensing}
An argument similar to phase retrieval can be applied to quadratic sensing in~\eqref{eq:min-PR_lowrank}, recognizing that the expectation of the surrogate matrix $\bY$ in \eqref{eq:D_PR_org} now becomes
\begin{equation}
\mathbb{E}[\bY] = 2 \bX_{\star}\bX_{\star}^{\top} + \|\bX_{\star}\|_{\mathrm{F}}^2\, \bm{I}_n.
\end{equation}
The spectral method then proceeds by computing $\bm{U}$ (which consists of the top-$r$ eigenvectors of $\bm{Y}$), and a diagonal matrix $\bm{\Sigma}$ whose $i$th diagonal value is given as $(\lambda_i(\bm{Y})-\sigma)/2$, where $\sigma = \frac{1}{m} \sum_{i=1}^{m} y_{i}$ serves as an estimate of $\|\bX_{\star}\|_{\mathrm{F}}^2$. In words, the diagonal entries of $\bm{\Sigma}$ can be approximately viewed as the top-$r$ eigenvalues of $\frac{1}{2} \big( \bm{Y}-\|\bX_{\star}\|_{\mathrm{F}}^2\, \bm{I}_n \big)$.   The initial guess is then set as
\begin{align} 
	\label{eq:initial_guess_qs}
	\bm{X}_0 = \bm{U}\bm{\Sigma}^{1/2}
\end{align}
for estimating the low-rank factor $\bm{X}_{\star}$. The theory is as follows.

\begin{theorem}[$\mathsf{Spectral~method~for~quadratic~sensing}$ \cite{candes2015phase}]
	\label{li2018nonconvex}
	Consider quadratic sensing in \sref{PR}, where $\boldsymbol{X}_{\star} \in \R^{n\times r}$.  Fix any $\zeta > 0$, and suppose $m \ge c_0 nr^4 \log n$ for some sufficiently large constant $c_0>0$.  Then the spectral estimate \eqref{eq:initial_guess_qs} obeys
\[
	\mathsf{dist}^2( \bm{X}_0, \bm{X}_{\star})  \le  \zeta \sigma_{r}(\bm{M}_{\star}) 
\]
with  probability at least $1-O(n^{-2})$.
\end{theorem}
%

\subsubsection{Blind deconvolution}

The blind deconvolution problem introduced in \sref{BD} has a similar mathematical structure to that of phase retrieval. Recall the sensing model in \eref{samples-BD}. Instead of reconstructing a symmetric rank-1 matrix, we now aim to recover an \emph{asymmetric} rank-1 matrix $\vhp \vxp^\mathsf{H}$ with  sensing matrices  $\mA_i = \vb_i \va_i^\mathsf{H}$ $(1 \le i \le m)$. Following \eref{surrogate}, we form a surrogate matrix
\[
	\bm{Y} = \frac{1}{m} \sum_{i=1}^m y_i \vb_i \va_i^\mathsf{H}.
\]
Let $\bm{u}$, $\bm{v}$, and $\sigma$ denote the leading left singular vector, right singular vector, and singular value, respectively. The initial guess is then formed as 
\begin{align}
	\label{eq:init-BD}
	\bm{h}_0 = \sqrt{\sigma} \, \bm{u} \qquad \text{and} \qquad \bm{x}_0 = \sqrt{\sigma} \, \bm{v}.
\end{align}
This  estimate provably reveals sufficient information about the truth, provided that the sample size $m$ is sufficiently large. 
\begin{theorem}[$\mathsf{Spectral~method~for~blind~deconvolution}$ \cite{li2016deconvolution,ma2017implicit}]
	Consider blind deconvolution in \sref{BD}. Suppose $\vhp$ satisfies the incoherence condition in Definition~\ref{def:BD-mu} with parameter $\mu$, and assume $\|\vhp\|_2=\|\bm{x}_{\star}\|_2$. For any $\zeta > 0$, if  $m \ge c_0 \zeta^{-2} \mu^2 K \log^2 m $ for some sufficiently large constant $c_0 > 0$, then the spectral estimate \eqref{eq:init-BD} obeys
\[
	\min_{\alpha \in \C, \abs{\alpha} = 1} \big\{ \| \alpha \vh_0 - \vhp \|_2 + \| \alpha \vx_0 - \vxp\|_2 \big\}  \le \zeta \|\vhp\|_2
\]
with probability at least $1 - O(m^{-10})$.
\end{theorem}

\subsubsection{Matrix completion}

Turning to matrix completion as introduced in \sref{MC}, which is another instance of  matrix sensing with  sensing matrices taking the form of 
\begin{align}
	\label{eq:Aij-MC}
	\mA_{i,j} = \frac{1}{\sqrt{p}} \bm{e}_i \bm{e}_j^{\top} \in \R^{n_1 \times n_2}.
\end{align}
Then the measurements obey $\langle \bm{A}_{i,j}, \bm{M}_{\star} \rangle = \frac{1}{\sqrt{p}} (M_{\star})_{i,j}$. Following the aforementioned procedure, we can form a surrogate matrix as 
\begin{align}
	\label{eq:D_MC}
	\bm{Y} = \sum_{(i, j) \in \Omega} \inprod{ \bm{A}_{i,j}, \mMp} \, \bm{A}_{i,j} = \frac{1}{p} \mathcal{P}_{\Omega}(\bm{M}_{\star}).
\end{align}
Notably, the scaling factor in \eqref{eq:Aij-MC} is chosen to ensure that $\mathbb{E}[\bm{Y}]=\bm{M}_{\star}$. We then construct the initial guess for the low-rank factors $\bL_0$ and $\bR_0$ in the same manner as \eqref{eq:initial_guess}, using $\bY$ in \eqref{eq:D_MC}.

%

%
%
%
%

As $\set{ \mA_{i,j} \charfn_{\{(i,j)\in \Omega\}}}$ is a collection of independent random matrices, we can use the matrix Bernstein inequality \cite{vershynin2010nonasym} to get a high-probability upper bound on the deviation $\norm{ \bm{Y} - \EE[\bm{Y}]}$.  This in turn allows us to apply the matrix perturbation bounds to control the accuracy of the spectral method.

\begin{theorem}[$\mathsf{Spectral~method~for~matrix~completion}$ \cite{sun2016guaranteed,chen2015fast,ma2017implicit}]\label{thm:spectral_MC}
	Consider matrix completion in \sref{MC}. Fix $\zeta>0$, and suppose the condition number $\kappa$ of $\bm{M}^{\star}$ is a fixed constant. There exist a constant $c_0>0$ such that if $np > c_0 \mu^2 r^2 \log n$, then with probability at least $1 - O(n^{-10})$, the spectral estimate \eqref{eq:initial_guess} obeys
	\[
		\mathsf{dist}^2\left({\footnotesize\begin{bmatrix}
\bm{L}_{0} \\
\bm{R}_{0}
	\end{bmatrix}} , {\footnotesize\begin{bmatrix}
\bm{L}_{\star} \\
\bm{R}_{\star}
	\end{bmatrix} }\right)  \leq  \zeta \sigma_{r}(\bm{M}_{\star}) .
	\]		
\end{theorem}

\subsection{Variants of spectral methods}
\label{sec:improved-spectral}

We illustrate modifications to the spectral method, which are often found necessary to further enhance sample efficiency, increase robustness to outliers, and incorporate signal priors.

\subsubsection{Truncated spectral method for sample efficiency} \label{sec:improved-spectral-truncated}

The generic recipe for spectral methods described above works well when one has  sufficient samples compared to the underlying signal dimension. It might not be effective though if the sample complexity is on the order of  the information-theoretic limit. 

In what follows, we use phase retrieval to demonstrate the underlying issues and how to address them\footnote{Strategies of similar spirit have been proposed for other problems; see, e.g.,~\cite{keshavan2010matrix} for matrix completion. }. Recall that Theorem~\ref{thm:spectral_PR} requires a sample complexity $m \gtrsim n \log n$, which is a logarithmic factor larger than the  signal dimension $n$.  
%
What happens if we only have access to $m \asymp n$ samples, which is the information-theoretic limit (order-wise) for phase retrieval? In this more challenging regime, it turns out that we have to modify the standard recipe by applying appropriate preprocessing before forming the surrogate matrix $\bm{Y}$ in \eref{surrogate}.

We start by explaining why the surrogate matrix \eref{D_PR_org} for phase retrieval must be suboptimal in terms of sample complexity. For all $1 \le j \le m$,
\begin{equation*}
	\norm{\bm{Y}} \ge \frac{\va_j^\top \bm{Y} \va_j}{\va_j^\top \va_j}=  \frac{1}{m} \sum_{i=1}^m y_i \frac{(\va_i^\top \va_j)^2}{\va_j^\top \va_j} \geq \frac{1}{m}  y_j  \| \va_j \|_2^2.
\end{equation*}
In particular, taking $j = i^\ast  = \underset{{1 \le i \le m}}{\arg \max} \ y_i$ gives us
\begin{equation}\label{eq:bound_id2}
	\norm{\bm{Y}} \ge \frac{(\max_i y_i) \| \va_{i^\ast} \|_2^2}{m}.
\end{equation}
Under the Gaussian design, $\set{y_i / \| \vxp \|_2^2}$ is a collection of i.i.d.~$\chi^2$ random variables with 1 degree of freedom. It follows from well-known estimates in extreme value theory \cite{Ferguson:1996} that 
\[
\max_{1\le i \le m} \, y_i \approx 2\|\vxp\|_2^2 \log m,\qquad \text{as}\; m \to \infty.
\] 
Meanwhile, $\| \va_{i^\ast} \|_2^2  \approx n$ for $n$ sufficiently large. It follows from \eqref{eq:bound_id2} that
\begin{equation}\label{eq:PR_sc}
	\norm{\bm{Y}} \ge \big(1+o(1)\big) \|\vxp\|_2^2 (n \log m) / m.
\end{equation}
%



Recall that $\mathbb{E}[\bm{Y}] = 2 \vxp \vxp^\top +  \| \vxp\|_2^2 \boldsymbol{I}_n$ has a bounded spectral norm, then \eref{PR_sc} implies that, to keep the deviation between $\bm{Y}$ and $\mathbb{E}[\bm{Y}]$ well-controlled, we must at least have
\[
(n \log m) / m \lesssim 1.
\]
This condition, however, cannot be satisfied when we have linear sample complexity $m \asymp  n$. This explains why we need a sample complexity $m \gtrsim n \log n$ in Theorem~\ref{thm:spectral_PR}. 

The above analysis also suggests an easy fix: since the main culprit lies in the fact that $\max_i y_i$ is unbounded (as $m \to \infty$), we can apply a preprocessing function $\mathcal{T}(\cdot)$ to $y_i$ to keep the quantity bounded. Indeed, this is the key idea behind the {\em truncated spectral method} proposed by Chen and Cand\`es \cite{chen2015solving}, in which the surrogate matrix is modified as
\begin{equation}
	\label{eq:D_PR_T}
	\bm{Y}_{\mathcal{T}} := \frac{1}{m} \sum_{i=1}^m \mathcal{T}(y_i) \va_i \va_i^\top,
\end{equation} 
where 
\begin{equation}\label{eq:trimming}
\mathcal{T}(y) := y \charfn_{\set{\abs{y} \le  \frac{\gamma}{m} \sum_{i=1}^m y_i}}
\end{equation}
for some predetermined truncation threshold $\gamma$. The initial point $\bm{x}_0$ is then formed by scaling the leading eigenvector of $\bm{Y}_{\mathcal{T}}$ to have roughly the same norm of $\bx_{\star}$, which can be estimated by $\sigma =\frac{1}{m} \sum_{i=1}^m y_i $. This is essentially performing a trimming operation, removing any entry of $y_i$ that bears too much influence on the leading eigenvector.  The trimming step turns out to be very effective, allowing one to achieve order-wise optimal sample complexity. 
%
%

%
\begin{theorem}[$\mathsf{Truncated~spectral~method~for~phase~retrieval}$ \cite{chen2015solving}]
	\label{thm:truncated-PR}
	Consider phase retrieval in \sref{PR}. Fix any $\zeta > 0$, and suppose $m \ge c_0 n $ for some sufficiently large constant $c_0>0$.  Then the truncated spectral estimate obeys
\[
	\min\set{ \| \vx_0  - \vxp \|_2,  \| \vx_0 + \vxp  \|_2 } \le \zeta \| \vxp \|_2
\]
with probability at least $1 - O(n^{-2})$.
\end{theorem}

Subsequently, several different designs of the preprocessing function have been proposed in the literature. One example was given in \cite{wang2017solving}, where
\begin{equation}\label{eq:subset}
	\mathcal{T}(y)  = \charfn_{\set{y \ge \gamma  }},
\end{equation}
where $\gamma$ is the $(cm)$-largest value (e.g.~$c=1/6$) in $\{ y_j \}_{1\le j\le m}$. In words, this method only employs a subset of design vectors that are better aligned with the truth $\bm{x}_{\star}$.  By properly tuning the parameter $c$, this truncation scheme performs competitively as the scheme in \eref{trimming}.

\subsubsection{Truncated spectral method for removing sparse outliers}

When the samples are susceptible to adversarial entries, e.g. in the robust phase retrieval problem \eqref{eq:robust_pr_model}, the spectral method might not work properly even with the presence of a single outlier whose magnitude can be arbitrarily large to perturb the leading eigenvector of $\bm{Y}$. To mitigate this issue, a median-truncation scheme was proposed in \cite{zhang2016provable,li2017nonconvex}, where
\begin{equation}\label{eq:T_median}
	\mathcal{T}(y)  = y \charfn_{\set{y_i \leq \gamma \, \mathsf{median}\{ y_j \}_{j=1}^m}}
\end{equation}
for some predetermined constant $\gamma >0$.
 By including only a subset of samples whose values are not excessively large compared with the sample median of the samples, the preprocessing function in \eref{T_median} makes the spectral method more robust against sparse and large outliers.
 
\begin{theorem}[$\mathsf{Median}$-$\mathsf{truncated~spectral~method~for~robust}$ $\mathsf{phase~retrieval}$ \cite{zhang2016provable}]
	\label{thm:truncated-PR-init}
	Consider the robust phase retrieval problem in \eqref{eq:robust_pr_model}, and fix any $\zeta > 0$. There exist some constants  $c_0, c_1>0$ such that if 
$m\geq c_{0}n$ and $\alpha\leq c_1$, then the median-truncated spectral estimate obeys
\[
	\min\set{ \| \vx_0  - \vxp \|_2,  \| \vx_0 + \vxp  \|_2 } \le \zeta \| \vxp \|_2
\]
with probability at least $1 - O(n^{-2})$.
\end{theorem}


The idea of applying truncation to form a spectral estimate is also used in the robust PCA problem (see \sref{RPCA}). Since the observations are also potentially corrupted by large but sparse outliers, it is useful to first clean up the observations $(\bm{\Gamma}_{\star})_{i,j} = (\bm{M}_{\star})_{i,j} + (\bm{S}_{\star})_{i,j}$ before constructing the surrogate matrix as in \eref{D_MC}. Indeed, this is the strategy proposed in \cite{yi2016fast}. 
We start by forming an estimate of the sparse outliers via the hard-thresholding operation $\mathcal{H}_l(\cdot)$ defined in \eqref{eq:hard_threshodling}, as
\begin{equation}\label{eq:hard_threshodling_init}
\bm{S}_{0}=\mathcal{H}_{c\alpha np}\big(\mathcal{P}_{\Omega}(\bm{\Gamma}_{\star} )\big).
\end{equation}
where $c>0$ is some predetermined constant (e.g.~$c=3$) and $\mathcal{P}_{\Omega}$ is defined in \eqref{defn:Pomega}. Armed with this estimate, we form the surrogate matrix as
\begin{equation}
	\label{eq:Y-RPCA}
	\bY =\frac{1}{p} \mathcal{P}_{\Omega}(\bm{\Gamma}_{\star}   - \bm{S}_{0}).
\end{equation}

One can then apply the spectral method to $\bm{Y}$ in \eqref{eq:Y-RPCA} (similar to the matrix completion case). This approach enjoys the following performance guarantee.
\begin{theorem}[$\mathsf{Spectral~method~for~robust~PCA}$ \cite{yi2016fast}]\label{thm:spectral_RPCA}
Suppose that the condition number $\kappa$ of $\bm{M}_{\star}= \bm{L}_{\star}\bm{R}_{\star}^{\top}$ is a fixed constant. Fix $\zeta>0$. If the sample size and the sparsity fraction satisfy $n^{2}p\geq c_0\mu r^{2}n\log n$ and $\alpha\leq c_1/( \mu r^{3/2})$ for
	some large constant $c_0,c_1>0$, then 
with probability at least $1-O\left(n^{-1}\right)$,  
	\[
		\mathsf{dist}^2 \left( {\footnotesize\begin{bmatrix}
\bm{L}_{0} \\
\bm{R}_{0}
	\end{bmatrix}} , {\footnotesize\begin{bmatrix}
\bm{L}_{\star} \\
\bm{R}_{\star}
	\end{bmatrix} } \right)  \leq  \zeta \sigma_{r}(\bm{M}_{\star}) .
	\]	
\end{theorem}

\subsubsection{Spectral method for sparse phase retrieval}\label{sec:spectral_init_spr}
Last but not least, we briefly discuss how the spectral method can be modified to incorporate structural priors. As before, we use  the example of sparse phase retrieval to illustrate the strategy, where we assume $\bx_{\star}$ is  $k$-sparse (see Section~\ref{sec:pgd_sparse_pr}). A simple idea is to first identify the support of $\bx_{\star}$, and then try to estimate  the nonzero values by applying the spectral method over the submatrix of $\bY$ on the estimated support. Towards this end, we recall that $\mathbb{E}[\bY] = 2\bx_{\star}\bx_{\star}^{\top}+ \|\bx_{\star}\|_2^2\bI_n$, and therefore the larger ones of its diagonal entries are more likely to be included in the support. In light of this, a simple thresholding strategy adopted in \cite{cai2016optimal} is to compare $Y_{i,i}$ against some preset threshold $\gamma$:
$$ \hat{\mathcal{S}} = \{i: \; Y_{i,i}> \gamma \} .$$
The nonzero part of $\bx_{\star}$ is then found by applying the spectral method outlined in Section~\ref{sec:PR_init} to $\bm{Y}_{\hat{\mathcal{S}}} =\frac{1}{m}\sum_{i=1}^m y_i \bm{a}_{i,\hat{\mathcal{S}}}\bm{a}_{i,\hat{\mathcal{S}}}^{\top}$, where $\bm{a}_{i,\hat{\mathcal{S}}}$ is the subvector of $\bm{a}_{i}$ coming from the support $\hat{\mathcal{S}}$. 
In short, this strategy provably leads to a reasonably good initial estimate, as long as $m\gtrsim k^2 \log n$; the complete theory can be found in \cite{cai2016optimal}. See also \cite{wang66sparse} for a more involved approach, which provides better empirical performance.

\subsection{Precise asymptotic characterization and phase transitions for phase retrieval}

The Davis-Kahan and Wedin $\sin\bm{\Theta}$ Theorems are broadly applicable and convenient to use, but they usually fall short of providing the tightest estimates. For many problems, if one examines the underlying statistical models carefully, it is often possible to obtain much more precise performance guarantees for spectral methods.

In \cite{lu2017phase}, Lu and Li provided an asymptotically exact characterization of  spectral initialization in the context of generalized linear regression, which subsumes phase retrieval as a special case. One way to quantify the quality of this eigenvector is via 
the squared cosine similarity
\begin{equation}
\label{eq:rho}
\rho(\vxp, \vx_0) := \frac{({\vxp^\top \vx_0})^2}{ \| \vxp \|_2^2 \, \| \vx_0 \|_2^2},
\end{equation}
which measures the (squared) correlation between the truth $\bx_{\star}$ and the initialization $\bx_0$. The result is this:

\begin{theorem}[$\mathsf{Precise~asymptotic~characterization}$ \cite{lu2017phase}]\label{thm:cos2}
	Consider  phase retrieval  in \sref{PR}, and let $m/n=\alpha$ for some constant $\alpha>0$.  Under mild technical conditions on the preprocessing function $\mathcal{T}$, the leading eigenvector $\bx_0/\|\bx_0\|_2$ of $\bm{Y}_{\mathcal{T}}$ in \eqref{eq:D_PR_T} obeys
\begin{equation}
\label{eq:cos2_general}
\rho(\vxp, \vx_0) \cip \begin{cases}
0, &\text{if } \alpha  < \alpha_{\mathrm{c}} \\
\rho^\ast(\alpha), &\text{if } \alpha  > \alpha_{\mathrm{c}}
\end{cases}
\end{equation}
as $n\rightarrow \infty$. 
Here, $\alpha_{\mathrm{c}} > 0$ is a fixed constant and $\rho^\ast(\cdot)$ is a fixed function that is positive when $\alpha > \alpha_c$. Furthermore, $\lambda_1(\bm{Y}_{\mathcal{T}}) - \lambda_2(\bm{Y}_{\mathcal{T}})$ converges to a positive constant iff $\alpha > \alpha_{\mathrm{c}}$.


%
\end{theorem}
\begin{remark}
The above characterization was first obtained in \cite{lu2017phase}, under the assumption that $\mathcal{T}(\cdot)$ is nonnegative. Later, this technical restriction was removed in \cite{mondelli2017fundamental}.
Analytical formulas of $\alpha_{\mathrm{c}}$ and $\rho^\ast(\alpha)$ are available for any given $\mathcal{T}(\cdot)$, which can be found in \cite{lu2017phase,mondelli2017fundamental}. 
\end{remark}

The asymptotic prediction given in Theorem~\ref{thm:cos2} reveals a phase transition phenomenon: there is a critical sampling ratio $
\alpha_c$ that marks the transition between two very contrasting regimes. 
\begin{itemize}
	\item An \emph{uncorrelated phase} takes place when the sampling ratio $\alpha <  \alpha_{\mathrm{c}}$. Within this phase,  $\rho(\vxp, \vx_0) \rightarrow 0$, meaning that the spectral estimate is uncorrelated with the target. 

	\item A \emph{correlated phase} takes place when $\alpha > \alpha_{\mathrm{c}}$. Within this phase, the spectral estimate is strictly better than a random guess. Moreover, there is a nonzero gap between the 1st and 2nd largest eigenvalues of $\bm{Y}_{\mathcal{T}}$, which in turn implies that $\vx_0$ can be efficiently computed by the power method.
\end{itemize}
%


The phase transition boundary $\alpha_{\mathrm{c}}$ is determined by the preprocessing function $\mathcal{T}(\cdot)$. A natural question arises as to which  preprocessing function optimizes the phase transition point. At first glance, this seems to be a challenging task, as it is an infinite-dimensional functional optimization problem. Encouragingly, this can be analytically determined using the asymptotic characterizations stated in Theorem~\ref{thm:cos2}. 

\begin{theorem}[$\mathsf{Optimal~preprocessing}$ \cite{mondelli2017fundamental}]\label{thm:opt_T}
Consider phase retrieval in the real-valued setting. The phase transition point satisfies $\alpha_{\mathrm{c}} > \alpha^\ast = 1/2$ for any $\mathcal{T}(\cdot)$. Further,  $\alpha^\ast$ can be approached by the following preprocessing function
\begin{equation}\label{eq:T_weak}
\mathcal{T}^{\ast}_{\alpha}(y) := \frac{\sqrt{\alpha^\ast} \cdot \mathcal{T}^\ast(y)}{\sqrt{\alpha} - (\sqrt{\alpha} - \sqrt{\alpha^\ast}) \mathcal{T}^\ast(y)} ,
\end{equation}
%
where 
\begin{equation}\label{eq:T_ast}
\mathcal{T}^\ast(y) = 1 - 1/y.
\end{equation}
\end{theorem}
%

%
%
%

The value $\alpha^\ast $ is called the \emph{weak recovery threshold}. When $\alpha < \alpha^\ast$, no algorithm can generate an estimate that is asymptotically positively correlated with $\vxp$. The function $\mathcal{T}^{\ast}_{\alpha}(\cdot)$ is optimal in the sense that it approaches this weak recovery threshold. 

Another way to formulate  optimality is  via the squared cosine similarity in \eref{rho}. For any fixed sampling ratio $\alpha > 0$, we seek a preprocessing function that maximizes the squared cosine similarity, namely,
\[
\mathcal{T}_\alpha^{\mathsf{opt}}(y) = \underset{\mathcal{T}(\cdot)}{\argmax} \, \rho(\vxp, \vx_0).
\]

The following theorem \cite{luo2019optimal} shows that the fixed function $\mathcal{T}^\ast(\cdot)$ is in fact \emph{uniformly} optimal for all sampling ratio $\alpha$. Therefore, instead of using \eref{T_weak} which takes different forms depending on $\alpha$, one should directly use $\mathcal{T}^\ast(\cdot)$.

\begin{theorem}[$\mathsf{Uniformly~optimal~preprocessing}$]\label{thm:T_uniform}
Under the same settings of Theorem~\ref{thm:opt_T}, we have
$\mathcal{T}_\alpha^{\mathsf{opt}}(\cdot) = \mathcal{T}^\ast(\cdot)$,
where $\mathcal{T}^\ast(\cdot)$ is defined in \eref{T_ast}.
\end{theorem}

Finally, to demonstrate the improvements brought by the optimal preprocessing function, we show in Fig.~\ref{fig:phase_optimal} the results of applying the spectral methods to estimate a $64 \times 64$ \emph{cameraman} image from phaseless measurements under Poisson noise.
It is evident that the optimal design  significantly improves  the performance of the method. 

\begin{figure}
	\centering
	\includegraphics[width=0.3\textwidth]{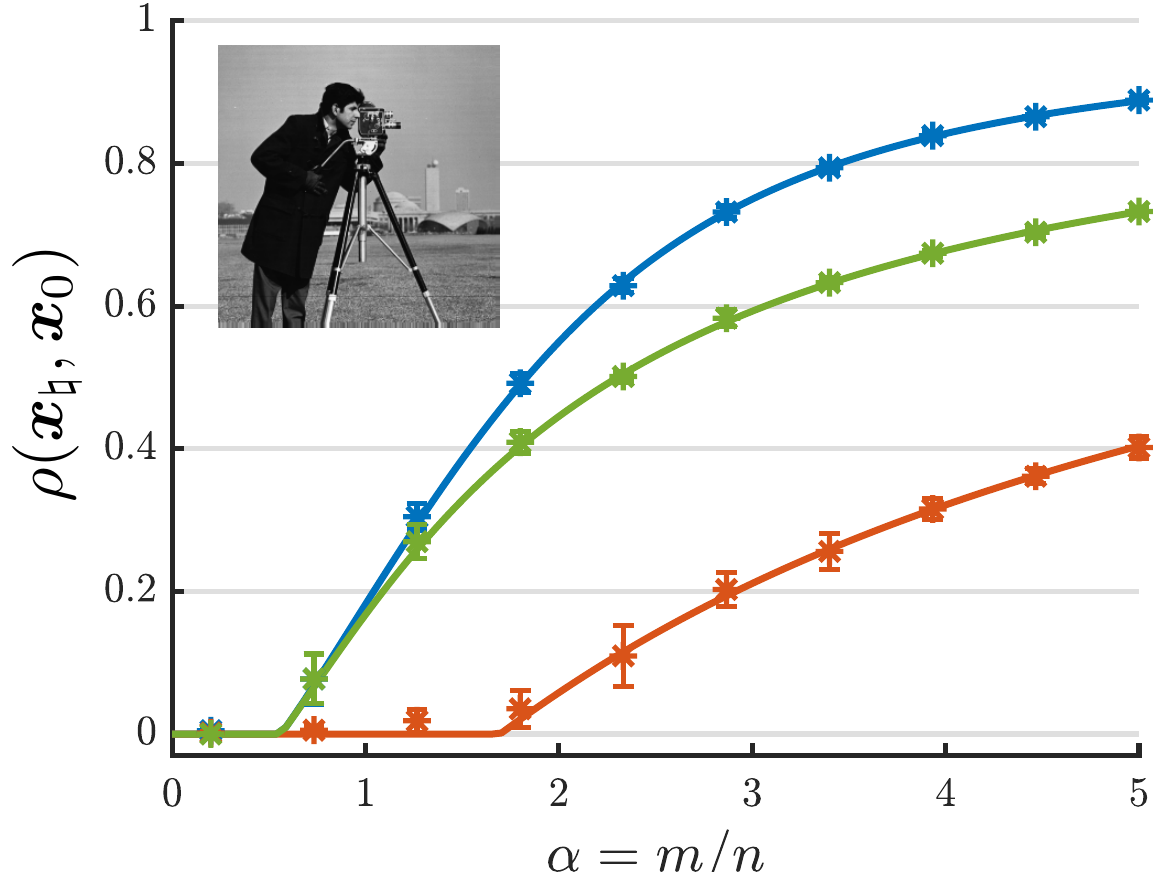}
	\caption{Comparisons of 3  designs of the preprocessing function for estimating a $64 \times 64$ \emph{cameraman} image from phaseless measurements under Poisson noise. (Red curve) the subset scheme in \eref{subset}, where the parameter $\gamma$ has been optimized for each fixed $\alpha$;  (Green curve) the function $\mathcal{T}_\alpha^\ast(\cdot)$ defined in \eref{T_weak}; (Blue curve) the uniformly optimal design $\mathcal{T}^\ast(\cdot)$ in \eref{T_ast}. }
	\label{fig:phase_optimal}
\end{figure}

\subsection{Notes} 

The idea of  spectral methods can be traced back to the early work of Li \cite{Li:92}, under the name of {\em Principal Hessian Directions} for general multi-index models. Similar spectral techniques were also proposed in \cite{achlioptas2007fast, keshavan2010matrix, jain2013low}, for initializing algorithms for low-rank matrix completion. Regarding phase retrieval,  Netrapalli et al.~\cite{netrapalli2015phase} used this method to address the problem of phase retrieval,  the theoretical guarantee of which was tightened in \cite{candes2015phase}.
Similar guarantees were also provided for the randomly coded diffraction pattern model in \cite{candes2015phase}. 
The first order-wise optimal spectral method was proposed by Chen and Cand\`es \cite{chen2015solving}, based on the truncation idea. This method has multiple variants \cite{li2017nonconvex, wang2017solving}, and has been shown to be robust against noise. The precise asymptotic characterization of the spectral method was first obtained in \cite{lu2017phase}. Based on this characterization, \cite{mondelli2017fundamental} determined the optimal weak reconstruction threshold for spectral methods. 

Finally, the spectral method has been applied to many other  problems beyond the ones discussed here, including but not limited to community detection \cite{rohe2011spectral,abbe2017community}, phase synchronization \cite{singer2011angular}, joint alignment \cite{chen2016projected},  ranking from pairwise comparisons \cite{negahban2016rank,chen2015spectral,chen2017spectral}, tensor estimation \cite{montanari2016spectral,hao2018sparse,zhang2018tensor,cai2019tensor}. We have to omit these due to the space limit. 
 

\section{Global landscape \mylinebreak
and initialization-free algorithms}
 \label{sec:global}
A separate line of work aims to study the global geometry of a loss function $f(\cdot)$ over the entire parameter space, often under appropriate statistical models of the data. As alluded by the warm-up example in Section~\ref{sec:landscape_mf}, such studies characterize the critical points and geometric curvatures of the loss surface, and highlight the (non-)existence of spurious local minima. The results of the geometric landscape analysis can then be used to understand the effectiveness of a particular optimization algorithm of choice.


\subsection{Global landscape analysis}


In general, global minimization requires one to avoid two types of undesired critical points: (1) local minima that are not global minima; (2) saddle points.
%
%
%
If all critical points of a function $f(\cdot)$ are either global minima or strict saddle points, we say that $f(\cdot)$ has  {\em benign}  landscape. Here, 
we single out strict saddles from all possible saddle points, since they are easier to escape due to the existence of descent directions.\footnote{Degenerate saddle points \cite{dauphin2014identifying} refer to critical points whose Hessian contain some eigenvalues equal to $0$. Such and higher-order saddle points are harder to escape; we refer interested readers to \cite{anandkumar2016efficient} for more discussions. } 


Loosely speaking, nonconvexity arises in these problems partly due to ``{\em symmetry}'', where the  global solutions are identifiable only up to certain global transformations. 
This necessarily leads to multiple indistinguishable local minima that are globally optimal. Further, saddle points  arise naturally when interpolating the loss surface between two separated local minima. Nonetheless, in spite of nonconvexity,  a large family of problems  exhibit benign landscape.
This subsection gives a few such examples.

\subsubsection{Two-layer linear neural network}

A straightforward instance that has already been discussed is the warm-up example in Section~\ref{sec:noncvx_eg}. It  can be slightly generalized as follows.

\begin{example}[$\mathsf{Two\text{-}layer~linear~neural~network}$ \cite{baldi1989neural}]
	\label{example:2nn}
	Given {\em arbitrary} data $\{\bx_i, \by_i \}_{i=1}^m$, $\bx_i, \by_i \in\mathbb{R}^n$, we wish to fit a two-layer linear network (see Fig.~\ref{fig:linear_2nn}) using the quadratic loss:
	$$ f(\bA, \bB) =\sum_{i=1}^{m} \left\| \by_i - \bA\bB \bx_i \right\|_2^2 =  \left\| \bY - \bA\bB \bX \right\|_{\mathrm{F}}^2, $$
where $\bA,\bB^{\top}\in\mathbb{R}^{n\times r}$ with $r\leq n$, and $\bX := [\bx_1,\cdots, \bx_m]$ and $\bY:=[\by_1,\cdots,\by_m]$.

	In this setup,  \cite{baldi1989neural} established that: 
	under mild conditions,\footnote{Specifically, \cite{baldi1989neural} assumed  $\bY\bX^{\top} (\bX\bX^{\top})^{-1} \bX\bY^{\top}$
	is full rank with $n$ distinct positive eigenvalues.} $f(\bA,\bB)$ has no spurious local minima.\footnote{In a recent work \cite{zhu2018global}, the entire landscape is further characterized.}
\end{example}

In particular, when $\bX=\bY$, Example \ref{example:2nn} reduces to rank-$r$ matrix factorization [or principal component analysis (PCA)],  an immediate extension of the rank-1 warm-up example. When $\bX\neq \bY$, Example \ref{example:2nn} is precisely the canonical correlation analysis (CCA) problem.  This explains why both PCA and CCA, though highly nonconvex, admit efficient solutions.

\begin{figure}[t]
\begin{center}
\includegraphics[width=0.25\textwidth]{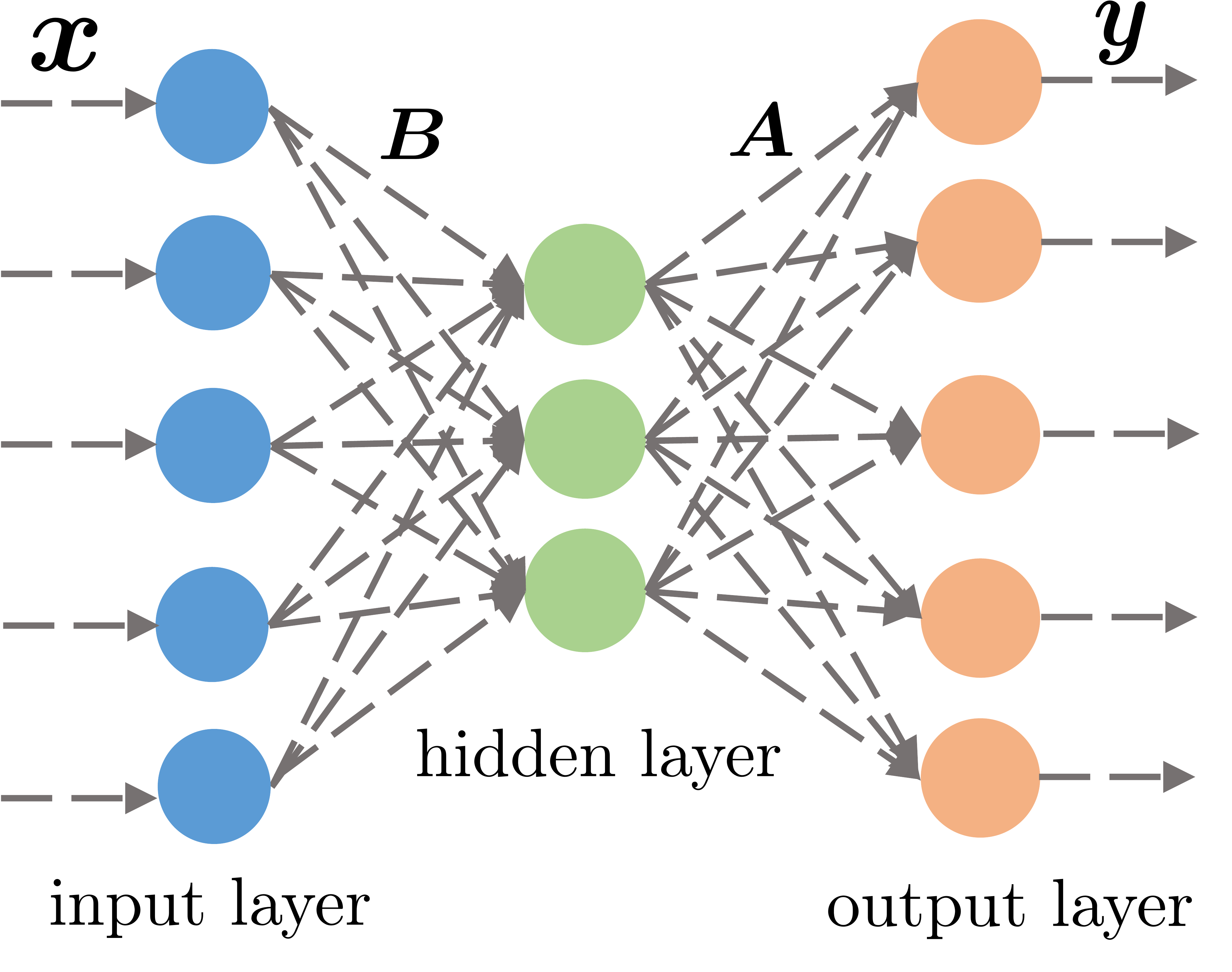}
\end{center}
\caption{Illustration of a 2-layer linear neural network.}\label{fig:linear_2nn}
\end{figure}

\subsubsection{Matrix sensing and rank-constrained optimization}

Moving beyond PCA and CCA, a more nontrivial problem is matrix sensing in the presence of RIP. The analysis for this problem, while much simpler than other problems like phase retrieval and matrix completion --- is representative of a typical strategy for analyzing problems of this kind.


\begin{theorem}
	[$\mathsf{Global~landscape~for~matrix~sensing}$ \cite{bhojanapalli2016global,li2018non}]
	\label{thm:sensing-global}
Consider the matrix sensing problem \eqref{eq:min-matrix-sensing-rank-r}.
If $\mathcal{A}$ satisfies  $2r$-RIP with $\delta_{2r}<{1}/{10}$,  then: 
\begin{itemize}
	\item {(All local minima are global):} for any local minimum $\bX$ of $f(\cdot)$, it satisfies $\bX\bX^{\top}=\bM_{\star}$;
	\item {(Strict saddles):} for any critical point $\bX$ that is not a local minimum, it satisfies $\lambda_{\min}\big( \nabla^2 f(\bX) \big) \leq -4 \sigma_{r}(\bm{M}_{\star})/5$.
\end{itemize}
\end{theorem}

\begin{proof}[Proof of Theorem~\ref{thm:sensing-global}]
For conciseness, we focus on the rank-1 case with $\bM_{\star}=\bx_{\star}\bx_{\star}^{\top}$ and show that all local minima are global.  The complete proof can be found in \cite{bhojanapalli2016global,li2018non}.

Consider any local minimum $\bm{x}$   of $f(\cdot)$. This is characterized by  the first-order
and second-order optimality conditions
\begin{subequations}
\begin{align}
\nabla f(\bm{x}) & =\frac{1}{m}\sum_{i=1}^{m}\langle\bm{A}_{i},\bm{x}\bm{x}^{\top}-\bm{x}_{\star}\bm{x}_{\star}^{\top}\rangle\bm{A}_{i}\bm{x}=\bm{0};\label{eq:1st-order-optimality-RIP}\\
	\nabla^{2}f(\bm{x}) & =\frac{1}{m}\sum_{i=1}^{m}2\bm{A}_{i}\bm{x}\bm{x}^{\top}\bm{A}_{i}+\langle\bm{A}_{i},\bm{x}\bm{x}^{\top}-\bm{x}_{\star}\bm{x}_{\star}^{\top}\rangle\bm{A}_{i}\succeq\bm{0}.\label{eq:2nd-order-optimality-RIP}
\end{align}
\end{subequations}
Without loss of generality, assume that $\|\bx-\bx_{\star}\|_2 \leq \|\bx+\bx_{\star}\|_2$.

A typical proof idea is to demonstrate that: if $\bm{x}$ is not globally optimal, then one can identify a descent direction, thus contradicting the local optimality of $\bm{x}$. A natural guess of such a descent direction would
be the direction towards the truth, i.e.~$\bm{x}-\bm{x}_{\star}$. As a result, the proof consists in showing that: when the RIP constant is sufficiently small, one has
\begin{equation}
	\label{eq:goal-UB-sensing}
	(\bm{x}-\bm{x}_{\star})^{\top}\nabla^{2}f(\bm{x})\,(\bm{x}-\bm{x}_{\star}) < 0
\end{equation}
	unless $\bx = \bx_{\star}$. Additionally, the value of \eqref{eq:goal-UB-sensing} is helpful in upper bounding $\lambda_{\min} ( \nabla^2 f(\bx) )$ if $\bx$ is a saddle point. 
	See Appendix \ref{appendix:sec-proof-theorem-sensing-global} for details. 
\end{proof}

We take a moment to expand on this result. Recall that we have introduced a version of strong convexity and smoothness for $f(\cdot)$ when accounting for global orthogonal transformation (Section \ref{sec:GD-RIP-rankr}). Another way to express this is through a different  parameterization
\begin{align}
	g(\bm{M}):=\frac{1}{4m}\sum_{i=1}^{m}\big(\langle\bm{A}_{i},\bm{M}\rangle-\langle\bm{A}_{i},\bm{X}_{\star}\bm{X}_{\star}^{\top}\rangle \big)^{2},    	\label{eq:min-matrix-sensing-rank-r-g}	
\end{align}
which clearly satisfies $g(\bX \bX^{\top}) = f(\bX)$. It is not hard to show that: in the presence of the RIP, the Hessian $\nabla^2 g(\cdot)$ is well-conditioned when restricted to low-rank decision variables and directions. This motivates the following more general result, stated in terms of certain {\em restricted well-conditionedness} of $g(\cdot)$.  One of the advantages is its applicability to more general loss functions beyond the squared loss.

\begin{theorem}[$\mathsf{Global~landscape~for~rank\text{-}constrained~problems}$ \cite{li2018non}]\label{thm:global_landscape_general}Let $g(\cdot)$ be a convex function.  Suppose  that
\begin{equation}\label{eq:ncvx_symmetric}
	\underset{\bm{M}\succeq \bm{0}}{\mathrm{minimize}}~~ g(\bM)
\end{equation}
admits a solution $\bM_{\star}$ with $\mathrm{rank}(\bM_{\star})=r< n$. 
Assume that for all $\bM$ and $\bD$ with $\mathrm{rank}(\bM)\leq 2r$ and $\mathrm{rank}(\bD)\leq 4r$,
\begin{equation}
\alpha \|\bm{D}\|_{\mathrm{F}}^2 \leq \mathsf{vec}(\bD)^{\top} \nabla^2g(\bM) \mathsf{vec}(\bD) \leq \beta \|\bm{D}\|_{\mathrm{F}}^2  
\end{equation}
holds with $\beta/\alpha \leq 3/2$.
Then the function $f(\bX)=g(\bX\bX^{\top})$, where $\bX\in\mathbb{R}^{n\times r}$, has no spurious local minima, and all saddle points of $f(\cdot)$ are strict saddles. 	

\end{theorem}

\begin{remark}
	This result continue to hold if $\mathrm{rank}(\bM)\leq 2\tilde{r}$ and $\mathrm{rank}(\bD)\leq 4\tilde{r}$ for some $\tilde{r}>r$ (although the restricted well-conditionedness needs to be valid w.r.t.~this new rank) \cite{li2018non}. 
	In addition, it has also been extended to accommodate nonsymmetric matrix factorization in \cite{li2018non}. 
\end{remark}


\subsubsection{Phase retrieval and matrix completion}

Next, we move on to problems that fall short of restricted well-conditionedness. As it turns out, it is still possible to have benign landscape, although the Lischiptz constants w.r.t.~both gradients and Hessians might be much larger.  A typical example is phase retrieval, for which the smoothness condition is not well-controlled (as discussed in Lemma \ref{lem:strong-cvx-PR}).

\begin{theorem}[$\mathsf{Global~landscape~for~phase~retrieval}$ \cite{sun2016geometric}] Consider the phase retrieval problem \eqref{eq:min-PR}.  If the sample size  $m\gtrsim n\log^3 n $,
then with high probability, there is no spurious local minimum, and all saddle points of $f(\cdot)$ are strict saddles. 
\end{theorem}


Further, we turn to the kind of loss functions that only satisfy {\em highly restricted} strong convexity and smoothness.  In some cases,  one might be able to properly regularize the loss function to enable benign landscape. Here, regularization can be enforced in a way similar to regularized gradient descent as discussed in Section \ref{sec:Reg-GD}. In the following, we use matrix completion as a representative example.

\begin{theorem}[$\mathsf{Global~landscape~for~matrix~completion}$ \cite{ge2016matrix,ge2017no,chen2017memory}]
	\label{thm:landscape-MC}
	Consider the problem \eqref{eq:MC-empirical-risk} but replace $f(\cdot)$ with a regularized loss
$$ f_{\mathsf{reg}}(\bX) = f(\bX) +  \lambda \sum_{i=1}^n G_0(\|\bX_{i,\cdot}\|_2)$$
with $\lambda>0$ a regularization parameter, and $G_0(z)= \max\{z-\alpha,0\}^4$ for some $\alpha>0$.  For properly selected $\alpha$ and $\lambda$,  if the  sample size $n^2p\gtrsim  n \max\{ \mu \kappa r \log n, \mu^2 \kappa^2 r^2 \}$, then with high probability,  
all local  minima $\bX$ of $f_{\mathsf{reg}}(\cdot)$  satisfies $\bX\bX^{\top}=\bX_{\star}\bX_{\star}^{\top}$, and all saddle points of $f_{\mathsf{reg}}(\cdot)$ are strict saddles. 
\end{theorem}
\begin{remark}
	The study of global landscape in matrix completion was initiated in \cite{ge2016matrix}. The current result in Theorem \ref{thm:landscape-MC} is taken from  \cite{chen2017memory}.
\end{remark}

\subsubsection{Over-parameterization}

Another scenario that merits special attention is over-parametrization \cite{gunasekar2017implicit,li2017algorithmic}. Take matrix sensing  and phase retrieval for instance: if we lift the decision variable to the full matrix space (as opposed to using the low-rank decision variable), the resulting optimization problem is
\begin{equation*}
\underset{\bm{X}\in\mathbb{R}^{n\times n}}{\text{minimize}} ~ f_{\mathsf{op}}(\bm{X})=\frac{1}{m}\sum_{i=1}^{m}\big( \langle\bm{A}_{i},\bm{X}\bm{X}^{\top}\rangle-\langle\bm{A}_{i},\bm{M}_{\star}\rangle\big)^{2}. 
\end{equation*}
where $\bm{M}_{\star}$ is the true low-rank matrix. Here, $\bA_i = \ba_i \ba_i^{\top}$ and $\bm{M}_{\star}=\bm{x}_{\star} \bm{x}_{\star}^{\top}$ for phase retrieval.  As it turns out, under minimal sample complexity,  $f_{\mathsf{op}}(\cdot)$ does not have spurious local minima even though we over-parametrize the model significantly; in other words, enforcing the low-rank constraint is not crucial for optimization.  
\begin{theorem}[$\mathsf{Over\text{-}parametrized~matrix~sensing}$ $\mathsf{and}$ $\mathsf{phase}$ $\mathsf{retrieval}$]
	\label{thm:over-parametrize}
	Any local minimum $\bm{X}_{\star}$ of the above function $f_{\mathsf{op}}(\cdot)$  obeys  $\bm{X}_{\star} \bm{X}_{\star}^{\top} = \bm{M}_{\star}$, provided that  the set
	\begin{equation} \label{eq:singleton_property}
\left\{ \bm{M}\succeq\bm{0}\mid\langle\bm{A}_{i},\bm{M}\rangle=\langle\bm{A}_{i},\bm{M}_{\star}\rangle,\text{ }1\leq i\leq m\right\} 
\end{equation}
	is a singleton $\{\bM_{\star}\}$.
	\end{theorem}

The singleton property assumed in Theorem \ref{thm:over-parametrize} has been established, for example, for the following two problems:
	\begin{itemize}
		\item matrix sensing (the positive semidefinite case where $\bm{M}_{\star}\succeq \bm{0}$), as long as $\mathcal{A}$ satisfies $5r$-RIP with a RIP constant $\delta_{5r} \leq 1/10$ \cite{kyrillidis2018implicit}; a necessary and sufficient condition was established in \cite{wang2011unique}; 
		\item phase retrieval, as long as $m\geq c_0n$ for some sufficiently large constant $c_0>0$; see \cite{demanet2012stable,candes2012solving}.
	\end{itemize}  
As an important implication of Theorem \ref{thm:over-parametrize},  even  solving the over-parametrized optimization problem allows us to achieve perfect recovery for the above two problems. This was first observed by \cite{gunasekar2017implicit} and then made precise by  \cite{li2017algorithmic}; they showed that running GD w.r.t.~the over-parameterized loss $f_{\mathsf{op}}(\cdot)$  also (approximately) recovers  the truth under roughly the same sample complexity, provided that a near-zero initialization is adopted. 
	\begin{proof}[Proof of Theorem \ref{thm:over-parametrize}]
	Define the function
\[ 
	g(\bm{M}) :=\frac{1}{m}\sum_{i=1}^{m}\big( \langle\bm{A}_{i},\bm{M}\rangle-\langle\bm{A}_{i},\bm{M}_{\star}\rangle\big)^{2},
\]
which satisfies $g(\bX \bX^{\top}) = f_{\mathsf{op}}(\bX)$. Consider any local minimum $\bm{X}_{\star}$ of $f_{\mathsf{op}}(\cdot)$. By definition,
there is an $\epsilon$-ball $\mathcal{B}_{\epsilon}(\bm{X}_{\star}):=\{\bm{X}\mid\|\bm{X}-\bm{X}_{\star}\|_{\mathrm{F}}\leq\epsilon\}$ with  small enough $\epsilon>0$ 
such that
\begin{equation}
f_{\mathsf{op}}(\bm{X})\geq f_{\mathsf{op}}(\bm{X}_{\star}),\qquad\forall\bm{X}\in\mathcal{B}_{\epsilon}(\bm{X}_{\star}).\label{eq:local-min-sensing}
\end{equation}
Now suppose that $\bm{X}_{\star}$ (resp.~$\bm{X}_{\star}\bm{X}_{\star}^{\top}$) is not a global solution of $f_{\mathsf{op}}(\cdot)$ (resp.~$g(\cdot)$), then
there exists a $\tilde{\bm{X}}\tilde{\bm{X}}^{\top}$ (resp.~$\tilde{\bm{X}}$) arbitrarily
close to $\bm{X}_{\star}\bm{X}_{\star}^{\top}$ (resp.~${\bm{X}_{\star}}$) such that
\[
	f_{\mathsf{op}}(\tilde{\bm{X}}) = g(\tilde{\bm{X}}\tilde{\bm{X}}^{\top})<g(\bm{X}_{\star}\bm{X}_{\star}^{\top}) = f_{\mathsf{op}}(\bm{X}_{\star}).
\]
For instance, one can take $\tilde{\bm{X}}\tilde{\bm{X}}^{\top}=(1-\zeta) \bm{X}_{\star}\bm{X}_{\star}^{\top} + \zeta \bm{X}_{\star}\bm{X}_{\star}^{\top}$ with $\zeta \searrow 0$.
This implies the existence of an $\tilde{\bm{X}}$ that contradicts (\ref{eq:local-min-sensing}), thus establishing the claim. 
%
\end{proof}

\subsubsection{Beyond low-rank matrix factorization}

Finally, we note that benign landscape has been observed in numerous other contexts beyond matrix factorization. While they are beyond the scope of this paper, we briefly mention two important cases based on chronological order: 
\begin{itemize}
\item {\em Dictionary learning }\cite{sun2017complete,sun2017trust,zhai2019complete}. Observe a data matrix $\bY$ that can be factorized as $\bY=\bA\bX$, where $\bA$ is a square invertible dictionary matrix and $\bX$ encodes the sparse coefficients. The goal is to learn $\bA$.  It is shown that a certain smooth approximation of the $\ell_1$ loss exhibits benign nonconvex geometry over a sphere. 
\item {\em M-estimator in statistical estimation }\cite{mei2016landscape}. Given a set of independent data points $\{\bx_1,\cdots,\bx_n\}$, this work studied when the empirical loss function inherits the benign landscape of the population loss function. The results provide a general framework for establishing uniform convergence of the gradient and the Hessian of the empirical risk, and cover several examples such as binary
classification, robust regression, and Gaussian mixture models.  
\end{itemize}



\subsubsection{Notes}

The study of benign global landscapes dates back to the works on shallow neural networks in the 1980s \cite{baldi1989neural} with deterministic data, if not earlier. Complete dictionary learning analyzed by Sun et al.~\cite{sun2017complete,sun2017trust}  was perhaps the first ``modern'' example where benign landscape is analyzed, which exploits measure concentration of random data. This work inspired the line of global geometric analysis for many aforementioned problems, including phase retrieval \cite{sun2016geometric,davis2017nonsmooth}, matrix sensing \cite{bhojanapalli2016global,li2018non,ge2017no,zhang2019sharp,zhang2018much}, matrix completion \cite{ge2016matrix,ge2017no,chen2017memory,li2018non}, and robust PCA \cite{ge2017no}. 
For phase retrieval, \cite{sun2016geometric} focused on the smooth squared loss in \eqref{eq:min-PR}. The landscape of a more ``robust'' nonsmooth formulation $f(\bx)=\frac{1}{m} \sum_{i=1}^m  \big| (\bm{a}_i^{\top} \bm{x})^2 - (\bm{a}_i^{\top} \bm{x}_{\star})^2 \big|$ has been studied by \cite{davis2017nonsmooth}.
The optimization landscape for matrix completion was pioneered by \cite{ge2016matrix}, and  later improved by \cite{ge2017no,chen2017memory}.  In particular, \cite{chen2017memory} derived a model-free theory where  no assumptions are imposed on $\bm{M}_{\star}$, which accommodates, for example, the noisy case and the case where the truth is only approximately low-rank. The global landscape of asymmetric matrix sensing / completion holds similarly by considering a loss function regularized by the term $g(\bL,\bR):= \left\|\bL^{\top}\bL - \bR^{\top}\bR \right\|_{\mathrm{F}}^2$ \cite{park2017non,ge2017no} or by the term $g(\bL,\bR):= \left\|\bL\right\|_{\mathrm{F}}^2 + \left\| \bR \right\|_{\mathrm{F}}^2$ \cite{li2018non}. Theorem~\ref{thm:global_landscape_general} has been generalized to the asymmetric case in \cite{zhu2017global}. 
 Last but not least, we caution that, many more nonconvex problems are not benign and indeed have bad local minima; for example, spurious local minima are common even in simple neural networks with nonlinear activations \cite{auer1996exponentially,safran2018spurious}.

\subsection{Gradient descent with random initialization}

For many problems described above with benign landscape, there is no spurious local minima, and the only task is to escape strict saddle points and to find second-order critical points, which are now guaranteed to be global optima.  In particular, our main algorithmic goal is to find a second-order critical point of a function exhibiting benign geometry. To make the discussion more precise, we define the functions of interest as follows
%
\begin{definition}[$\mathsf{strict~saddle~property}$ \cite{ge2015escaping,sun2016nonconvex}]\label{def:strict_saddle}
A function $f(\cdot) $ is said to satisfy the $ (\varepsilon, \gamma, \zeta) $-strict saddle property for some $\varepsilon,\gamma,\zeta>0$, if for each $ \bx $ at least one of the following holds:
	\begin{itemize}
		\item {\bf (strong gradient)} $ ~\|\nabla f(\bx)\|_2 \ge \varepsilon   $; 
		\item {\bf (negative curvature)} $ ~\lambda_{\min} \left( \nabla^2 f(\bx)\right) \le -\gamma $;
		\item {\bf (local minimum)} ~there exists a local minimum $ \bx_{\star}$ such that $ \|\bx - \bx_{\star}\|_2 \le \zeta$.
	\end{itemize}
\end{definition}
\noindent In words, this property says that: every point either has a large gradient, or has a negative directional curvature, or lies sufficiently close to a local minimum. In addition, while we have not discussed the strong gradient condition in the preceding subsection,  it arises for most of the aforementioned problems when $\bm{x}$ is not close to the global minimum.

A natural question arises as to whether an algorithm as simple as gradient descent can converge to a second-order critical point of a function satisfying this property.  Apparently, GD cannot start from anywhere; for example, if it starts from any undesired critical point (which obeys $\nabla f(\bm{x})=\bm{0}$), then it gets trapped. But what happens if we initialize GD randomly? 

A recent work \cite{lee2016gradient} provides the first answer to this question. Borrowing tools from dynamical systems theory, it proves that:
\begin{theorem}[$\mathsf{Convergence~of~GD~with~random~initialization}$ \cite{lee2017first}]
	\label{thm:GD-random-Lee}
	Consider any twice continuously differentiable function $f$ that satisfies the strict saddle property (Definition \ref{def:strict_saddle}). If $\eta_t < 1/\beta$ with $\beta$ the smoothness parameter, then
	GD with a random initialization converges to a local minimizer or $-\infty$ almost surely. 
\end{theorem}
This theorem says that for a broad class of benign functions of interest,   GD --- when randomly initialized --- never gets stuck in the saddle points.  The following example helps  develop a better understanding of this theorem. 
\begin{example}
	\label{example:quadratic}
	Consider a quadratic minimization problem
	\[
		\underset{\bx\in \mathbb{R}^n}{\mathrm{minimize}} \quad f(\bm{x}) = \frac{1}{2} \bx^{\top} \bA \bx.
	\]
	The GD rule is $\bm{x}_{t+1}=\bm{x}_t - \eta_t \bm{A}\bm{x}_t$. If  $\eta_t \equiv \eta <1/ \|\bm{A}\|$, then
	\[
		\bm{x}_t = (\bm{I} - \eta \bA)^t \,\bm{x}_0.
	\]
Now suppose that $\lambda_1(\bm{A})\geq \cdots \geq \lambda_{n-1}(\bA)>0>\lambda_n(\bA)$, and let $\mathcal{E}_{+}$ be the subspace spanned by  the first $n-1$ eigenvectors. It is easy to see that $\bm{0}$ is a saddle point, and that $\bm{x}_t \rightarrow \bm{0}$ only if $\bm{x}_0\in \mathcal{E}_{+}$. In fact, as long as  $\bm{x}_0$ contains a component outside  $\mathcal{E}_{+}$, then this component will keep blowing up at a rate $1+\eta |\lambda_n(\bA)|$. Given that $\mathcal{E}_{+}$ is $(n-1)$-dimensional, we have $\mathbb{P}(\bm{x}_0 \in \mathcal{E}_{+})=0$. As a result, $\mathbb{P}(\bm{x}_t\rightarrow \bm{0})=0$.
\end{example}
The above theory has been extended to accommodate other optimization methods like coordinate descent, mirror descent,  the gradient primal-dual algorithm, and alternating minimization \cite{lee2017first,hong2018gradient,li2019alternating}. In addition, the above theory is generic and accommodates all benign functions satisfying the strict saddle property.

We caution, however, that almost-sure convergence does not imply fast convergence.  In fact, there exist non-pathological functions such that randomly initialized GD takes time exponential in the ambient dimension to escape saddle
points \cite{du2017gradient}. That being said, it is possible to develop {\em problem-specific theory} that reveals much better convergence guarantees. Once again, we take phase retrieval as an example. 
\begin{theorem}[$\mathsf{Randomly~initialized~GD~for~phase~retrieval}$ \cite{chen2018gradient}]
\label{thm:GD-PR-random-init}
	Consider the problem \eqref{eq:min-PR}, and suppose that $m\gtrsim n \mathsf{poly}\log m$. 
	The GD iterates with random initialization $\bx_0 \sim \mathcal{N}(\bm{0}, \frac{\|\bx_{\star}\|_2^2}{n} \bm{I}_n )$ and $\eta_t \equiv 1/(c_3\|\bm{x}_{\star}\|_2^2)$ obey
\begin{subequations}
\begin{align}
	& \|\bm{x}_{t}-\bm{x}_{\star}\|_{2} \leq\left(1- c_4\right)^{t-T_0}\|\bm{x}_{0}-\bm{x}_{\star}\|_{2}, \quad t\geq 0  \label{eq:PR-random-init} 
\end{align}
\end{subequations}
with probability $1-O(n^{-10})$. Here, $c_3,c_4>0$ are some constants, $T_0\lesssim \log n$,  and  we assume $\|\bm{x}_{0}-\bm{x}_{\star}\|_2 \leq \|\bm{x}_{0}+\bm{x}_{\star}\|_2$. 
\end{theorem}
To be more precise, the algorithm consists of two stages:
\begin{itemize}
	\item {\em When $0\leq t \leq T_0 \lesssim \log n$}: this stage allows GD  to find and enter the local basin surrounding the truth, which takes time no more than $O(\log n)$ steps. To explain why this is fast, we remark that the signal strength $| \langle \bx_t, \bx_{\star} \rangle |$ in this stage grows exponentially fast, while the residual strength $\big \| \bm{x}_{t}-\frac{|\langle\bm{x}_{t},\bm{x}_{\star}\rangle|}{\|\bm{x}_{\star}\|_{2}^{2}}\bm{x}_{\star} \big\|_2$ does not increase by much.  
	
	\item {\em When $t > T_0$}: once the iterates enters the local basin,  the $\ell_2$ estimation error decays exponentially fast, similar to Theorem \ref{thm:GD-PR-improved}. This stage takes about $O(\log \frac{1}{\varepsilon})$ iterations to reach $\varepsilon$-accuracy. 
\end{itemize}
Taken collectively, GD with random initialization achieves $\varepsilon$-accuracy in $O\big(\log n + \log \frac{1}{\varepsilon}\big)$ iterations, making it appealing for solving large-scale problems.
It is worth noting that the GD iterates never approach or hit the saddles; in fact, there is often a positive force dragging the GD iterates away from the saddle points.

Furthermore, there are other examples for which  randomly initialized GD converges fast; see \cite{Wright2018random,brutzkus2017globally} for further examples. We note, however, that the theoretical support for random initialization  is currently lacking for many important problems (including matrix completion and blind deconvolution).



\subsection{Generic saddle-escaping algorithms}



Given that gradient descent with random initialization has only been shown to be efficient for specific examples, it is natural to ask how to design generic optimization algorithms to efficiently escape saddle points for all functions with benign geometry (i.e.~those satisfying the strict saddle property in Definition \ref{def:strict_saddle}).  To see why this is hopeful, consider any strict saddle point $\bm{x}$ (which obeys $\nabla f(\bm{x})=\bm{0}$). 
The Taylor expansion implies 
\[
f(\bm{x}+\bm{\Delta})\approx f(\bm{x})+\frac{1}{2}\bm{\Delta}^{\top}\nabla^{2}f(\bm{x})\,\bm{\Delta}
\]
for any $\bm{\Delta}$ sufficiently small. Since $\bm{x}$ is a strict saddle, one can identify a direction of negative curvature and further decrease the objective value (i.e.~$f(\bm{x}+\bm{\Delta})<f(\bm{x})$). In other words, the existence of negative curvatures enables efficient escaping from undesired saddle points.

Many algorithms have been proposed towards the above goal. 
Roughly speaking, the available algorithms can be categorized into three classes, depending on which basic operations are needed: (1) Hessian-based algorithms; (2) Hessian-vector-product-based algorithms; (3) gradient-based algorithms.
\begin{remark}
	Caution needs to be exercised as this categorization is very rough at best. One can certainly argue that many Hessian-based operations can be carried out via Hessian-vector products, and that Hessian-vector products can be computed using only gradient information. 
\end{remark}
Before proceeding, we introduce two notations that are helpful in presenting the theory. First, following the convention \cite{nesterov2006cubic}, a point $\bx$  is said to be an $\varepsilon$-second-order critical point if
\begin{equation}
	\label{eq:epsilon-2nd-critical}
	\|\nabla f(\bx)\|_2 \leq \varepsilon~~\mbox{and}~~\lambda_{\min}(\nabla^2 f(\bx))\geq -\sqrt{L_2 \varepsilon}. 
\end{equation}
Next, the algorithms often require the loss functions to be sufficiently smooth in the following sense:
\begin{definition}[$\mathsf{Hessian~Lipschitz~continuity}$] 
	\label{def:Hessian-Lipschitz}
	A function $f(\cdot)$ is $L_2$-Hessian-Lipschitz if, for any $\bx_1, \bx_2$, one has 
	\begin{equation}
		\| \nabla^2 f(\bx_1) - \nabla^2 f(\bx_2)\| \leq L_2 \| \bx_1 -\bx_2 \|_2.
	\end{equation}
\end{definition}

\subsubsection{Hessian-based algorithms} 
This class of algorithms requires an oracle that returns $\nabla^2 f(\bx)$ for any given $\bx$. Examples include the trust-region method  and the cubic-regularized Newton method. 
\begin{itemize}
	\item {\em Trust-region methods (TRM).}
	At the $t$th iterate, this method minimizes a  quadratic approximation of the loss function in a local neighborhood, where the quadratic approximation is typically based on a second-order Taylor expansion at the current iterate \cite{conn2000trust,absil2007trust}.  Mathematically, 
\begin{align}\label{eq:trust_{r}egion}\bm{x}_{t+1}=\argmin_{\bm{z}:\|\bm{z}-\bm{x}_{t}\|_{2}\leq\Delta}\Big\{\langle\nabla f(\bm{x}_{t}),\bm{z}-\bm{x}_{t}\rangle \mynonumber\mylinebreak
\myquad \myquad\myquad+\frac{1}{2}\big\langle\nabla^{2}f(\bm{x}_{t})\,(\bm{z}-\bm{x}_{t}),\bm{z}-\bm{x}_{t}\big\rangle\Big\}
\end{align}
		for some properly chosen radius $\Delta$. By choosing the size of the local neighborhood --- called the ``{\em trust region}'' --- to be sufficiently small, we can expect the second-order approximation within this region to be reasonably reliable. Thus, if $\bm{x}_t$ happens to be a strict saddle point, then TRM is still able to find a descent direction  due to the negative curvature condition.  This means that such local search will not get stuck in undesired saddle points.


\item {\em Cubic regularization.}
This method attempts to minimize a cubic-regularized upper bound on the loss function \cite{nesterov2006cubic}:
\begin{align} 
	\label{eq:cubic_sub}
	& \hspace{-0.6em} \bx_{t+1} = \argmin_{\bz}  \Big\{ \big\langle \nabla f(\bx_t), \bz-\bx_t \big\rangle + \mynonumber \mylinebreak
	\myalign  \hspace{-0.2em} \frac{1}{2} \big\langle \nabla^2 f(\bx_t)\,(\bz-\bx_t), \bz-\bx_t \big\rangle +\frac{L_2}{6}\|\bz-\bx_t\|_2^3 \Big\}. 
\end{align}
Here, $L_2$ is taken to be the Lipschitz constant of the Hessians (see Definition \ref{def:Hessian-Lipschitz}), so as to ensure that the objective function in \eqref{eq:cubic_sub} majorizes the true objective $f(\cdot)$. 
While the subproblem \eqref{eq:cubic_sub} is nonconvex and may have local minima, it can often be efficiently solved by 
minimizing an explicitly written univariate convex function \cite[Section 5]{nesterov2006cubic}, or even by gradient descent \cite{carmon2016gradient}. 

\end{itemize}
Both methods are capable of finding an $\varepsilon$-second-order stationary  point (cf.~\eqref{eq:epsilon-2nd-critical}) in $O(\varepsilon^{-1.5})$ iterations \cite{curtis2017trust,nesterov2006cubic}, where each iteration consists of one gradient and one Hessian computations.


\subsubsection{Hessian-vector-product-based algorithms} 
Here, the oracle returns $\nabla^2 f(\bx)\,\bm{u}$ for any given $\bx$ and $\bu$. In what follows, we review the method by Carmon et al.~\cite{carmon2018accelerated} which falls under this category. There are two basic subroutines that are carried out using Hessian-vector products. 
		\begin{itemize}
			\item $\mathsf{Negative\text{-}Curvature\text{-}Descent}$, which allows one to find, if possible, a direction that decreases the objective value. This is achieved by computing the eigenvector corresponding to the smallest eigenvalue of the Hessian.  
			\item $\mathsf{Almost\text{-}Convex\text{-}AGD}$, which invokes Nesterov's accelerated gradient method \cite{nesterov1983method}   to optimize an almost convex function superimposed with a squared proximity term. 
		\end{itemize}
		In each iteration, based on the estimate of the smallest eigenvalue of the current Hessian, the algorithm decides whether to move along the direction computed by $\mathsf{Negative\text{-}Curvature\text{-}Descent}$ (so as not to get trapped in saddle points), or to apply $\mathsf{Almost\text{-}Convex\text{-}AGD}$ to optimize an almost convex function.  	See \cite{carmon2018accelerated} for details. This method converges to an $\varepsilon$-second-order stationary  point in $O(\varepsilon^{-7/4}\log\frac{1}{\varepsilon} )$ steps, where each step involves computing a Hessian-vector product.

Another method that achieves about the same computational complexity is Agarwal et al.~\cite{agarwal2017finding}.  This is  a fast variant of cubic regularization. The key idea is to invoke, among other things, fast multiplicative approximations to accelerate the subproblem of the cubic-regularized Newton step.  The interested readers shall consult \cite{agarwal2017finding} for details.

\subsubsection{Gradient-based algorithms} 

Here, an oracle outputs $\nabla f(\bx)$ for any given $\bx$. These methods are computationally efficient since only first-order information is required.

Ge et al.~\cite{ge2015escaping} initiated this line of work by designing a first-order algorithm that provably escapes strict saddle points in polynomial time {\em with high probability}. The algorithm proposed therein is a noise-injected version of (stochastic) gradient descent, namely,
\begin{align}
	\bm{x}_{t+1} = \bm{x}_t - \eta_t ( \nabla f (\bm{x}_t) + \bm{\zeta}_t ),
	\label{eq:SGD-Ge}
\end{align}
where  $\bm{\zeta}_t$ is some noise sampled uniformly from a sphere, and $\nabla f(\bx_t)$ can also be replaced by a stochastic gradient. The iteration complexity, however, depends on the dimension polynomially, which grows prohibitively as the ambient dimension of $\bm{x}_t$ increases. Similar high iteration complexity holds for \cite{levy2016power} which is based on injecting noise to normalized gradient descent.

The computational guarantee was later improved by  {\em perturbed gradient descent}, proposed in \cite{jin2017escape}. In contrast to \eqref{eq:SGD-Ge}, perturbed GD adds noise to the iterate before computing the gradient, namely, 
\begin{align}
	  \bx_{t} &\leftarrow \bx_t + \bm{\zeta}_t \label{eq:perturbed-GD} \\
	\bm{x}_{t+1} &= \bm{x}_t - \eta_t  \nabla f (\bm{x}_t)   \nonumber
\end{align}
with $\bm{\zeta}_t$ uniformly drawn from a sphere.  The crux of perturbed GD  is to realize that strict saddles are unstable; it is possible to escape them and make progress with a slight perturbation. It has been shown that perturbed GD finds an $\varepsilon$-second-order stationary point in $O( \varepsilon^{-2})$ iterations (up to some logarithmic factor) with high probability, and is hence almost dimension-free. Note that each iteration only needs to call the gradient oracle once. 

Moreover, this can be further accelerated via Nesterov's momentum-based methods \cite{nesterov1983method}. Specifically,  \cite{jin2017accelerated} proposed a perturbed version of Nesterov's accelerated gradient descent (AGD), which adds perturbation to AGD and combines it with an operation similar to $\mathsf{Negative\text{-}Curvature\text{-}Descent}$ described above. This accelerated method converges to an  $\varepsilon$-second-order stationary  point in $O(\varepsilon^{-7/4})$ iterations (up to some logarithmic factor) with high probability, which matches the computational complexity of \cite{carmon2018accelerated,agarwal2017finding}.

There are also several other algorithms  that provably work well in the presence of stochastic\,/\,incremental first-order and\,/\,or second-order oracles \cite{allen2017natasha,xu2017first,allen2017neon2}. Interested readers are referred to the ICML tutorial \cite{allen2017tutorial} and the references therein. Additionally, we note that it is possible to adapt many of the above mentioned saddle-point escaping algorithms onto manifolds; we refer the readers to \cite{absil2007trust,sun2017trust,boumal2016global}.

\subsubsection{Caution}

Finally, it is worth noting that: in addition to the dependency on $\varepsilon$,  all of the above iteration complexities also rely on (1) the smoothness parameter, (2) the Lipschitz constant of the Hessians, and (3) the local strong convexity parameter.  These parameters, however, do  not necessarily fall within the desired level.  For instance, in most of the problems mentioned herein (e.g.~phase retrieval, matrix completion),  the nonconvex objective function is not globally smooth, namely, the local gradient Lipschitz constant might grow to infinity as the parameters become unbounded.  Motivated by this, one might need to further assume that optimization is performed within a bounded set of  parameters, as is the case for \cite{mei2016landscape}. Leaving out this boundedness issue, a more severe matter is that these parameters often depend on  the problem size.  For example, recall that in phase retrieval, the smoothness parameter scales with $n$ even within a reasonably small bounded range (see Lemma~\ref{lem:strong-cvx-PR}).  As a consequence, the resulting iteration complexities of the above saddle-escaping algorithm might be significantly larger than, say, $O(\varepsilon^{-7/4})$ for various large-scale problems.

\begin{table*}
\caption{A selective summary of problems with provable nonconvex solutions}\label{tab:summary}
{\scriptsize
\begin{center}
\begin{tabular}{c|c|c|c}
\hline 
 & Global landscape & Two-stage approaches & Convex relaxation \tabularnewline
\hline 
\hline 
	\multirow{2}{*}{matrix sensing} &  \multirow{2}{*}{\cite{bhojanapalli2016global,ge2017no,li2018non}} & vanilla GD \cite{zheng2015convergent,tu2015low}  &  \multirow{2}{*}{\cite{recht2010guaranteed,candes2011tight}} \tabularnewline
  &   & alternating minimization \cite{jain2013low}\tabularnewline
\hline
 &  \multirow{4}{*}{\cite{sun2016geometric,davis2017nonsmooth}} & vanilla GD \cite{candes2015phase,soltanolkotabi2014algorithms,sanghavi2017local,ma2017implicit,zhang2017reshaped,qu2017convolutional,bendory2018non,chen2018gradient,li2018nonconvex}  &  \multirow{4}{*}{\cite{candes2012phaselift,candes2012solving,demanet2012stable,waldspurger2015phase,chen2015exact,cai2015rop,kueng2014low,tropp2015convex}}  \tabularnewline
Phase retrieval &  & regularized GD \cite{chen2015solving,cai2016optimal,zhang2016provable,wang2017solving,wang2017solving,soltanolkotabi2017structured,yang2017misspecified} & \tabularnewline
 and quadratic sensing &  & alternating minimization \cite{gerchberg1972practical,netrapalli2015phase,waldspurger2016phase,zhang2017phase} & \tabularnewline
  &  & approximate message passing \cite{ma2018optimization,ma2018approximate} & \tabularnewline  
\hline 
  \multirow{4}{*}{matrix completion} &  \multirow{4}{*}{\cite{ge2016matrix,ge2017no,chen2017memory,zhang2018primal}} & vanilla gradient descent \cite{ma2017implicit}  &  \multirow{4}{*}{\cite{candes2009exact,CanTao10,fazel2002matrix,Gross2011recovering,Negahban2012restricted,koltchinskii2011nuclear,chen2015incoherence}}  \tabularnewline
  &  & regularized Grassmanian GD \cite{keshavan2010matrix,keshavan2010noisy} \tabularnewline
  & & projected\,/\,regularized GD \cite{sun2016guaranteed,chen2015fast,zheng2016convergence} \tabularnewline
  &  & alternating minimization \cite{jain2013low,hardt2014understanding,hardt2014fast}\tabularnewline
\hline 
	blind deconvolution\,/\,demixing& \multirow{2}{*}{---}  & vanilla GD \cite{ma2017implicit,shi2018demising} &\multirow{2}{*}{\cite{ahmed2014blind,ling2015self} } \tabularnewline
	(subspace model)  &  & regularized GD  \cite{li2016deconvolution,ling2017fast,huang2017blind} & \tabularnewline

\hline
\multirow{3}{*}{robust PCA} & \multirow{3}{*}{\cite{ge2017no} }  & low-rank projection + thresholding \cite{netrapalli2014non} & \multirow{3}{*}{\cite{chandrasekaran2011siam,candes2009robustPCA,ganesh2010dense,chen2013low}}  \tabularnewline
 &  & GD + thresholding \cite{yi2016fast,cherapanamjeri2017nearly}\tabularnewline 
	&  & alternating minimization \cite{gu2016low} \tabularnewline
\hline
spectrum estimation & ---  & projected GD \cite{cai2017spectral} & \cite{chen2013robustSpectralMC,tang2012compressive} \tabularnewline
\hline
\hline
\end{tabular}
\end{center}
}

\end{table*}

%

\section{Concluding remarks}
\label{sec:conclusions}

Given that nonconvex matrix factorization is a rapidly growing field, we expect that many of the results reviewed herein may be improved or superseded. Nonetheless, the core techniques and insights reviewed in this article will continue to play a key role in understanding nonconvex statistical estimation and learning. While a paper of this length is impossible to cover all recent developments, our aim has been to convey the key ideas via  representative examples and to strike a balance between  useful theory and varied algorithms. Table~\ref{tab:summary} summarizes representative references on the canonical problems reviewed herein, with a pointer to their corresponding convex approaches for completeness. It should be noted that most of these results are obtained under random data models, which may not be satisfied in practice (e.g., in matrix completion the sampling patterns can be non-uniform). This means one should not take directly the performance guarantees as is before checking the reasonability of assumptions, and indeed, the performance may drop significantly without carefully modifying the nonconvex approaches on a case-by-case basis to specific problems.

Caution needs to be exercised that the theory derived for nonconvex paradigms may still be sub-optimal in several aspects.  For instance, the number of samples required for recovering\,/\,completing a rank-$r$ matrix often scales at least quadratically with the rank $r$ in the current nonconvex optimization theory, which is outperformed by convex relaxation  (whose sample complexity typically scales linearly with $r$). The theoretical dependency on the condition number of the matrix also falls short of optimality. 

Due to the space limits, we have omitted  developments in several other important aspects,  most notably the statistical guarantees vis-\`a-vis noise. Most of our discussions (and analyses) can be extended without much difficulty to the regime with small to moderate noise; see, for example, \cite{chen2015solving,keshavan2010noisy,chen2015fast,li2016deconvolution,chen2019noisy} for  the stability of gradient methods in phase retrieval,  matrix completion, and blind deconvolution.  Encouragingly, the nonconvex methods not only allow to control the Euclidean estimation errors in a minimax optimal manner, they also provably achieve optimal statistical accuracy in terms of finer error metrics (e.g.~optimal entrywise error control in matrix completion \cite{ma2017implicit}).  In addition, there is an intimate link between convex relaxation and nonconvex optimization,  which allows us to improve the stability guarantees of convex relaxation via the theory of nonconvex optimization; see  \cite{chen2019noisy,chen2019inference} for details.

Furthermore, we have also left out several important nonconvex problems and methods in order to stay focused,  
 including but not limited to (1) blind calibration and finding sparse vectors in a subspace \cite{li2017blind,cambareri2016non,qu2014finding};
	(2) tensor completion, decomposition  and mixture models \cite{yi2014alternating,ge2017optimization,arous2017landscape,li2017convex, cai2019tensor};
	(3) parameter recovery and inverse problems in shallow and deep neural networks \cite{zhong2017recovery,li2017convergence,hand2017global,brutzkus2017globally,fu2018local,fan2019selective};
	(4) analysis of Burer-Monteiro factorization to semidefinite programs \cite{bhojanapalli2016dropping,burer2003nonlinear,bhojanapalli2018smoothed,bandeira2016low}.
%
The interested readers can consult \cite{sunju_ncvx} for an extensive list of further references.  
We would also like to refer the readers to an 
excellent recent monograph  by Jain and Kar \cite{jain2017non} that complements our treatment.  Roughly speaking,  \cite{jain2017non} concentrates on the introduction of generic optimization algorithms, with low-rank factorization (and other learning problems) used as special examples to illustrate the generic results.  In contrast, the current article focuses on unveiling deep statistical and algorithmic insights specific to nonconvex low-rank factorization.

Finally, we conclude this paper by singling out several exciting avenues for future research:
\begin{itemize}
\item investigate the convergence rates of randomly initialized gradient methods for problems without Gaussian-type measurements (e.g.~matrix completion and blind deconvolution);  
\item characterize generic landscape properties that enable fast convergence of gradient methods from random initialization;
\item relax the stringent assumptions on the statistical models underlying the data; for example, a few recent works studied nonconvex phase retrieval under more physically-meaningful measurement models \cite{bendory2018non,qu2017convolutional};
\item develop robust and scalable nonconvex methods that can handle distributed data samples with strong statistical guarantees;
\item study the capability of other commonly used optimization algorithms (e.g.~alternating minimization) in escaping saddle points; 
\item characterize the optimization landscape of constrained nonconvex statistical estimation problems like non-negative low-rank matrix factorization (the equality-constrained case has recently been explored by \cite{8682568});
\item develop a unified and powerful theory that can automatically accommodate the problems considered herein and beyond. 
\end{itemize}

\section*{Acknowledgment}

The authors thank the editors of IEEE Transactions on Signal Processing, Prof.~Pier Luigi Dragotti and Prof.~Jarvis Haupt, for handling this invited overview article, and thank the editorial board for useful feedbacks of the white paper. The authors thank Cong Ma, Botao Hao, Tian Tong and Laixi Shi for proofreading an early version of the manuscript.

\appendix


\section{Proof of Theorem \ref{thm:convergence-rank1-sensing}}\label{appendix:proof-thm:convergence-rank1-sensing}
 
	To begin with, simple calculation reveals that for any $\bm{z},\bm{x}\in\mathbb{R}^{n}$, the Hessian obeys (see also \cite[Lemma 4.3]{bhojanapalli2016global})
\begin{align*}
	&\bm{z}^{\top}\nabla^{2}f(\bm{x})\bm{z}= \mylinebreak
	\myalign \myquad \frac{1}{m}\sum_{i=1}^{m}\langle\bm{A}_{i},\bm{x}\bm{x}^{\top}-\bm{x}_{\star}\bm{x}_{\star}^{\top}\rangle(\bm{z}^{\top}\bm{A}_{i}\bm{z})+2(\bm{z}^{\top}\bm{A}_{i}\bm{x})^{2}.
\end{align*}

With the notation  \eqref{eq:defn-A-sensing} at hand, 
we can rewrite
\begin{align}
&\bm{z}^{\top}\nabla^{2}f(\bm{x})\bm{z}  =\big\langle\mathcal{A}(\bm{x}\bm{x}^{\top}-\bm{x}_{\star}\bm{x}_{\star}^{\top}),\mathcal{A}(\bm{z}\bm{z}^{\top})\big\rangle \mynonumber\mylinebreak
	\myalign \myquad+0.5\big\langle\mathcal{A}(\bm{z}\bm{x}^{\top}+\bm{x}\bm{z}^{\top}),\mathcal{A}(\bm{z}\bm{x}^{\top}+\bm{x}\bm{z}^{\top})\big\rangle,  \label{eq:zHessz-sensing}
\end{align}
where the last line uses the symmetry of $\bm{A}_i$. 

To establish local strong convexity and smoothness, we need to control $\bm{z}^{\top}\nabla^{2}f(\bm{x})\bm{z}$ for all $\bm{z}$. 	
The key ingredient is to show that: $\bm{z}^{\top}\nabla^{2}f(\bm{x})\bm{z}$ is not too far away from 
\begin{equation}
	\label{eq:defn-gxz-sensing}
	g(\bm{x},\bm{z}) := \big\langle\bm{x}\bm{x}^{\top}-\bm{x}_{\star}\bm{x}_{\star}^{\top},\bm{z}\bm{z}^{\top}\big\rangle + 
	0.5 \| \bm{z}\bm{x}^{\top}+\bm{x}\bm{z}^{\top} \|_{\mathrm{F}}^2,
\end{equation}
a function that can be easily shown to be locally strongly convex and smooth. To this end, we resort to the RIP (see Definition \ref{defn:RIPs}). 
When $\mathcal{A}$ satisfies 4-RIP, Lemma \ref{lemmq:RIP-cross}
 indicates that 
\begin{align}
  \big|\big\langle\mathcal{A}(\bm{x}\bm{x}^{\top}-\bm{x}_{\star}\bm{x}_{\star}^{\top}),\mathcal{A}(\bm{z}\bm{z}^{\top})\big\rangle-\big\langle\bm{x}\bm{x}^{\top}-\bm{x}_{\star}\bm{x}_{\star}^{\top},\bm{z}\bm{z}^{\top}\big\rangle\big| 
 & \leq \delta_4 \|\bm{x}\bm{x}^{\top}-\bm{x}_{\star}\bm{x}_{\star}^{\top}\|_{\mathrm{F}}\|\bm{z}\bm{z}^{\top}\|_{\mathrm{F}} \nonumber\\
	& \leq 3\delta_4 \|\bm{x}_{\star}\|_2^2 \|\bm{z}\|_2^2, \label{eq:approx-sensing-1}
\end{align}
where the last line holds if $\|\bm{x}-\bm{x}_{\star}\|_2\leq \|\bm{x}_{\star}\|_2$. Similarly, 
\begin{align}
	 \big| \big\langle\mathcal{A}(\bm{z}\bm{x}^{\top}+\bm{x}\bm{z}^{\top}),\mathcal{A}(\bm{z}\bm{x}^{\top}+\bm{x}\bm{z}^{\top})\big\rangle - \| \bm{z}\bm{x}^{\top}+\bm{x}\bm{z}^{\top} \|_{\mathrm{F}}^2 \big| 
	 \leq 4\delta_{4}\|\bm{x}\|_{2}^{2}\|\bm{z}\|_2^2 \leq 16\delta_{4}\|\bm{x}_{\star}\|_{2}^{2}\|\bm{z}\|_2^2.  \label{eq:approx-sensing-2}
\end{align}
%
%
%
	As a result, if  $\delta_4$ is  small enough, then putting the above results together implies that: $\bm{z}^{\top}\nabla^{2}f(\bm{x})\bm{z}$ is sufficiently close to $g(\bm{x},\bm{z})$. A little algebra then gives  (see Appendix \ref{appendix:proof-claim-Hessian-rank1-sensing})
\begin{equation}\label{eq:claim-Hessian-rank1-sensing}
	0.25\|\bm{z}\|_2^2 \|\bm{x}_{\star}\|_2^2 \leq\bm{z}^{\top}\nabla^{2}f(\bm{x})\bm{z}\leq 3 \|\bm{z}\|_2^2 \|\bm{x}_{\star}\|_2^2
\end{equation}
for all $\bm{z}$, which  provides bounds on local strong convexity and smoothness parameters. 
Applying Lemma \ref{lem:GD-convergence}  thus establishes the theorem.

\section{Proof of Claim \eqref{eq:claim-Hessian-rank1-sensing}}
\label{appendix:proof-claim-Hessian-rank1-sensing}

Without loss of generality, we  assume that 
$\|\bm{z}\|_2=\|\bm{x}_{\star}\|_2=1$. When $\|\bm{x}-\bm{x}_{\star}\|_{2}\leq \|\bm{x}_{\star}\|_{2}$,
some elementary algebra gives
\[
	|\langle \bm{x}\bm{x}^{\top}-\bm{x}_{\star}\bm{x}_{\star}^{\top}, \bm{z}\bm{z}^{\top} \rangle| \leq \| \bm{x}\bm{x}^{\top}-\bm{x}_{\star}\bm{x}_{\star}^{\top} \|_{\mathrm{F}} \leq 3 \|\bm{x}-\bm{x}_{\star}\|_2.
\]
In addition, 
\begin{align*}
\|\bm{x}\bm{z}^{\top}+\bm{z}\bm{x}^{\top}\|_{\mathrm{F}}^{2} & \leq2\|\bm{x}\bm{z}^{\top}\|_{\mathrm{F}}^{2}+2\|\bm{z}\bm{x}^{\top}\|_{\mathrm{F}}^{2}=4\|\bm{x}\|_{2}^{2},\\
\|\bm{x}\bm{z}^{\top}+\bm{z}\bm{x}^{\top}\|_{\mathrm{F}}^{2} & \geq\|\bm{x}\bm{z}^{\top}\|_{\mathrm{F}}^{2}+\|\bm{z}\bm{x}^{\top}\|_{\mathrm{F}}^{2}=2\|\bm{x}\|_{2}^{2}.
\end{align*}
These in turn yield
\begin{align*}
\big\langle\bm{x}\bm{x}^{\top}-\bm{x}_{\star}\bm{x}_{\star}^{\top},\bm{z}\bm{z}^{\top}\big\rangle
+ 0.5 \|\bm{x}\bm{z}^{\top}+\bm{z}\bm{x}^{\top}\|_{\mathrm{F}}^{2}
& \leq 2\|\bm{x}\|_{2}^{2} +3\|\bm{x}-\bm{x}_{\star}\|_{2}, \\
  \big\langle\bm{x}\bm{x}^{\top}-\bm{x}_{\star}\bm{x}_{\star}^{\top},\bm{z}\bm{z}^{\top}\big\rangle
+ 0.5\|\bm{x}\bm{z}^{\top}+\bm{z}\bm{x}^{\top}\|_{\mathrm{F}}^{2}
& \geq \|\bm{x}\|_{2}^{2} -3\|\bm{x}-\bm{x}_{\star}\|_{2}.
\end{align*} 
Putting these together with \eqref{eq:zHessz-sensing}, \eqref{eq:approx-sensing-1} and \eqref{eq:approx-sensing-2}, we have
\begin{align*}
\bm{z}^{\top}\nabla^{2}f(\bm{x})\bm{z} & \leq2\|\bm{x}\|_{2}^{2}+3\|\bm{x}-\bm{x}_{\star}\|_{2}+11\delta_{4},\\
\bm{z}^{\top}\nabla^{2}f(\bm{x})\bm{z} & \geq\|\bm{x}\|_{2}^{2}-3\|\bm{x}-\bm{x}_{\star}\|_{2}-11\delta_{4}.
\end{align*}	
Recognizing that $|\|\bm{x}\|_2^2-1|\leq 3\|\bm{x}-\bm{x}_{\star}\|_{2}$, we further reach
\begin{align*}
\bm{z}^{\top}\nabla^{2}f(\bm{x})\bm{z} & \leq 2 + 9\|\bm{x}-\bm{x}_{\star}\|_{2} + 11\delta_{4} ,\\
\bm{z}^{\top}\nabla^{2}f(\bm{x})\bm{z} & \geq 1 - 6\|\bm{x}-\bm{x}_{\star}\|_{2} - 11\delta_{4} .
\end{align*}
If $\delta_{4}\leq1/44$ and $\|\bm{x}-\bm{x}_{\star}\|_{2}\leq1/12$,
then we arrive at
\[
	0.25 \leq \bm{z}^{\top} \nabla^2 f(\bm{x}) \bm{z} \leq 3. 
\]

\section{Modified strong convexity for \eqref{eq:pop-level-factorization}}
\label{appendix-strong-cvx-rotation-pop}


Here, we demonstrate that when $\bm{X}$ is sufficiently close to $\bm{X}_{\star}$, then the objective function of $f_{\infty}(\cdot)$ of \eqref{eq:pop-level-factorization} exhibits the modified form of strong convexity in the form of \eqref{eq:strong-cvx-rotate-1point}.

Set $\bm{V}=\bm{Z}\bm{H}_{\bm{Z}}-\bm{X}_{\star}$. In view of (\ref{eq:Hess-infty}), it suffices to show
\begin{align*}
g(\bm{X},\bm{V}) & :=0.5\|\bm{X}\bm{V}^{\top}+\bm{V}\bm{X}^{\top}\|_{\mathrm{F}}^{2} \mylinebreak
 \myalign \myquad+\langle\bm{X}\bm{X}^{\top}-\bm{X}_{\star}\bm{X}_{\star}^{\top},\bm{V}\bm{V}^{\top}\rangle > 0
\end{align*}
for  all $\bm{X}$ sufficiently close
to $\bm{X}_{\star}$ and all $\bm{Z}$. To this end, we first observe a simple perturbation bound (the proof is omitted)
\[
	\big| g(\bm{X},\bm{V})-g(\bm{X}_{\star},\bm{V}) \big| \leq c_{1}\|\bm{X}-\bm{X}_{\star}\|_{\mathrm{F}}\cdot\|\bm{X}_{\star}\|_{\mathrm{F}}\cdot\|\bm{V}\|_{\mathrm{F}}^{2}
\]
for some universal constant $c_{1}>0$, provided that $\|\bm{X}-\bm{X}_{\star}\|_{\mathrm{F}}$
is small enough. We then turn attention to $g(\bm{X}_{\star},\bm{V})$:
\begin{align*} 
 \myalign g(\bm{X}_{\star},\bm{V}) \myaligninv = 0.5\|\bm{X}_{\star}\bm{V}^{\top}+\bm{V}\bm{X}_{\star}^{\top}\|_{\mathrm{F}}^{2}\\
 & \myquad=\|\bm{X}_{\star}\bm{V}^{\top}\|_{\mathrm{F}}^{2}+\langle\bm{X}_{\star}(\bm{H}_{\bm{Z}}^{\top}\bm{Z}^{\top}-\bm{X}_{\star}^{\top}),(\bm{Z}\bm{H}_{\bm{Z}}-\bm{X}_{\star})\bm{X}_{\star}^{\top}\rangle\\
 & \myquad=\langle\bm{V}^{\top}\bm{V},\bm{X}_{\star}^{\top}\bm{X}_{\star}\rangle+\mathsf{Tr}\big(\bm{X}_{\star}^{\top}\bm{Z}\bm{H}_{\bm{Z}}\bm{X}_{\star}^{\top}\bm{Z}\bm{H}_{\bm{Z}}\big)
   +2\mathsf{Tr}\big(\bm{X}_{\star}^{\top}\bm{Z}\bm{H}_{\bm{Z}}\bm{X}_{\star}^{\top}\bm{X}_{\star}\big)+\mathsf{Tr}\big(\bm{X}_{\star}^{\top}\bm{X}_{\star}\bm{X}_{\star}^{\top}\bm{X}_{\star}\big)\\
 & \myquad\geq\sigma_{r}(\bm{X}_{\star}^{\top}\bm{X}_{\star})\|\bm{V}^{\top}\bm{V}\|_{\mathrm{F}}.
\end{align*}
The last line holds by recognizing that  $\bm{X}_{\star}^{\top}\bm{Z}\bm{H}_{\bm{Z}}\succeq \bm{0}$ 
\cite[Theorem 2]{ten1977orthogonal}, which implies that $\mathsf{Tr}(\bm{X}_{\star}^{\top}\bm{Z}\bm{H}_{\bm{Z}}\bm{X}_{\star}^{\top}\bm{Z}\bm{H}_{\bm{Z}})\geq 0$ and $\mathsf{Tr}\big(\bm{X}_{\star}^{\top}\bm{Z}\bm{H}_{\bm{Z}}\bm{X}_{\star}^{\top}\bm{X}_{\star}\big)\geq0$. Thus, 
\begin{align*}
g(\bm{X},\bm{V}) & \geq g(\bm{X}_{\star},\bm{V})-\big|g(\bm{X},\bm{V})-g(\bm{X}_{\star},\bm{V})\big| \geq0.5\sigma_{r}(\bm{X}_{\star}^{\top}\bm{X}_{\star})\|\bm{V}^{\top}\bm{V}\|_{\mathrm{F}}
\end{align*}
as long as $\|\bm{X}-\bm{X}_{\star}\|_{\mathrm{F}}\leq\frac{\sigma_{r}(\bm{X}_{\star}^{\top}\bm{X}_{\star})\|\bm{V}^{\top}\bm{V}\|_{\mathrm{F}}^{2}}{2c_{1}\|\bm{X}_{\star}\|_{\mathrm{F}}\cdot\|\bm{V}\|_{\mathrm{F}}}$. 
In summary, 
\begin{align*}
	&\mathsf{vec}(\bm{Z}\bm{H}_{\bm{Z}}-\bm{X}_{\star})^{\top} \nabla^2 f (\bm{X}) \,\mathsf{vec}(\bm{Z}\bm{H}_{\bm{Z}}-\bm{X}_{\star})  
	\geq \frac{\sigma_{r}^2(\bm{X}_{\star})\|\bm{V}^{\top}\bm{V}\|_{\mathrm{F}}}{2\|\bm{V}\|_{\mathrm{F}}^2}  \|\bm{V}\|_{\mathrm{F}}^2
  \geq \frac{\sigma_{r}^2(\bm{X}_{\star})}{2\sqrt{r}}  \|\bm{Z}\bm{H}_{\bm{Z}}-\bm{X}_{\star}\|_{\mathrm{F}}^2. 
\end{align*}

\section{Proof of Theorem~\ref{thm:spectral-sensing}} \label{appendix:proof-lem:perturbation-sensing}

We start with the rank-$r$ PSD case, where we denote by $\bX_0$ the initial estimate (i.e.~$\bX_0=\bL_0=\bR_0$ for this case) and $\bX_{\star}$  the ground truth. Observe that
\begin{align}
  \| \bX_0\bX_0^{\top} - \bM_{\star}  \|  & \leq \| \bX_0\bX_0^{\top} - \bY \| + \| \bX_{\star}\bX_{\star}^{\top}  - \bY  \| \leq  2\| \bM_{\star}  - \bY  \| \nonumber  \\
& \leq 2 \delta_{2r} \sqrt{r}\| \bm{M}_{\star} \|, \label{eq:spectrl_norm_error}
\end{align}
where the second line follows since  $\bX_{0}\bX_0^{\top}$ is the best rank-$r$ approximation of $\bY$ (and hence $\| \bX_0\bX_0^{\top} - \bY \| \leq \| \bX_{\star}\bX_{\star}^{\top} - \bY \|$), and the
last inequality follows from Lemma \ref{lem:perturbation-sensing}. A useful lemma from \cite[Lemma 5.4]{tu2015low} allows us to directly upper bound $\mathsf{dist}^2(\bX_0,\bX_{\star})$ by the Euclidean distance between their low-rank counterparts.
\begin{lemma}[\cite{tu2015low}]
For any $\bU,\bX\in\mathbb{R}^{n\times r}$, we have
\begin{equation} 
	\label{dist_psd_lowrank}
	\mathsf{dist}^2(\bU,\bX)  \leq \frac{1}{2(\sqrt{2}-1)\sigma_r^2(\bX)} \| \bm{U}\bm{U}^\top - \bm{X}\bm{X}^\top\|_{\mathrm{F}}^2.
\end{equation}
\end{lemma}

As an immediate consequence of the above lemma, we get
\begin{align*}
\mathsf{dist}^2(\bX_0,\bX_{\star})  &\leq  \frac{1}{2(\sqrt{2}-1)\sigma_r(\bM_{\star})} \| \bm{X}_0\bm{X}_0^\top - \bm{M}_{\star}\|_{\mathrm{F}}^2 \\
& \leq  \frac{r}{(\sqrt{2}-1)\sigma_r(\bM_{\star})} \| \bm{X}_0\bm{X}_0^\top - \bm{M}_{\star}\|^2 \\
& \leq \frac{4\delta_{2r}^2 r^2 \|\bM_{\star}\|^2}{(\sqrt{2}-1)\sigma_r(\bM_{\star})} ,
\end{align*}
where the last line follows from \eqref{eq:spectrl_norm_error}. Therefore, as long as $\delta_{2r} \ll \sqrt{\zeta} / (r \kappa)$, we have
$$ \mathsf{dist}^2(\bX_0,\bX_{\star})  \leq \zeta \sigma_r(\bM_{\star}).$$
In fact, \cite{oymak2018sharp} shows a stronger consequence of the RIP: one can improve the left-hand-side of \eqref{eq:spectrl_norm_error} to $\| \bX_0\bX_0^{\top} - \bM_{\star}  \|_{\mathrm{F}} $, which allows relaxing the requirement on the RIP constant to $\delta_{2r}\ll  \sqrt{\zeta} / (\sqrt{r} \kappa)$ \cite{tu2015low}. The asymmetric case can be proved in a similar fashion.


\section{Proof of Theorem \ref{thm:sensing-global} (the rank-1 case)}
\label{appendix:sec-proof-theorem-sensing-global}

Before proceeding, we first single out an immediate consequence of the first-order optimality condition \eqref{eq:1st-order-optimality-RIP} that will prove useful. Specifically, for any critical point $\bx$,
%
\begin{equation}
\Big\|\big(\bm{x}\bm{x}^{\top}-\bm{x}_{\star}\bm{x}_{\star}^{\top}\big)\bm{x}\bm{x}^{\top}\Big\|_{\mathrm{F}}\leq\delta_{2}\big\|\bm{x}\bm{x}^{\top}-\bm{x}_{\star}\bm{x}_{\star}^{\top}\big\|_{\mathrm{F}}\|\bm{x}\|_{2}^{2}.\label{eq:fact1-RIP}
\end{equation}
This fact basically says that any critical point of $f(\cdot)$ stays
very close to the truth in the subpace spanned by this point.

To verify \eqref{eq:goal-UB-sensing}, we observe the identity
\begin{align*}
  (\bm{x}-\bm{x}_{\star})^{\top}\nabla^{2}f(\bm{x})\,(\bm{x}-\bm{x}_{\star}) 
	& =\frac{1}{m}\sum_{i=1}^{m}2\langle\bm{A}_{i},\bm{x}(\bm{x}-\bm{x}_{\star})^{\top}\rangle^{2} +\langle\bm{A}_{i},\bm{x}\bm{x}^{\top}-\bm{x}_{\star}\bm{x}_{\star}^{\top}\rangle\langle\bm{A}_{i},(\bm{x}-\bm{x}_{\star})(\bm{x}-\bm{x}_{\star})^{\top}\rangle\\
 &  =\frac{1}{m}\sum_{i=1}^{m}2\langle\bm{A}_{i},\bm{x}(\bm{x}-\bm{x}_{\star})^{\top}\rangle^{2}-\langle\bm{A}_{i},\bm{x}\bm{x}^{\top}-\bm{x}_{\star}\bm{x}_{\star}^{\top}\rangle^{2},
\end{align*}
where the last line follows from (\ref{eq:1st-order-optimality-RIP})
with a little algebra. This combined with (\ref{eq:defn-RIP2r}) gives
\begin{align}
	& (\bm{x}-\bm{x}_{\star})^{\top}\nabla^{2}f(\bm{x})\,(\bm{x}-\bm{x}_{\star}) \mynonumber \mylinebreak
 	\myalign \leq2(1+\delta_{2})\big\|\bm{x}(\bm{x}-\bm{x}_{\star})^{\top}\big\|_{\mathrm{F}}^{2}-(1-\delta_{2})\big\|\bm{x}\bm{x}^{\top}-\bm{x}_{\star}\bm{x}_{\star}^{\top}\big\|_{\mathrm{F}}^{2}.\label{eq:UB1-Hess}
\end{align}

We then need to show that \eqref{eq:UB1-Hess} is negative. To this end, we
quote an elementary algebraic inequality \cite[Lemma 4.4]{bhojanapalli2016global}
\begin{align*}
 & \big\|\bm{x}(\bm{x}-\bm{x}_{\star})^{\top}\big\|_{\mathrm{F}}^{2} \mylinebreak
 \myalign \leq\frac{1}{8}\big\|\bm{x}\bm{x}^{\top}-\bm{x}_{\star}\bm{x}_{\star}^{\top}\big\|_{\mathrm{F}}^{2}+\frac{34}{8\|\bm{x}\|_{2}^{2}}\big\|\big(\bm{x}\bm{x}^{\top}-\bm{x}_{\star}\bm{x}_{\star}^{\top}\big)\bm{x}\bm{x}^{\top}\big\|_{\mathrm{F}}^{2}.
\end{align*}
This taken collectively with (\ref{eq:UB1-Hess}) and (\ref{eq:fact1-RIP})
yields
\begin{align*}
 (\bm{x}-\bm{x}_{\star})^{\top}\nabla^{2}f(\bm{x})\,(\bm{x}-\bm{x}_{\star})  
	& \leq\left(-1+\delta_{2}+\frac{1+\delta_{2}}{4}\right)\big\|\bm{x}\bm{x}^{\top}-\bm{x}_{\star}\bm{x}_{\star}^{\top}\big\|_{\mathrm{F}}^{2} 
  +\frac{34(1+\delta_{2})}{4\|\bm{x}\|_{2}^{2}}\big\|\big(\bm{x}\bm{x}^{\top}-\bm{x}_{\star}\bm{x}_{\star}^{\top}\big)\bm{x}\bm{x}^{\top}\big\|_{\mathrm{F}}^{2}\\
 & \leq-\left(1-\delta_{2}-\frac{1+\delta_{2}}{4}-\frac{34(1+\delta_{2})\delta^2_{2}}{4}\right)\big\|\bm{x}\bm{x}^{\top}-\bm{x}_{\star}\bm{x}_{\star}^{\top}\big\|_{\mathrm{F}}^{2},
\end{align*}
which is negative unless $\bm{x}=\bm{x}_{\star}$. This concludes
the proof.

\begin{proof}[Proof of Claim \eqref{eq:fact1-RIP}]
To see why (\ref{eq:fact1-RIP}) holds, observe that
\[
\frac{1}{m}\sum_{i=1}^{m}\langle\bm{A}_{i},\bm{x}\bm{x}^{\top}-\bm{x}_{\star}\bm{x}_{\star}^{\top}\rangle\bm{A}_{i}\bm{x}\bm{x}^{\top}=\nabla f(\bm{x})\bm{x}^{\top}=\bm{0}.
\]
Therefore, for any $\bm{Z}$, using (\ref{eq:defn-RIP2r}) we have
\begin{align*}
 \Big|\Big\langle\big(\bm{x}\bm{x}^{\top}-\bm{x}_{\star}\bm{x}_{\star}^{\top}\big)\bm{x}\bm{x}^{\top},\bm{Z}\Big\rangle\Big| 
 & \leq \Big| \frac{1}{m}\sum_{i=1}^{m}\langle\bm{A}_{i},\bm{x}\bm{x}^{\top}-\bm{x}_{\star}\bm{x}_{\star}^{\top}\rangle\langle\bm{A}_{i}\bm{x}\bm{x}^{\top},\bm{Z}\rangle \Big| +\delta_{2}\big\|\big(\bm{x}\bm{x}^{\top}-\bm{x}_{\star}\bm{x}_{\star}^{\top}\big)\bm{x}\bm{x}^{\top}\big\|_{\mathrm{F}}\|\bm{Z}\|_{\mathrm{F}}\\
 & =\delta_{2}\big\|\bm{x}\bm{x}^{\top}-\bm{x}_{\star}\bm{x}_{\star}^{\top}\big\|_{\mathrm{F}}\|\bm{x}\|_{2}^{2}\|\bm{Z}\|_{\mathrm{F}}.
\end{align*}
Since $\|\bm{M}\|_{\mathrm{F}}=\sup_{\bm{Z}}\frac{|\langle\bm{M},\bm{Z}\rangle|}{\|\bm{Z}\|_{\mathrm{F}}}$,
we arrive at (\ref{eq:fact1-RIP}). 
\end{proof}

\bibliographystyle{IEEEtran}%
\bibliography{bibfileNonconvex_TSP_updated}

\begin{thebibliography}{100}
\providecommand{\url}[1]{#1}
\csname url@samestyle\endcsname
\providecommand{\newblock}{\relax}
\providecommand{\bibinfo}[2]{#2}
\providecommand{\BIBentrySTDinterwordspacing}{\spaceskip=0pt\relax}
\providecommand{\BIBentryALTinterwordstretchfactor}{4}
\providecommand{\BIBentryALTinterwordspacing}{\spaceskip=\fontdimen2\font plus
\BIBentryALTinterwordstretchfactor\fontdimen3\font minus
  \fontdimen4\font\relax}
\providecommand{\BIBforeignlanguage}[2]{{%
\expandafter\ifx\csname l@#1\endcsname\relax
\typeout{** WARNING: IEEEtran.bst: No hyphenation pattern has been}%
\typeout{** loaded for the language `#1'. Using the pattern for}%
\typeout{** the default language instead.}%
\else
\language=\csname l@#1\endcsname
\fi
#2}}
\providecommand{\BIBdecl}{\relax}
\BIBdecl

\bibitem{candes2009exact}
E.~J. Cand{\`e}s and B.~Recht, ``Exact matrix completion via convex
  optimization,'' \emph{Foundations of Computational Mathematics}, vol.~9,
  no.~6, pp. 717--772, 2009.

\bibitem{davenport2016overview}
M.~A. Davenport and J.~Romberg, ``An overview of low-rank matrix recovery from
  incomplete observations,'' \emph{IEEE Journal of Selected Topics in Signal
  Processing}, vol.~10, no.~4, pp. 608--622, 2016.

\bibitem{chen2018harnessing}
Y.~Chen and Y.~Chi, ``Harnessing structures in big data via guaranteed low-rank
  matrix estimation: Recent theory and fast algorithms via convex and nonconvex
  optimization,'' \emph{IEEE Signal Processing Magazine}, vol.~35, no.~4, pp.
  14 -- 31, 2018.

\bibitem{shechtman2015phase}
Y.~Shechtman, Y.~C. Eldar, O.~Cohen, H.~N. Chapman, J.~Miao, and M.~Segev,
  ``Phase retrieval with application to optical imaging: a contemporary
  overview,'' \emph{IEEE signal processing magazine}, vol.~32, no.~3, pp.
  87--109, 2015.

\bibitem{ahmed2014blind}
A.~Ahmed, B.~Recht, and J.~Romberg, ``Blind deconvolution using convex
  programming,'' \emph{IEEE Transactions on Information Theory}, vol.~60,
  no.~3, pp. 1711--1732, 2014.

\bibitem{ling2015self}
S.~Ling and T.~Strohmer, ``Self-calibration and biconvex compressive sensing,''
  \emph{Inverse Problems}, vol.~31, no.~11, p. 115002, 2015.

\bibitem{chandrasekaran2011siam}
V.~Chandrasekaran, S.~Sanghavi, P.~Parrilo, and A.~Willsky, ``Rank-sparsity
  incoherence for matrix decomposition,'' \emph{SIAM {Journal} on
  {Optimization}}, vol.~21, no.~2, pp. 572--596, 2011.

\bibitem{candes2009robustPCA}
E.~J. Cand{\`e}s, X.~Li, Y.~Ma, and J.~Wright, ``Robust principal component
  analysis?'' \emph{Journal of the ACM}, vol.~58, no.~3, pp. 11:1--11:37, 2011.

\bibitem{singer2011angular}
A.~Singer, ``Angular synchronization by eigenvectors and semidefinite
  programming,'' \emph{Applied and computational harmonic analysis}, vol.~30,
  no.~1, pp. 20--36, 2011.

\bibitem{chen2016projected}
Y.~Chen and E.~Cand\`es, ``The projected power method: An efficient algorithm
  for joint alignment from pairwise differences,'' \emph{Communications on Pure
  and Applied Mathematics}, vol.~71, no.~8, pp. 1648--1714, 2018.

\bibitem{fazel2002matrix}
M.~Fazel, ``Matrix rank minimization with applications,'' Ph.D. dissertation,
  Stanford University, 2002.

\bibitem{recht2010guaranteed}
B.~Recht, M.~Fazel, and P.~A. Parrilo, ``Guaranteed minimum-rank solutions of
  linear matrix equations via nuclear norm minimization,'' \emph{SIAM review},
  vol.~52, no.~3, pp. 471--501, 2010.

\bibitem{CanTao10}
E.~Cand\`es and T.~Tao, ``The power of convex relaxation: Near-optimal matrix
  completion,'' \emph{IEEE Transactions on Information Theory}, vol.~56, no.~5,
  pp. 2053 --2080, May 2010.

\bibitem{koren2009matrix}
Y.~Koren, R.~Bell, and C.~Volinsky, ``Matrix factorization techniques for
  recommender systems,'' \emph{Computer}, no.~8, pp. 30--37, 2009.

\bibitem{chen2004recovering}
P.~Chen and D.~Suter, ``Recovering the missing components in a large noisy
  low-rank matrix: Application to {SFM},'' \emph{IEEE transactions on pattern
  analysis and machine intelligence}, vol.~26, no.~8, pp. 1051--1063, 2004.

\bibitem{keshavan2010matrix}
R.~H. Keshavan, A.~Montanari, and S.~Oh, ``Matrix completion from a few
  entries,'' \emph{IEEE Transactions on Information Theory}, vol.~56, no.~6,
  pp. 2980 --2998, June 2010.

\bibitem{jain2013low}
P.~Jain, P.~Netrapalli, and S.~Sanghavi, ``Low-rank matrix completion using
  alternating minimization,'' in \emph{Proceedings of the forty-fifth annual
  ACM symposium on Theory of computing}.\hskip 1em plus 0.5em minus 0.4em\relax
  ACM, 2013, pp. 665--674.

\bibitem{candes2015phase}
E.~Cand\`es, X.~Li, and M.~Soltanolkotabi, ``Phase retrieval via {W}irtinger
  flow: Theory and algorithms,'' \emph{Information Theory, IEEE Transactions
  on}, vol.~61, no.~4, pp. 1985--2007, 2015.

\bibitem{sun2016guaranteed}
R.~Sun and Z.-Q. Luo, ``Guaranteed matrix completion via non-convex
  factorization,'' \emph{IEEE Transactions on Information Theory}, vol.~62,
  no.~11, pp. 6535--6579, 2016.

\bibitem{chen2015solving}
Y.~Chen and E.~Cand\`es, ``Solving random quadratic systems of equations is
  nearly as easy as solving linear systems,'' \emph{Communications on Pure and
  Applied Mathematics}, vol.~70, no.~5, pp. 822--883, 2017.

\bibitem{chen2015fast}
Y.~Chen and M.~J. Wainwright, ``Fast low-rank estimation by projected gradient
  descent: General statistical and algorithmic guarantees,'' \emph{arXiv
  preprint arXiv:1509.03025}, 2015.

\bibitem{li2016deconvolution}
X.~Li, S.~Ling, T.~Strohmer, and K.~Wei, ``Rapid, robust, and reliable blind
  deconvolution via nonconvex optimization,'' \emph{Applied and Computational
  Harmonic Analysis}, 2018.

\bibitem{tu2015low}
S.~Tu, R.~Boczar, M.~Simchowitz, M.~Soltanolkotabi, and B.~Recht, ``Low-rank
  solutions of linear matrix equations via {P}rocrustes flow,'' in
  \emph{International Conference Machine Learning}, 2016, pp. 964--973.

\bibitem{zhang2016provable}
H.~Zhang, Y.~Chi, and Y.~Liang, ``Provable non-convex phase retrieval with
  outliers: Median truncated {W}irtinger flow,'' in \emph{International
  conference on machine learning}, 2016, pp. 1022--1031.

\bibitem{netrapalli2014non}
P.~Netrapalli, U.~Niranjan, S.~Sanghavi, A.~Anandkumar, and P.~Jain,
  ``Non-convex robust {PCA},'' in \emph{Advances in Neural Information
  Processing Systems}, 2014, pp. 1107--1115.

\bibitem{ma2017implicit}
C.~Ma, K.~Wang, Y.~Chi, and Y.~Chen, ``Implicit regularization in nonconvex
  statistical estimation: Gradient descent converges linearly for phase
  retrieval, matrix completion and blind deconvolution,'' \emph{arXiv preprint
  arXiv:1711.10467, accepted to Foundations of Computational Mathematics},
  2017.

\bibitem{sun2016geometric}
J.~Sun, Q.~Qu, and J.~Wright, ``A geometric analysis of phase retrieval,''
  \emph{Foundations of Computational Mathematics}, pp. 1--68, 2018.

\bibitem{sun2015complete}
------, ``Complete dictionary recovery using nonconvex optimization,'' in
  \emph{Proceedings of the 32nd International Conference on Machine Learning},
  2015, pp. 2351--2360.

\bibitem{ge2016matrix}
R.~Ge, J.~D. Lee, and T.~Ma, ``Matrix completion has no spurious local
  minimum,'' in \emph{Advances in Neural Information Processing Systems}, 2016,
  pp. 2973--2981.

\bibitem{bhojanapalli2016global}
S.~Bhojanapalli, B.~Neyshabur, and N.~Srebro, ``Global optimality of local
  search for low rank matrix recovery,'' in \emph{Advances in Neural
  Information Processing Systems}, 2016, pp. 3873--3881.

\bibitem{li2018non}
Q.~Li, Z.~Zhu, and G.~Tang, ``The non-convex geometry of low-rank matrix
  optimization,'' \emph{Information and Inference: A Journal of the IMA}, 2018,
  in press.

\bibitem{bubeck2015convex}
S.~Bubeck, ``Convex optimization: Algorithms and complexity,''
  \emph{Foundations and Trends{\textregistered} in Machine Learning}, vol.~8,
  no. 3-4, pp. 231--357, 2015.

\bibitem{lang1993real}
S.~Lang, ``Real and functional analysis,'' \emph{Springer-Verlag, New York,},
  vol.~10, pp. 11--13, 1993.

\bibitem{beck2017first}
A.~Beck, \emph{First-Order Methods in Optimization}.\hskip 1em plus 0.5em minus
  0.4em\relax SIAM, 2017, vol.~25.

\bibitem{nesterov2013introductory}
Y.~Nesterov, \emph{Introductory lectures on convex optimization: A basic
  course}.\hskip 1em plus 0.5em minus 0.4em\relax Springer Science \& Business
  Media, 2013, vol.~87.

\bibitem{baldi1989neural}
P.~Baldi and K.~Hornik, ``Neural networks and principal component analysis:
  Learning from examples without local minima,'' \emph{Neural networks},
  vol.~2, no.~1, pp. 53--58, 1989.

\bibitem{baldi1995learning}
P.~F. Baldi and K.~Hornik, ``Learning in linear neural networks: A survey,''
  \emph{IEEE Transactions on neural networks}, vol.~6, no.~4, pp. 837--858,
  1995.

\bibitem{yang1995projection}
B.~Yang, ``Projection approximation subspace tracking,'' \emph{IEEE
  Transactions on Signal processing}, vol.~43, no.~1, pp. 95--107, 1995.

\bibitem{li2016symmetry}
X.~Li, Z.~Wang, J.~Lu, R.~Arora, J.~Haupt, H.~Liu, and T.~Zhao, ``Symmetry,
  saddle points, and global geometry of nonconvex matrix factorization,''
  \emph{arXiv preprint arXiv:1612.09296}, 2016.

\bibitem{zhu2017global}
Z.~Zhu, Q.~Li, G.~Tang, and M.~B. Wakin, ``Global optimality in low-rank matrix
  optimization,'' \emph{IEEE Transactions on Signal Processing}, vol.~66,
  no.~13, pp. 3614--3628, 2018.

\bibitem{zhu2018global}
Z.~Zhu, D.~Soudry, Y.~C. Eldar, and M.~B. Wakin, ``The global optimization
  geometry of shallow linear neural networks,'' \emph{arXiv preprint
  arXiv:1805.04938}, 2018.

\bibitem{hauser2018pca}
R.~A. Hauser, A.~Eftekhari, and H.~F. Matzinger, ``{PCA} by determinant
  optimisation has no spurious local optima,'' in \emph{Proceedings of the 24th
  ACM SIGKDD International Conference on Knowledge Discovery \& Data
  Mining}.\hskip 1em plus 0.5em minus 0.4em\relax ACM, 2018, pp. 1504--1511.

\bibitem{candes2011tight}
E.~J. Cand\`es and Y.~Plan, ``Tight oracle inequalities for low-rank matrix
  recovery from a minimal number of noisy random measurements,'' \emph{IEEE
  Transactions on Information Theory}, vol.~57, no.~4, pp. 2342--2359, 2011.

\bibitem{candes2008restricted}
E.~J. Cand{\`e}s, ``The restricted isometry property and its implications for
  compressed sensing,'' \emph{Comptes Rendus Mathematique}, vol. 346, no.~9,
  pp. 589--592, 2008.

\bibitem{fienup1982phase}
J.~R. Fienup, ``Phase retrieval algorithms: a comparison.'' \emph{Applied
  optics}, vol.~21, no.~15, pp. 2758--2769, 1982.

\bibitem{candes2012phaselift}
E.~J. Cand\`es, T.~Strohmer, and V.~Voroninski, ``Phaselift: Exact and stable
  signal recovery from magnitude measurements via convex programming,''
  \emph{Communications on Pure and Applied Mathematics}, vol.~66, no.~8, pp.
  1241--1274, 2013.

\bibitem{jaganathan2015phase}
K.~Jaganathan, Y.~C. Eldar, and B.~Hassibi, ``Phase retrieval: An overview of
  recent developments,'' \emph{arXiv preprint arXiv:1510.07713}, 2015.

\bibitem{chen2015exact}
Y.~Chen, Y.~Chi, and A.~Goldsmith, ``Exact and stable covariance estimation
  from quadratic sampling via convex programming,'' \emph{IEEE Transactions on
  Information Theory}, vol.~61, no.~7, pp. 4034--4059, July 2015.

\bibitem{cai2015rop}
T.~Cai and A.~Zhang, ``{ROP}: Matrix recovery via rank-one projections,''
  \emph{The Annals of Statistics}, vol.~43, no.~1, pp. 102--138, 2015.

\bibitem{tian2012experimental}
L.~Tian, J.~Lee, S.~B. Oh, and G.~Barbastathis, ``Experimental compressive
  phase space tomography,'' \emph{Optics Express}, vol.~20, no.~8, p. 8296,
  2012.

\bibitem{bao2018coherence}
C.~Bao, G.~Barbastathis, H.~Ji, Z.~Shen, and Z.~Zhang, ``Coherence retrieval
  using trace regularization,'' \emph{SIAM Journal on Imaging Sciences},
  vol.~11, no.~1, pp. 679--706, 2018.

\bibitem{vaswani2017robust}
N.~Vaswani, T.~Bouwmans, S.~Javed, and P.~Narayanamurthy, ``Robust subspace
  learning: Robust {PCA}, robust subspace tracking, and robust subspace
  recovery,'' \emph{IEEE Signal Processing Magazine}, vol.~35, no.~4, pp.
  32--55, 2018.

\bibitem{ten1977orthogonal}
J.~M. Ten~Berge, ``Orthogonal {P}rocrustes rotation for two or more matrices,''
  \emph{Psychometrika}, vol.~42, no.~2, pp. 267--276, 1977.

\bibitem{zheng2015convergent}
Q.~Zheng and J.~Lafferty, ``A convergent gradient descent algorithm for rank
  minimization and semidefinite programming from random linear measurements,''
  \emph{arXiv preprint arXiv:1506.06081}, 2015.

\bibitem{bhojanapalli2016dropping}
S.~Bhojanapalli, A.~Kyrillidis, and S.~Sanghavi, ``Dropping convexity for
  faster semi-definite optimization,'' in \emph{Conference on Learning Theory},
  2016, pp. 530--582.

\bibitem{tropp2015introduction}
J.~A. Tropp, ``An introduction to matrix concentration inequalities,''
  \emph{Foundations and Trends{\textregistered} in Machine Learning}, vol.~8,
  no. 1-2, pp. 1--230, 2015.

\bibitem{vershynin2010nonasym}
R.~Vershynin, ``Introduction to the non-asymptotic analysis of random
  matrices,'' in \emph{Compressed Sensing}, Y.~C. Eldar and G.~Kutyniok,
  Eds.\hskip 1em plus 0.5em minus 0.4em\relax Cambridge University Press, 2012,
  pp. 210--268.

\bibitem{soltanolkotabi2014algorithms}
M.~Soltanolkotabi, ``Algorithms and theory for clustering and nonconvex
  quadratic programming,'' Ph.D. dissertation, Stanford University, 2014.

\bibitem{sanghavi2017local}
S.~Sanghavi, R.~Ward, and C.~D. White, ``The local convexity of solving systems
  of quadratic equations,'' \emph{Results in Mathematics}, vol.~71, no. 3-4,
  pp. 569--608, 2017.

\bibitem{li2018nonconvex}
Y.~Li, C.~Ma, Y.~Chen, and Y.~Chi, ``Nonconvex matrix factorization from
  rank-one measurements,'' in \emph{Proceedings of International Conference on
  Artificial Intelligence and Statistics (AISTATS)}, vol.~89.\hskip 1em plus
  0.5em minus 0.4em\relax PMLR, 16--18 Apr 2019, pp. 1496--1505.

\bibitem{ling2017regularized}
S.~Ling and T.~Strohmer, ``Regularized gradient descent: a non-convex recipe
  for fast joint blind deconvolution and demixing,'' \emph{Information and
  Inference: A Journal of the IMA}, vol.~8, no.~1, pp. 1--49, 2018.

\bibitem{shi2018demising}
J.~Dong and Y.~Shi, ``Nonconvex demixing from bilinear measurements,''
  \emph{IEEE Transactions on Signal Processing}, vol.~66, no.~19, pp.
  5152--5166, 2018.

\bibitem{zheng2016convergence}
Q.~Zheng and J.~Lafferty, ``Convergence analysis for rectangular matrix
  completion using {B}urer-{M}onteiro factorization and gradient descent,''
  \emph{arXiv preprint arXiv:1605.07051}, 2016.

\bibitem{yi2016fast}
X.~Yi, D.~Park, Y.~Chen, and C.~Caramanis, ``Fast algorithms for robust {PCA}
  via gradient descent,'' in \emph{Advances in neural information processing
  systems}, 2016, pp. 4152--4160.

\bibitem{keshavan2010noisy}
R.~H. Keshavan, A.~Montanari, and S.~Oh, ``Matrix completion from noisy
  entries,'' \emph{The Journal of Machine Learning Research}, vol.~11, pp.
  2057--2078, July 2010.

\bibitem{huang2017blind}
W.~Huang and P.~Hand, ``Blind deconvolution by a steepest descent algorithm on
  a quotient manifold,'' \emph{SIAM Journal on Imaging Sciences}, vol.~11,
  no.~4, pp. 2757--2785, 2018.

\bibitem{chen2019nvx}
J.~Chen, D.~Liu, and X.~Li, ``Nonconvex rectangular matrix completion via
  gradient descent without $\ell_{2,\infty}$ regularization,''
  \emph{arXiv:1901.06116}, 2019.

\bibitem{stein1972bound}
C.~Stein, ``A bound for the error in the normal approximation to the
  distribution of a sum of dependent random variables,'' in \emph{Proceedings
  of the Sixth Berkeley Symposium on Mathematical Statistics and Probability},
  1972.

\bibitem{chen2010normal}
L.~H. Chen, L.~Goldstein, and Q.-M. Shao, \emph{Normal approximation by
  {S}tein's method}.\hskip 1em plus 0.5em minus 0.4em\relax Springer Science \&
  Business Media, 2010.

\bibitem{el2015impact}
N.~El~Karoui, ``On the impact of predictor geometry on the performance on
  high-dimensional ridge-regularized generalized robust regression
  estimators,'' \emph{Probability Theory and Related Fields}, pp. 1--81, 2015.

\bibitem{zhong2017near}
Y.~Zhong and N.~Boumal, ``Near-optimal bounds for phase synchronization,''
  \emph{SIAM Journal on Optimization}, vol.~28, no.~2, pp. 989--1016, 2018.

\bibitem{sur2017likelihood}
P.~Sur, Y.~Chen, and E.~J. Cand{\`e}s, ``The likelihood ratio test in
  high-dimensional logistic regression is asymptotically a rescaled
  chi-square,'' \emph{Probability Theory and Related Fields}, pp. 1--72, 2017.

\bibitem{chen2017spectral}
Y.~Chen, J.~Fan, C.~Ma, and K.~Wang, ``Spectral method and regularized {MLE}
  are both optimal for top-{$K$} ranking,'' \emph{Annals of Statistics},
  vol.~47, no.~4, pp. 2204--2235, August 2019.

\bibitem{abbe2017entrywise}
E.~Abbe, J.~Fan, K.~Wang, and Y.~Zhong, ``Entrywise eigenvector analysis of
  random matrices with low expected rank,'' \emph{arXiv preprint
  arXiv:1709.09565}, 2017.

\bibitem{chen2018gradient}
Y.~Chen, Y.~Chi, J.~Fan, and C.~Ma, ``Gradient descent with random
  initialization: Fast global convergence for nonconvex phase retrieval,''
  \emph{Mathematical Programming}, vol. 176, no. 1-2, pp. 5--37, July 2019.

\bibitem{ding2018leave}
L.~Ding and Y.~Chen, ``The leave-one-out approach for matrix completion: Primal
  and dual analysis,'' \emph{arXiv preprint arXiv:1803.07554}, 2018.

\bibitem{chen2019noisy}
Y.~Chen, Y.~Chi, J.~Fan, C.~Ma, and Y.~Yan, ``Noisy matrix completion:
  Understanding statistical guarantees for convex relaxation via nonconvex
  optimization,'' \emph{arXiv preprint arXiv:1902.07698}, 2019.

\bibitem{chen2019inference}
Y.~Chen, J.~Fan, C.~Ma, and Y.~Yan, ``Inference and uncertainty quantification
  for noisy matrix completion,'' \emph{arXiv preprint arXiv:1906.04159}, 2019.

\bibitem{soltanolkotabi2017structured}
M.~Soltanolkotabi, ``Structured signal recovery from quadratic measurements:
  Breaking sample complexity barriers via nonconvex optimization,'' \emph{IEEE
  Transactions on Information Theory}, vol.~65, no.~4, pp. 2374--2400, 2019.

\bibitem{cai2016optimal}
T.~T. Cai, X.~Li, and Z.~Ma, ``Optimal rates of convergence for noisy sparse
  phase retrieval via thresholded {W}irtinger flow,'' \emph{The Annals of
  Statistics}, vol.~44, no.~5, pp. 2221--2251, 2016.

\bibitem{duchi2008efficient}
J.~Duchi, S.~Shalev-Shwartz, Y.~Singer, and T.~Chandra, ``Efficient projections
  onto the $\ell_1$-ball for learning in high dimensions,'' in
  \emph{International conference on Machine learning}, 2008, pp. 272--279.

\bibitem{wang66sparse}
G.~Wang, L.~Zhang, G.~B. Giannakis, M.~Ak{\c{c}}akaya, and J.~Chen, ``Sparse
  phase retrieval via truncated amplitude flow,'' \emph{IEEE Transactions on
  Signal Processing}, vol.~66, no.~2, pp. 479--491, 2018.

\bibitem{li2013sparse}
X.~Li and V.~Voroninski, ``Sparse signal recovery from quadratic measurements
  via convex programming,'' \emph{SIAM Journal on Mathematical Analysis},
  vol.~45, no.~5, pp. 3019--3033, 2013.

\bibitem{oymak2012simultaneously}
S.~Oymak, A.~Jalali, M.~Fazel, Y.~C. Eldar, and B.~Hassibi, ``Simultaneously
  structured models with application to sparse and low-rank matrices,''
  \emph{IEEE Transactions on Information Theory}, vol.~61, no.~5, pp.
  2886--2908, 2015.

\bibitem{jagatap2017fast}
G.~Jagatap and C.~Hegde, ``Fast, sample-efficient algorithms for structured
  phase retrieval,'' in \emph{Advances in Neural Information Processing
  Systems}, 2017, pp. 4917--4927.

\bibitem{kolte2016phase}
R.~Kolte and A.~{\"O}zg{\"u}r, ``Phase retrieval via incremental truncated
  {W}irtinger flow,'' \emph{arXiv preprint arXiv:1606.03196}, 2016.

\bibitem{wang2017solving}
G.~Wang, G.~B. Giannakis, and Y.~C. Eldar, ``Solving systems of random
  quadratic equations via truncated amplitude flow,'' \emph{IEEE Transactions
  on Information Theory}, vol.~64, no.~2, pp. 773--794, 2018.

\bibitem{li2017nonconvex}
Y.~Li, Y.~Chi, H.~Zhang, and Y.~Liang, ``Nonconvex low-rank matrix recovery
  with arbitrary outliers via median-truncated gradient descent,'' \emph{arXiv
  preprint arXiv:1709.08114}, 2017.

\bibitem{clarke1975generalized}
F.~H. Clarke, ``Generalized gradients and applications,'' \emph{Transactions of
  the American Mathematical Society}, vol. 205, pp. 247--262, 1975.

\bibitem{zhang2017reshaped}
H.~Zhang, Y.~Zhou, Y.~Liang, and Y.~Chi, ``A nonconvex approach for phase
  retrieval: Reshaped {W}irtinger flow and incremental algorithms,''
  \emph{Journal of Machine Learning Research}, vol.~18, no. 141, pp. 1--35,
  2017.

\bibitem{davis2017nonsmooth}
D.~Davis, D.~Drusvyatskiy, and C.~Paquette, ``The nonsmooth landscape of phase
  retrieval,'' \emph{arXiv preprint arXiv:1711.03247}, 2017.

\bibitem{boumal2016nonconvex}
N.~Boumal, ``Nonconvex phase synchronization,'' \emph{SIAM Journal on
  Optimization}, vol.~26, no.~4, pp. 2355--2377, 2016.

\bibitem{liu2017estimation}
H.~Liu, M.-C. Yue, and A.~Man-Cho~So, ``On the estimation performance and
  convergence rate of the generalized power method for phase synchronization,''
  \emph{SIAM Journal on Optimization}, vol.~27, no.~4, pp. 2426--2446, 2017.

\bibitem{chen2016information}
Y.~Chen, C.~Suh, and A.~J. Goldsmith, ``Information recovery from pairwise
  measurements,'' \emph{IEEE Transactions on Information Theory}, vol.~62,
  no.~10, October 2016.

\bibitem{waldspurger2015phase}
I.~Waldspurger, A.~d{'}Aspremont, and S.~Mallat, ``Phase recovery, maxcut and
  complex semidefinite programming,'' \emph{Mathematical Programming}, vol.
  149, no. 1-2, pp. 47--81, 2015.

\bibitem{yuan2013truncated}
X.-T. Yuan and T.~Zhang, ``Truncated power method for sparse eigenvalue
  problems,'' \emph{Journal of Machine Learning Research}, vol.~14, no. Apr,
  pp. 899--925, 2013.

\bibitem{li2017blind}
Y.~Li, K.~Lee, and Y.~Bresler, ``Blind gain and phase calibration for
  low-dimensional or sparse signal sensing via power iteration,'' in
  \emph{Sampling Theory and Applications (SampTA), 2017 International
  Conference on}.\hskip 1em plus 0.5em minus 0.4em\relax IEEE, 2017, pp.
  119--123.

\bibitem{edelman1998geometry}
A.~Edelman, T.~A. Arias, and S.~T. Smith, ``The geometry of algorithms with
  orthogonality constraints,'' \emph{SIAM journal on Matrix Analysis and
  Applications}, vol.~20, no.~2, pp. 303--353, 1998.

\bibitem{absil2009optimization}
P.-A. Absil, R.~Mahony, and R.~Sepulchre, \emph{Optimization algorithms on
  matrix manifolds}.\hskip 1em plus 0.5em minus 0.4em\relax Princeton
  University Press, 2009.

\bibitem{balzano2014subspace}
L.~Balzano and S.~J. Wright, ``Local convergence of an algorithm for subspace
  identification from partial data,'' \emph{Foundations of Computational
  Mathematics}, pp. 1--36, 2014.

\bibitem{dai2012geometric}
W.~Dai, E.~Kerman, and O.~Milenkovic, ``A geometric approach to low-rank matrix
  completion,'' \emph{IEEE Transactions on Information Theory}, vol.~58, no.~1,
  pp. 237--247, 2012.

\bibitem{wei2016guarantees}
K.~Wei, J.-F. Cai, T.~F. Chan, and S.~Leung, ``Guarantees of {R}iemannian
  optimization for low rank matrix recovery,'' \emph{SIAM Journal on Matrix
  Analysis and Applications}, vol.~37, no.~3, pp. 1198--1222, 2016.

\bibitem{wei2016guarantees_mc}
------, ``Guarantees of {R}iemannian optimization for low rank matrix
  completion,'' \emph{arXiv preprint arXiv:1603.06610}, 2016.

\bibitem{vandereycken2013low}
B.~Vandereycken, ``Low-rank matrix completion by {R}iemannian optimization,''
  \emph{SIAM Journal on Optimization}, vol.~23, no.~2, pp. 1214--1236, 2013.

\bibitem{uschmajew2018critical}
A.~Uschmajew and B.~Vandereycken, ``On critical points of quadratic low-rank
  matrix optimization problems,'' 2018.

\bibitem{cai2018solving}
J.-F. Cai and K.~Wei, ``Solving systems of phaseless equations via {R}iemannian
  optimization with optimal sampling complexity,'' \emph{arXiv preprint
  arXiv:1809.02773}, 2018.

\bibitem{cai2018exploiting}
------, ``Exploiting the structure effectively and efficiently in low rank
  matrix recovery,'' \emph{arXiv preprint arXiv:1809.03652}, 2018.

\bibitem{robbins1985stochastic}
H.~Robbins and S.~Monro, ``A stochastic approximation method,'' in
  \emph{Herbert Robbins Selected Papers}.\hskip 1em plus 0.5em minus
  0.4em\relax Springer, 1985, pp. 102--109.

\bibitem{nemirovski2009robust}
A.~Nemirovski, A.~Juditsky, G.~Lan, and A.~Shapiro, ``Robust stochastic
  approximation approach to stochastic programming,'' \emph{SIAM Journal on
  optimization}, vol.~19, no.~4, pp. 1574--1609, 2009.

\bibitem{bottou2018optimization}
L.~Bottou, F.~E. Curtis, and J.~Nocedal, ``Optimization methods for large-scale
  machine learning,'' \emph{SIAM Review}, vol.~60, no.~2, pp. 223--311, 2018.

\bibitem{wei2015solving}
K.~Wei, ``Solving systems of phaseless equations via {K}aczmarz methods: A
  proof of concept study,'' \emph{Inverse Problems}, vol.~31, no.~12, p.
  125008, 2015.

\bibitem{li2015phase}
G.~Li, Y.~Gu, and Y.~M. Lu, ``Phase retrieval using iterative projections:
  Dynamics in the large systems limit,'' in \emph{53rd Annual Allerton
  Conference on Communication, Control, and Computing}, 2015, pp. 1114--1118.

\bibitem{chi2016kaczmarz}
Y.~Chi and Y.~M. Lu, ``{K}aczmarz method for solving quadratic equations,''
  \emph{IEEE Signal Processing Letters}, vol.~23, no.~9, pp. 1183--1187, 2016.

\bibitem{tan2017phase}
Y.~S. Tan and R.~Vershynin, ``Phase retrieval via randomized {K}aczmarz:
  Theoretical guarantees,'' \emph{Information and Inference: A Journal of the
  IMA}, vol.~8, no.~1, pp. 97--123, 2018.

\bibitem{monardo2019solving}
V.~Monardo, Y.~Li, and Y.~Chi, ``Solving quadratic equations via
  amplitude-based nonconvex optimization,'' in \emph{2019 IEEE International
  Conference on Acoustics, Speech and Signal Processing (ICASSP)}.\hskip 1em
  plus 0.5em minus 0.4em\relax IEEE, 2019, pp. 5526--5530.

\bibitem{jin2016provable}
C.~Jin, S.~M. Kakade, and P.~Netrapalli, ``Provable efficient online matrix
  completion via non-convex stochastic gradient descent,'' in \emph{Advances in
  Neural Information Processing Systems}, 2016, pp. 4520--4528.

\bibitem{trefethen1997numerical}
L.~N. Trefethen and D.~Bau~III, \emph{Numerical linear algebra}.\hskip 1em plus
  0.5em minus 0.4em\relax {SIAM}, 1997, vol.~50.

\bibitem{netrapalli2015phase}
P.~Netrapalli, P.~Jain, and S.~Sanghavi, ``Phase retrieval using alternating
  minimization,'' \emph{IEEE Transactions on Signal Processing}, vol.~18,
  no.~63, pp. 4814--4826, 2015.

\bibitem{waldspurger2016phase}
I.~Waldspurger, ``Phase retrieval with random gaussian sensing vectors by
  alternating projections,'' \emph{IEEE Transactions on Information Theory},
  vol.~64, no.~5, pp. 3301--3312, May 2018.

\bibitem{gerchberg1972practical}
R.~W. Gerchberg, ``A practical algorithm for the determination of phase from
  image and diffraction plane pictures,'' \emph{Optik}, vol.~35, p. 237, 1972.

\bibitem{hastie2015matrix}
T.~Hastie, R.~Mazumder, J.~D. Lee, and R.~Zadeh, ``Matrix completion and
  low-rank {SVD} via fast alternating least squares.'' \emph{Journal of Machine
  Learning Research}, vol.~16, pp. 3367--3402, 2015.

\bibitem{keshavan2012efficient}
R.~H. Keshavan, ``Efficient algorithms for collaborative filtering,'' Ph.D.
  dissertation, Stanford University, 2012.

\bibitem{hardt2014understanding}
M.~Hardt, ``Understanding alternating minimization for matrix completion,'' in
  \emph{Foundations of Computer Science (FOCS), 2014 IEEE 55th Annual Symposium
  on}.\hskip 1em plus 0.5em minus 0.4em\relax IEEE, 2014, pp. 651--660.

\bibitem{hardt2014fast}
M.~Hardt and M.~Wootters, ``Fast matrix completion without the condition
  number,'' in \emph{Proceedings of The 27th Conference on Learning Theory},
  2014, pp. 638--678.

\bibitem{zhao2015nonconvex}
T.~Zhao, Z.~Wang, and H.~Liu, ``Nonconvex low rank matrix factorization via
  inexact first order oracle.''\hskip 1em plus 0.5em minus 0.4em\relax Advances
  in Neural Information Processing Systems, 2015.

\bibitem{jain2010guaranteed}
P.~Jain, R.~Meka, and I.~S. Dhillon, ``Guaranteed rank minimization via
  singular value projection,'' in \emph{Advances in Neural Information
  Processing Systems}, 2010, pp. 937--945.

\bibitem{oymak2018sharp}
S.~Oymak, B.~Recht, and M.~Soltanolkotabi, ``Sharp time--data tradeoffs for
  linear inverse problems,'' \emph{IEEE Transactions on Information Theory},
  vol.~64, no.~6, pp. 4129--4158, 2018.

\bibitem{halko2011finding}
N.~Halko, P.-G. Martinsson, and J.~A. Tropp, ``Finding structure with
  randomness: Probabilistic algorithms for constructing approximate matrix
  decompositions,'' \emph{SIAM Revieweview}, vol.~53, no.~2, pp. 217--288,
  2011.

\bibitem{tanner2013normalized}
J.~Tanner and K.~Wei, ``Normalized iterative hard thresholding for matrix
  completion,'' \emph{SIAM Journal on Scientific Computing}, vol.~35, no.~5,
  pp. S104--S125, 2013.

\bibitem{lee2010admira}
K.~Lee and Y.~Bresler, ``Admira: Atomic decomposition for minimum rank
  approximation,'' \emph{IEEE Transactions on Information Theory}, vol.~56,
  no.~9, pp. 4402--4416, 2010.

\bibitem{duchi2017solving}
J.~C. Duchi and F.~Ruan, ``Solving (most) of a set of quadratic equalities:
  Composite optimization for robust phase retrieval,'' \emph{arXiv preprint
  arXiv:1705.02356, accepted to {\em Information and Inference: A Journal of
  the IMA}}, 2017.

\bibitem{duchi2017stochastic}
J.~Duchi and F.~Ruan, ``Stochastic methods for composite optimization
  problems,'' \emph{arXiv preprint arXiv:1703.08570}, 2017.

\bibitem{charisopoulos2019lowrank}
V.~Charisopoulos, Y.~Chen, D.~Davis, M.~Díaz, L.~Ding, and D.~Drusvyatskiy,
  ``Low-rank matrix recovery with composite optimization: good conditioning and
  rapid convergence,'' \emph{arXiv preprint arXiv:1904.10020}, 2019.

\bibitem{charisopoulos2019composite}
V.~Charisopoulos, D.~Davis, M.~D{\'\i}az, and D.~Drusvyatskiy, ``Composite
  optimization for robust blind deconvolution,'' \emph{arXiv preprint
  arXiv:1901.01624}, 2019.

\bibitem{donoho2013phase}
D.~L. Donoho, M.~Gavish, and A.~Montanari, ``The phase transition of matrix
  recovery from gaussian measurements matches the minimax mse of matrix
  denoising,'' \emph{Proceedings of the National Academy of Sciences}, vol.
  110, no.~21, pp. 8405--8410, 2013.

\bibitem{ma2018optimization}
J.~Ma, J.~Xu, and A.~Maleki, ``Optimization-based {AMP} for phase retrieval:
  The impact of initialization and $\ell\_2 $-regularization,'' \emph{arXiv
  preprint arXiv:1801.01170}, 2018.

\bibitem{ma2018approximate}
------, ``Approximate message passing for amplitude based optimization,''
  \emph{arXiv preprint arXiv:1806.03276}, 2018.

\bibitem{zeng2017coordinate}
W.-J. Zeng and H.-C. So, ``Coordinate descent algorithms for phase retrieval,''
  \emph{arXiv preprint arXiv:1706.03474}, 2017.

\bibitem{pinilla2018phase}
S.~Pinilla, J.~Bacca, and H.~Arguello, ``Phase retrieval algorithm via
  nonconvex minimization using a smoothing function,'' \emph{IEEE Transactions
  on Signal Processing}, vol.~66, no.~17, pp. 4574--4584, 2018.

\bibitem{davis1970rotation}
C.~Davis and W.~M. Kahan, ``The rotation of eigenvectors by a perturbation.
  iii,'' \emph{{SIAM} Journal on Numerical Analysis}, vol.~7, no.~1, pp. 1--46,
  1970.

\bibitem{Bjorck1973principal}
{\.A}.~Bj{\"{o}}rck and G.~H. Golub, ``Numerical methods for computing angles
  between linear subspaces,'' \emph{Mathematics of Computation}, vol.~27, no.
  123, pp. 579--594, Jul. 1973.

\bibitem{wedin1972perturbation}
P.-{\AA}. Wedin, ``Perturbation bounds in connection with singular value
  decomposition,'' \emph{BIT Numerical Mathematics}, vol.~12, no.~1, pp.
  99--111, 1972.

\bibitem{Ferguson:1996}
T.~S. Ferguson, \emph{A Course in Large Sample Theory}.\hskip 1em plus 0.5em
  minus 0.4em\relax Chapman \& {Hall/CRC}, 1996, vol.~38.

\bibitem{lu2017phase}
Y.~M. Lu and G.~Li, ``Phase transitions of spectral initialization for
  high-dimensional nonconvex estimation,'' \emph{arXiv preprint
  arXiv:1702.06435}, 2017.

\bibitem{mondelli2017fundamental}
M.~Mondelli and A.~Montanari, ``Fundamental limits of weak recovery with
  applications to phase retrieval,'' \emph{arXiv preprint arXiv:1708.05932},
  2017.

\bibitem{luo2019optimal}
W.~{Luo}, W.~{Alghamdi}, and Y.~M. {Lu}, ``Optimal spectral initialization for
  signal recovery with applications to phase retrieval,'' \emph{IEEE
  Transactions on Signal Processing}, vol.~67, no.~9, pp. 2347--2356, May 2019.

\bibitem{Li:92}
K.-C. Li, ``On principal {H}essian directions for data visualization and
  dimension reduction: Another application of {Stein}'s lemma,'' \emph{J. Am.
  Stat. Assoc}, vol.~87, no. 420, pp. 1025--1039, 1992.

\bibitem{achlioptas2007fast}
D.~Achlioptas and F.~McSherry, ``Fast computation of low-rank matrix
  approximations,'' \emph{Journal of the ACM}, vol.~54, no.~2, p.~9, 2007.

\bibitem{rohe2011spectral}
K.~Rohe, S.~Chatterjee, and B.~Yu, ``Spectral clustering and the
  high-dimensional stochastic blockmodel,'' \emph{The Annals of Statistics},
  vol.~39, no.~4, pp. 1878--1915, 2011.

\bibitem{abbe2017community}
E.~Abbe, ``Community detection and stochastic block models: Recent
  developments,'' \emph{Journal of Machine Learning Research}, vol.~18, no.
  177, pp. 1--86, 2018.

\bibitem{negahban2016rank}
S.~Negahban, S.~Oh, and D.~Shah, ``Rank centrality: Ranking from pairwise
  comparisons,'' \emph{Operations Research}, vol.~65, no.~1, pp. 266--287,
  2016.

\bibitem{chen2015spectral}
Y.~Chen and C.~Suh, ``Spectral {MLE}: Top-$k$ rank aggregation from pairwise
  comparisons,'' in \emph{International Conference on Machine Learning}, 2015,
  pp. 371--380.

\bibitem{montanari2016spectral}
A.~Montanari and N.~Sun, ``Spectral algorithms for tensor completion,''
  \emph{Communications on Pure and Applied Mathematics}, 2016.

\bibitem{hao2018sparse}
B.~Hao, A.~Zhang, and G.~Cheng, ``Sparse and low-rank tensor estimation via
  cubic sketchings,'' \emph{arXiv preprint arXiv:1801.09326}, 2018.

\bibitem{zhang2018tensor}
A.~Zhang and D.~Xia, ``Tensor {SVD}: Statistical and computational limits,''
  \emph{IEEE Transactions on Information Theory}, 2018.

\bibitem{cai2019tensor}
C.~Cai, G.~Li, H.~V. Poor, and Y.~Chen, ``Nonconvex low-rank symmetric tensor
  completion from noisy data,'' 2019.

\bibitem{dauphin2014identifying}
Y.~N. Dauphin, R.~Pascanu, C.~Gulcehre, K.~Cho, S.~Ganguli, and Y.~Bengio,
  ``Identifying and attacking the saddle point problem in high-dimensional
  non-convex optimization,'' in \emph{Advances in neural information processing
  systems}, 2014, pp. 2933--2941.

\bibitem{anandkumar2016efficient}
A.~Anandkumar and R.~Ge, ``Efficient approaches for escaping higher order
  saddle points in non-convex optimization,'' in \emph{Conference on Learning
  Theory}, 2016, pp. 81--102.

\bibitem{ge2017no}
R.~Ge, C.~Jin, and Y.~Zheng, ``No spurious local minima in nonconvex low rank
  problems: A unified geometric analysis,'' in \emph{International Conference
  on Machine Learning}, 2017, pp. 1233--1242.

\bibitem{chen2017memory}
J.~Chen and X.~Li, ``Memory-efficient kernel {PCA} via partial matrix sampling
  and nonconvex optimization: a model-free analysis of local minima,''
  \emph{arXiv preprint arXiv:1711.01742}, 2017.

\bibitem{gunasekar2017implicit}
S.~Gunasekar, B.~E. Woodworth, S.~Bhojanapalli, B.~Neyshabur, and N.~Srebro,
  ``Implicit regularization in matrix factorization,'' in \emph{Advances in
  Neural Information Processing Systems}, 2017, pp. 6151--6159.

\bibitem{li2017algorithmic}
Y.~Li, T.~Ma, and H.~Zhang, ``Algorithmic regularization in over-parameterized
  matrix sensing and neural networks with quadratic activations,'' in
  \emph{Conference On Learning Theory}, 2018, pp. 2--47.

\bibitem{kyrillidis2018implicit}
A.~Kyrillidis and A.~Kalev, ``Implicit regularization and solution uniqueness
  in over-parameterized matrix sensing,'' \emph{arXiv preprint
  arXiv:1806.02046}, 2018.

\bibitem{wang2011unique}
M.~Wang, W.~Xu, and A.~Tang, ``A unique nonnegative solution to an
  underdetermined system: From vectors to matrices,'' \emph{IEEE Transactions
  on Signal Processing}, vol.~59, no.~3, pp. 1007--1016, 2011.

\bibitem{demanet2012stable}
L.~Demanet and P.~Hand, ``Stable optimizationless recovery from phaseless
  linear measurements,'' \emph{Journal of Fourier Analysis and Applications},
  vol.~20, no.~1, pp. 199--221, 2014.

\bibitem{candes2012solving}
E.~J. Cand\`es and X.~Li, ``Solving quadratic equations via {PhaseLift} when
  there are about as many equations as unknowns,'' \emph{Foundations of
  Computational Mathematics}, vol.~14, no.~5, pp. 1017--1026, 2014.

\bibitem{sun2017complete}
J.~Sun, Q.~Qu, and J.~Wright, ``Complete dictionary recovery over the sphere
  {I}: Overview and the geometric picture,'' \emph{IEEE Transactions on
  Information Theory}, vol.~63, no.~2, pp. 853--884, 2017.

\bibitem{sun2017trust}
------, ``Complete dictionary recovery over the sphere {II}: Recovery by
  {R}iemannian trust-region method,'' \emph{IEEE Transactions on Information
  Theory}, vol.~63, no.~2, pp. 885--914, 2017.

\bibitem{zhai2019complete}
Y.~Zhai, Z.~Yang, Z.~Liao, J.~Wright, and Y.~Ma, ``Complete dictionary learning
  via $\ell_4$-norm maximization over the orthogonal group,'' \emph{arXiv
  preprint arXiv:1906.02435}.

\bibitem{mei2016landscape}
S.~Mei, Y.~Bai, and A.~Montanari, ``The landscape of empirical risk for
  nonconvex losses,'' \emph{The Annals of Statistics}, vol.~46, no.~6A, pp.
  2747--2774, 2018.

\bibitem{zhang2019sharp}
R.~Y. Zhang, S.~Sojoudi, and J.~Lavaei, ``Sharp restricted isometry bounds for
  the inexistence of spurious local minima in nonconvex matrix recovery,''
  \emph{arXiv preprint arXiv:1901.01631}, 2019.

\bibitem{zhang2018much}
R.~Zhang, C.~Josz, S.~Sojoudi, and J.~Lavaei, ``How much restricted isometry is
  needed in nonconvex matrix recovery?'' in \emph{Advances in neural
  information processing systems}, 2018, pp. 5586--5597.

\bibitem{park2017non}
D.~Park, A.~Kyrillidis, C.~Carmanis, and S.~Sanghavi, ``Non-square matrix
  sensing without spurious local minima via the burer-monteiro approach,'' in
  \emph{Artificial Intelligence and Statistics}, 2017, pp. 65--74.

\bibitem{auer1996exponentially}
P.~Auer, M.~Herbster, and M.~K. Warmuth, ``Exponentially many local minima for
  single neurons,'' in \emph{Advances in neural information processing
  systems}, 1996, pp. 316--322.

\bibitem{safran2018spurious}
I.~Safran and O.~Shamir, ``Spurious local minima are common in two-layer {ReLU}
  neural networks,'' in \emph{International Conference on Machine Learning},
  2018, pp. 4430--4438.

\bibitem{ge2015escaping}
R.~Ge, F.~Huang, C.~Jin, and Y.~Yuan, ``Escaping from saddle points-online
  stochastic gradient for tensor decomposition.'' in \emph{Conference on
  Learning Theory (COLT)}, 2015, pp. 797--842.

\bibitem{sun2016nonconvex}
J.~Sun, ``When are nonconvex optimization problems not scary?'' Ph.D.
  dissertation, 2016.

\bibitem{lee2016gradient}
J.~D. Lee, M.~Simchowitz, M.~I. Jordan, and B.~Recht, ``Gradient descent only
  converges to minimizers,'' in \emph{Conference on Learning Theory}, 2016, pp.
  1246--1257.

\bibitem{lee2017first}
J.~D. Lee, I.~Panageas, G.~Piliouras, M.~Simchowitz, M.~I. Jordan, and
  B.~Recht, ``First-order methods almost always avoid strict saddle points,''
  \emph{Mathematical Programming}, pp. 1--27.

\bibitem{hong2018gradient}
M.~Hong, M.~Razaviyayn, and J.~Lee, ``Gradient primal-dual algorithm converges
  to second-order stationary solution for nonconvex distributed optimization
  over networks,'' in \emph{International Conference on Machine Learning},
  2018, pp. 2014--2023.

\bibitem{li2019alternating}
Q.~Li, Z.~Zhu, and G.~Tang, ``Alternating minimizations converge to
  second-order optimal solutions,'' in \emph{International Conference on
  Machine Learning}, 2019, pp. 3935--3943.

\bibitem{du2017gradient}
S.~S. Du, C.~Jin, J.~D. Lee, M.~I. Jordan, A.~Singh, and B.~Poczos, ``Gradient
  descent can take exponential time to escape saddle points,'' in
  \emph{Advances in Neural Information Processing Systems}, 2017, pp.
  1067--1077.

\bibitem{Wright2018random}
D.~Gilboa, S.~Buchanan, and J.~Wright, ``Efficient dictionary learning with
  gradient descent,'' \emph{ICML Workshop on Modern Trends in Nonconvex
  Optimization for Machine Learning}, 2018.

\bibitem{brutzkus2017globally}
A.~Brutzkus and A.~Globerson, ``Globally optimal gradient descent for a
  {C}onvnet with gaussian inputs,'' in \emph{International Conference on
  Machine Learning}, 2017, pp. 605--614.

\bibitem{nesterov2006cubic}
Y.~Nesterov and B.~T. Polyak, ``Cubic regularization of {N}ewton method and its
  global performance,'' \emph{Mathematical Programming}, vol. 108, no.~1, pp.
  177--205, 2006.

\bibitem{conn2000trust}
A.~R. Conn, N.~I. Gould, and P.~L. Toint, \emph{Trust region methods}.\hskip
  1em plus 0.5em minus 0.4em\relax {SIAM}, 2000, vol.~1.

\bibitem{absil2007trust}
P.-A. Absil, C.~G. Baker, and K.~A. Gallivan, ``Trust-region methods on
  {R}iemannian manifolds,'' \emph{Foundations of Computational Mathematics},
  vol.~7, no.~3, pp. 303--330, 2007.

\bibitem{carmon2016gradient}
Y.~Carmon and J.~C. Duchi, ``Gradient descent efficiently finds the
  cubic-regularized non-convex newton step,'' \emph{arXiv preprint
  arXiv:1612.00547}, 2016.

\bibitem{curtis2017trust}
F.~E. Curtis, D.~P. Robinson, and M.~Samadi, ``A trust region algorithm with a
  worst-case iteration complexity of $\mathcal{O}(\epsilon^{-3/2})$ for
  nonconvex optimization,'' \emph{Mathematical Programming}, vol. 162, no. 1-2,
  pp. 1--32, 2017.

\bibitem{carmon2018accelerated}
Y.~Carmon, J.~C. Duchi, O.~Hinder, and A.~Sidford, ``Accelerated methods for
  nonconvex optimization,'' \emph{SIAM Journal on Optimization}, vol.~28,
  no.~2, pp. 1751--1772, 2018.

\bibitem{nesterov1983method}
Y.~E. Nesterov, ``A method for solving the convex programming problem with
  convergence rate $o(1/k^2)$,'' in \emph{Dokl. Akad. Nauk SSSR}, vol. 269,
  1983, pp. 543--547.

\bibitem{agarwal2017finding}
N.~Agarwal, Z.~Allen-Zhu, B.~Bullins, E.~Hazan, and T.~Ma, ``Finding
  approximate local minima faster than gradient descent,'' in \emph{Proceedings
  of the 49th Annual ACM SIGACT Symposium on Theory of Computing}.\hskip 1em
  plus 0.5em minus 0.4em\relax ACM, 2017, pp. 1195--1199.

\bibitem{levy2016power}
K.~Y. Levy, ``The power of normalization: Faster evasion of saddle points,''
  \emph{arXiv preprint arXiv:1611.04831}, 2016.

\bibitem{jin2017escape}
C.~Jin, R.~Ge, P.~Netrapalli, S.~M. Kakade, and M.~I. Jordan, ``How to escape
  saddle points efficiently,'' in \emph{International Conference on Machine
  Learning}, 2017, pp. 1724--1732.

\bibitem{jin2017accelerated}
C.~Jin, P.~Netrapalli, and M.~I. Jordan, ``Accelerated gradient descent escapes
  saddle points faster than gradient descent,'' in \emph{Conference On Learning
  Theory}, 2018, pp. 1042--1085.

\bibitem{allen2017natasha}
Z.~Allen-Zhu, ``Natasha 2: Faster non-convex optimization than {SGD},'' in
  \emph{Advances in Neural Information Processing Systems}, 2018, pp.
  2675--2686.

\bibitem{xu2017first}
Y.~Xu, J.~Rong, and T.~Yang, ``First-order stochastic algorithms for escaping
  from saddle points in almost linear time,'' in \emph{Advances in Neural
  Information Processing Systems}, 2018, pp. 5530--5540.

\bibitem{allen2017neon2}
Z.~Allen-Zhu and Y.~Li, ``Neon2: Finding local minima via first-order
  oracles,'' in \emph{Advances in Neural Information Processing Systems}, 2018,
  pp. 3716--3726.

\bibitem{allen2017tutorial}
\BIBentryALTinterwordspacing
Z.~Allen-Zhu, ``Recent advances in stochastic convex and non-convex
  optimization,'' \emph{ICML tutorial}, 2017. [Online]. Available:
  \url{http://people.csail.mit.edu/zeyuan/topics/icml-2017.htm}
\BIBentrySTDinterwordspacing

\bibitem{boumal2016global}
N.~Boumal, P.-A. Absil, and C.~Cartis, ``Global rates of convergence for
  nonconvex optimization on manifolds,'' \emph{IMA Journal of Numerical
  Analysis}, vol.~39, no.~1, pp. 1--33, 2018.

\bibitem{qu2017convolutional}
Q.~Qu, Y.~Zhang, Y.~Eldar, and J.~Wright, ``Convolutional phase retrieval,'' in
  \emph{Advances in Neural Information Processing Systems}, 2017, pp.
  6086--6096.

\bibitem{bendory2018non}
T.~Bendory, Y.~C. Eldar, and N.~Boumal, ``Non-convex phase retrieval from
  {STFT} measurements,'' \emph{IEEE Transactions on Information Theory},
  vol.~64, no.~1, pp. 467--484, 2018.

\bibitem{kueng2014low}
R.~Kueng, H.~Rauhut, and U.~Terstiege, ``Low rank matrix recovery from rank one
  measurements,'' \emph{Applied and Computational Harmonic Analysis}, vol.~42,
  no.~1, pp. 88--116, 2017.

\bibitem{tropp2015convex}
J.~A. Tropp, ``Convex recovery of a structured signal from independent random
  linear measurements,'' in \emph{Sampling Theory, a Renaissance}.\hskip 1em
  plus 0.5em minus 0.4em\relax Springer, 2015, pp. 67--101.

\bibitem{yang2017misspecified}
Z.~Yang, L.~F. Yang, E.~X. Fang, T.~Zhao, Z.~Wang, and M.~Neykov,
  ``Misspecified nonconvex statistical optimization for sparse phase
  retrieval,'' \emph{Mathematical Programming}, pp. 1--27, 2019.

\bibitem{zhang2017phase}
T.~Zhang, ``Phase retrieval using alternating minimization in a batch
  setting,'' \emph{Applied and Computational Harmonic Analysis}, 2019.

\bibitem{zhang2018primal}
X.~Zhang, L.~Wang, Y.~Yu, and Q.~Gu, ``A primal-dual analysis of global
  optimality in nonconvex low-rank matrix recovery,'' in \emph{International
  conference on machine learning}, 2018, pp. 5857--5866.

\bibitem{Gross2011recovering}
D.~Gross, ``Recovering low-rank matrices from few coefficients in any basis,''
  \emph{IEEE Transactions on Information Theory}, vol.~57, no.~3, pp.
  1548--1566, March 2011.

\bibitem{Negahban2012restricted}
S.~Negahban and M.~Wainwright, ``Restricted strong convexity and weighted
  matrix completion: Optimal bounds with noise,'' \emph{The Journal of Machine
  Learning Research}, vol. 98888, pp. 1665--1697, May 2012.

\bibitem{koltchinskii2011nuclear}
V.~Koltchinskii, K.~Lounici, and A.~B. Tsybakov, ``Nuclear-norm penalization
  and optimal rates for noisy low-rank matrix completion,'' \emph{The Annals of
  Statistics}, vol.~39, no.~5, pp. 2302--2329, 2011.

\bibitem{chen2015incoherence}
Y.~Chen, ``Incoherence-optimal matrix completion,'' \emph{IEEE Transactions on
  Information Theory}, vol.~61, no.~5, pp. 2909--2923, 2015.

\bibitem{ling2017fast}
S.~Ling and T.~Strohmer, ``Fast blind deconvolution and blind demixing via
  nonconvex optimization,'' in \emph{Sampling Theory and Applications (SampTA),
  2017 International Conference on}.\hskip 1em plus 0.5em minus 0.4em\relax
  IEEE, 2017, pp. 114--118.

\bibitem{ganesh2010dense}
A.~Ganesh, J.~Wright, X.~Li, E.~Cand\`es, and Y.~Ma, ``Dense error correction
  for low-rank matrices via principal component pursuit,'' in
  \emph{International Symposium on Information Theory}, 2010, pp. 1513--1517.

\bibitem{chen2013low}
Y.~Chen, A.~Jalali, S.~Sanghavi, and C.~Caramanis, ``Low-rank matrix recovery
  from errors and erasures,'' \emph{IEEE Transactions on Information Theory},
  vol.~59, no.~7, pp. 4324--4337, 2013.

\bibitem{cherapanamjeri2017nearly}
Y.~Cherapanamjeri, K.~Gupta, and P.~Jain, ``Nearly optimal robust matrix
  completion,'' in \emph{International Conference on Machine Learning}, 2017,
  pp. 797--805.

\bibitem{gu2016low}
Q.~Gu, Z.~Wang, and H.~Liu, ``Low-rank and sparse structure pursuit via
  alternating minimization,'' in \emph{Artificial Intelligence and Statistics},
  2016, pp. 600--609.

\bibitem{cai2017spectral}
J.-F. Cai, T.~Wang, and K.~Wei, ``Spectral compressed sensing via projected
  gradient descent,'' \emph{SIAM Journal on Optimization}, vol.~28, no.~3, pp.
  2625--2653, 2018.

\bibitem{chen2013robustSpectralMC}
Y.~Chen and Y.~Chi, ``Robust spectral compressed sensing via structured matrix
  completion,'' \emph{IEEE Transactions on Information Theory}, vol.~60,
  no.~10, pp. 6576--6601, 2014.

\bibitem{tang2012compressive}
G.~Tang, B.~N. Bhaskar, P.~Shah, and B.~Recht, ``Compressed sensing off the
  grid,'' \emph{IEEE Transactions on Information Theory}, vol.~59, no.~11, pp.
  7465--7490, 2013.

\bibitem{cambareri2016non}
V.~Cambareri and L.~Jacques, ``A non-convex blind calibration method for
  randomised sensing strategies,'' in \emph{2016 4th International Workshop on
  Compressed Sensing Theory and its Applications to Radar, Sonar and Remote
  Sensing (CoSeRa)}.\hskip 1em plus 0.5em minus 0.4em\relax IEEE, 2016, pp.
  16--20.

\bibitem{qu2014finding}
Q.~Qu, J.~Sun, and J.~Wright, ``Finding a sparse vector in a subspace: Linear
  sparsity using alternating directions,'' in \emph{Advances in Neural
  Information Processing Systems}, 2014, pp. 3401--3409.

\bibitem{yi2014alternating}
X.~Yi, C.~Caramanis, and S.~Sanghavi, ``Alternating minimization for mixed
  linear regression,'' in \emph{International Conference on Machine Learning},
  2014, pp. 613--621.

\bibitem{ge2017optimization}
R.~Ge and T.~Ma, ``On the optimization landscape of tensor decompositions,'' in
  \emph{Advances in Neural Information Processing Systems}, 2017, pp.
  3653--3663.

\bibitem{arous2017landscape}
G.~B. Arous, S.~Mei, A.~Montanari, and M.~Nica, ``The landscape of the spiked
  tensor model,'' \emph{accepted to Communications on Pure and Applied
  Mathematics}, 2017.

\bibitem{li2017convex}
Q.~Li and G.~Tang, ``Convex and nonconvex geometries of symmetric tensor
  factorization,'' in \emph{Asilomar Conference on Signals, Systems, and
  Computers}, 2017, pp. 305--309.

\bibitem{zhong2017recovery}
K.~Zhong, Z.~Song, P.~Jain, P.~L. Bartlett, and I.~S. Dhillon, ``Recovery
  guarantees for one-hidden-layer neural networks,'' in \emph{International
  Conference on Machine Learning}, 2017, pp. 4140--4149.

\bibitem{li2017convergence}
Y.~Li and Y.~Yuan, ``Convergence analysis of two-layer neural networks with
  {ReLU} activation,'' in \emph{Advances in Neural Information Processing
  Systems}, 2017, pp. 597--607.

\bibitem{hand2017global}
P.~Hand and V.~Voroninski, ``Global guarantees for enforcing deep generative
  priors by empirical risk,'' in \emph{Conference On Learning Theory}, 2018,
  pp. 970--978.

\bibitem{fu2018local}
H.~Fu, Y.~Chi, and Y.~Liang, ``Local geometry of one-hidden-layer neural
  networks for logistic regression,'' \emph{arXiv preprint arXiv:1802.06463},
  2018.

\bibitem{fan2019selective}
J.~Fan, C.~Ma, and Y.~Zhong, ``A selective overview of deep learning,''
  \emph{arXiv preprint arXiv:1904.05526}, 2019.

\bibitem{burer2003nonlinear}
S.~Burer and R.~D. Monteiro, ``A nonlinear programming algorithm for solving
  semidefinite programs via low-rank factorization,'' \emph{Mathematical
  Programming}, vol.~95, no.~2, pp. 329--357, 2003.

\bibitem{bhojanapalli2018smoothed}
S.~Bhojanapalli, N.~Boumal, P.~Jain, and P.~Netrapalli, ``Smoothed analysis for
  low-rank solutions to semidefinite programs in quadratic penalty form,'' in
  \emph{Conference On Learning Theory}, 2018, pp. 3243--3270.

\bibitem{bandeira2016low}
A.~S. Bandeira, N.~Boumal, and V.~Voroninski, ``On the low-rank approach for
  semidefinite programs arising in synchronization and community detection,''
  in \emph{Conference on Learning Theory}, 2016, pp. 361--382.

\bibitem{sunju_ncvx}
\BIBentryALTinterwordspacing
 [Online]. Available: \url{http://sunju.org/research/nonconvex/}
\BIBentrySTDinterwordspacing

\bibitem{jain2017non}
P.~Jain and P.~Kar, ``Non-convex optimization for machine learning,''
  \emph{Foundations and Trends{\textregistered} in Machine Learning}, vol.~10,
  no. 3-4, pp. 142--336, 2017.

\bibitem{8682568}
Q.~{Li}, Z.~{Zhu}, G.~{Tang}, and M.~B. {Wakin}, ``The geometry of
  equality-constrained global consensus problems,'' in \emph{ICASSP}, May 2019,
  pp. 7928--7932.

\end{thebibliography}

\end{document}